\documentclass[final]{elsarticle}
\usepackage{lscape}
\usepackage[english]{babel}
\usepackage{url}
\usepackage{amsmath}
\usepackage{a4}
\usepackage{amssymb, amsmath}           
\usepackage{amsthm}
\usepackage{mathrsfs}
\usepackage{epsfig}
\usepackage{color,graphicx}             
\usepackage{setspace}
\usepackage{caption}
\usepackage{algorithm} 
\usepackage{algpseudocode}
\graphicspath{{fig/}}

\newtheorem{assumption}{Assumption}
\newtheorem{lemma}{Lemma}
\newtheorem{theorem}{Theorem}
\newtheorem{definition}{Definition}
\newtheorem{remark}{Remark}
\newtheorem{corollary}{Corollary}
\newtheorem{proposition}{Proposition}
\usepackage{subcaption}

\definecolor{wheat}{rgb}{0.96,0.87,0.70}
\DeclareMathOperator*{\argmax}{argmax}
\DeclareMathOperator*{\argmin}{argmin}

\begin{document}
\begin{frontmatter}
\title{No-Regret Constrained Bayesian Optimization of Noisy and Expensive Hybrid Models using Differentiable Quantile Function Approximations}

\author{Congwen Lu}
\ead{lu.2318@osu.edu}
\author{Joel A. Paulson\corref{c1}}
\cortext[c1]{Corresponding author}
\ead{paulson.82@osu.edu}

\address{Department of Chemical and Biomolecular Engineering, The Ohio State University, Columbus OH, USA}

\tnotetext[1]{This work was supported by the National Science Foundation (NSF) under award number 2237616.}

\begin{abstract}
This paper investigates the problem of efficient constrained global optimization of hybrid models that are a composition of a known white-box function and an expensive multi-output black-box function subject to noisy observations, which often arises in real-world science, engineering, manufacturing, and control applications.
We propose a novel method, Constrained Upper Quantile Bound (CUQB), to solve such problems that directly exploits the composite structure of the objective and constraint functions that we show leads substantially improved sampling efficiency. CUQB is a conceptually simple, deterministic approach that avoids the constraint approximations used by previous methods. Although the CUQB acquisition function is not available in closed form, we propose a novel differentiable sample average approximation that enables it to be efficiently maximized. We further derive bounds on the cumulative regret and constraint violation under a non-parametric Bayesian representation of the black-box function.
Since these bounds depend sublinearly on the number of iterations under some regularity assumptions, we establish explicit bounds on the convergence rate to the optimal solution of the original constrained problem. In contrast to most existing methods, CUQB further incorporates a simple infeasibility detection scheme, which we prove triggers in a finite number of iterations when the original problem is infeasible (with high probability given the Bayesian model). Numerical experiments on several test problems, including environmental model calibration and real-time optimization of a reactor system, show that CUQB significantly outperforms traditional Bayesian optimization in both constrained and unconstrained cases. Furthermore, compared to other state-of-the-art methods that exploit composite structure, CUQB achieves competitive empirical performance while also providing substantially improved theoretical guarantees.
\end{abstract}

\begin{keyword}
Constrained Bayesian optimization \sep Hybrid modeling \sep Multi-output Gaussian process regression \sep Quantile functions
\end{keyword}

\end{frontmatter}

\section{Introduction}
\label{sec:introduction}

In many real-world science and engineering applications, we are tasked with zeroth order derivative-free optimization (DFO) of an expensive-to-evaluate functions. For example, closed-loop performance optimization \cite{piga2019performance,paulson2020data,sorourifar2021data,del2021real}, hyperparameter tuning in machine learning algorithms \citep{pmlr-v115-wu20a}, the choice of laboratory experiments in material and drug design \citep{kapetanovic2008computer,ju2017designing,SCHWEIDTMANN2018277}, and calibration of computationally intensive simulators \citep{vrugt2001calibration,schultz2018bayesian,paulson2019fast} can all be formulated as derivative-free optimization problems.
In such applications, we can only learn about the unknown functions, which might be highly nonconvex and/or multi-modal, by sequentially querying them at specific input points. Additionally, when each function evaluation costs a large amount of energy, money, time, and/or computational resources, then we want to use as few evaluations as possible to find a globally optimal solution. For example, optimizing the hyperparameters of a deep neural network requires the weight and bias parameters to be internally optimized for every tested configuration, which can take substantial amounts of wall-clock time and/or GPU computation. 

DFO methods can be broadly categorized as either stochastic or deterministic \cite{rios13}. Stochastic DFO methods, such evolutionary and swarm-based algorithms \cite{eberhart1995,hansen2003reducing,mukhopadhyay09}, typically require a large number of high precision function evaluations and therefore cannot be applied to expensive noisy black-box functions. 
Deterministic DFO approaches, which are often motivated by optimization of expensive functions, are typically further subdivided into direct search and model-based methods. Direct search methods propose sample points that evaluate the function over a particular geometry such as Nelder-Mead (NM) \cite{nelder1965simplex} and generalized pattern search (GPS) \cite{kolda2003optimization}. Many other direct DFO algorithms based on box subdivision principles have been developed, which are often inspired from the DIviding RECTangle (DIRECT) algorithm \cite{jones1993}. Model-based DFO methods, on the other hand, construct surrogate models and propose sample points by optimizing the model. The choice of surrogate model plays a big role on the performance and convergence of the method. Parametric function approximators may not be sufficiently expressive enough to capture the complexity of the unknown function and can be inaccurate when a small number of noisy observations are available. Although some methods that can handle observation noise in the DFO setting have been developed, e.g., \cite{chen2018stochastic,curtis2023stochastic}, they often focus on finding a local optimum. 

The common theme amongst the previously described DFO approaches is that they do not directly address the tradeoff between exploration (learning for the future) and exploitation (immediately advancing toward the goal given current knowledge), which implies they are prone to sample the unknown functions more than required. 
Bayesian optimization (BO) \citep{shahriari15,frazier18}, on the other hand, refers to a collection of data-efficient global optimization methods for expensive black-box functions subjected to noisy observations. BO’s effectiveness is derived from its probabilistic representation of the unknown functions, which allows it to characterize the expected information content that future sample points show with respect to potential global optima. As such, BO can systematically address the exploration-exploitation tradeoff by taking an information-theoretic view of the global black-box optimization. 
BO consists of two key steps: (i) the construction of some ``surrogate model'' given by a Bayesian posterior distribution equipped with uncertainty estimates and (ii) the specification of an ``acquisition function'' (that depends on the posterior distribution) whose value at any point is a metric for the value/utility of evaluating the unknown functions at this point. 

In addition to expensive objective evaluations, many optimization problems also have expensive black-box constraint functions. For example, when designing a pharmaceutical drug, one may wish to maximize the drug's activity while constraining its toxicity to be below acceptable limits. When deploying a machine learning algorithm on low-cost embedded hardware, one often wants to maximize performance subject to constraints on the model's memory and power usage. Similarly, when tuning the parameters of a temperature controller in a building application, one may want to minimize closed-loop energy consumption subject to comfort constraints. The same challenge is present in each of these examples: both the performance and constraint functions have an unknown (black-box) relationship to the design/tuning parameters.
There has been a significant amount of work on extending BO to handle black-box constraints. All methods are based on the idea that separate probabilistic surrogate models for the black-box objective and constraint functions can be constructed. Thus, they mainly differ in terms of how they incorporate the constraint models into the acquisition function. 

Roughly speaking, all constrained BO methods fall into one of two groups, i.e., those that do and those that do not allow constraint violation during the optimization process. The setting wherein no violations are tolerated is often referred to as ``safe BO'' \cite{sui2015safe,bergmann2020safe,berkenkamp2021bayesian,krishnamoorthy2022safe} and are motivated by online safety-critical applications (e.g., tuning an autonomous vehicle, robot, or personalized medicine control system). The enforcement of such hard constraints under uncertainty, however, implies that these algorithms may get stuck in a local (suboptimal) solution. This setting is not relevant in this work since we are interested in developing an algorithm that is guaranteed to find the true global solution. The other class of constrained BO methods (that do allow violation at each iteration) can be further divided into either \textit{implicit} or \textit{explicit} methods. Implicit methods encode the constraint information as a type of merit function in the original unconstrained acquisition function -- in other words they do not directly enforce any constraint on the allowable violation at every iteration. Two examples of merit functions that have been proposed are the expected improvement with constraints (EIC) function \citep{gardner2014bayesian} and the augmented Lagrangian BO (ALBO) \citep{picheny2016bayesian}. However, to the best of our knowledge, no theoretical guarantees on convergence, optimality, or constraint violations incurred during the optimization process have been established for this class of methods. 

Explicit methods, on the other hand, solve a constrained acquisition optimization problem at every iteration and, thereby, place more direct limits on the degree of violation. One of the earliest explicit constrained BO methods is \citep{sasena2002exploration}, which simply used the mean prediction of the constraint function to construct an approximate feasible region; however, this was shown to lead to poor sample selections in the early iterations when the model is not very accurate. The upper trust bound (UTB) method was proposed in \cite{priem2020upper} to overcome this challenge by incorporating uncertainty in the constraint predictions -- the key observation in this work is that better performance is achieved when one constructs a \textit{relaxation} of the feasible region, though no theoretical results or strategies for parameter tuning were provided. This gap was addressed by the authors in \cite{lu2022no} wherein we developed exact penalty BO (EPBO), which is an extension of UTB that uses a specific type of penalty function to establish convergence guarantees under certain regularity assumptions. EPBO requires the selection of a single penalty weight parameter $\rho$; however, it was recently shown in \cite{xu2022constrained} that even stronger results can be obtained when setting $\rho = \infty$. In particular, \cite{xu2022constrained} derived explicit bounds for the rate of convergence to the optimal solution and the degree of constraint violation as a function of the total number of iterations. Furthermore, by setting $\rho = \infty$, one can naturally derive an infeasibility detection scheme, which is unique to \cite{xu2022constrained}.

The main downside of all the previously mentioned constrained BO methods is that they assume the objective and constraint functions are fully black box in nature, meaning we have little-to-no prior information about their structure. It has been shown in \cite{chen1988lower} that the black-box assumption provides a strong limit on the practical rate of convergence, which can only be overcome by utilizing prior knowledge about the structure of these functions. 
Luckily, in many real-world problems, we have a substantial amount of prior knowledge that can be exploited in the development of more efficient search methods. 
This concept forms the basis of most ``grey-box'' or ``hybrid'' optimization algorithms that have been developed within the process systems engineering community over the past several years, e.g., \cite{eason2016trust,beykal2018optimal,bajaj2018trust,kim2020surrogate}. Despite the prevalence of grey-box modeling/optimization, there has been limited work on (constrained) grey-box BO due to the complexities introduced when combining known nonlinear equations with probabilistic surrogate models. Our group recently developed a method, COBALT \cite{paulson2022cobalt}, that is applicable to such problems where the objective and constraints are represented by composite functions, i.e., $f(x) = g(x,h(x))$ where $h$ is an expensive black-box vector-valued function and $g$ is a known function that can be cheaply evaluated. COBALT can be interpreted as an extension of EIC to handle the composite structure of the objective and constraint functions; however, it inherits several of the aforementioned challenges with EIC, which are described in detail in Section \ref{subsec:challenge-cobalt}. 

In this paper, we propose the constrained upper quantile bound (CUQB) as a simple and effective algorithm for efficient constrained global optimization of expensive hybrid models (grey-box composite functions) subjected to noisy observations. We can interpret CUQB as an extension of the EPBO method \cite{lu2022no,xu2022constrained} to handle general composite functions. As such, at every iteration, CUQB requires the solution of an auxiliary constrained optimization problem with the objective and constraint functions replaced by their high-probability quantile bounds. In the standard black-box setting, we can develop closed-form expressions for the quantile bounds by exploiting the properties of Gaussian random variables, which leads to the well-known upper confidence bound (UCB) acquisition function \cite{srinivas09}. This no longer holds in the composite function setting (since a nonlinear transformation of a Gaussian random variable is no longer Gaussian), so CUQB must be evaluated via stochastic simulation. We first develop a convergent simulation-based approximation of the quantile function using a sorting operator. We then develop a simple differentiable approximation of this estimator using the soft sorting projection operator from \cite{blondel2020fast}, which enables efficient maximization of the quantile functions. Similarly to \cite{xu2022constrained}, we theoretically analyze the performance of CUQB under several conditions. Not only do we show that CUQB has at least the same cumulative regret and cumulative constraint violation bounds as the black-box case, we also establish explicit bounds on the rate of convergence to the true constrained solution under certain standard regularity assumptions. Additionally, we also incorporate an infeasibility detection scheme in CUQB that will only trigger (in a finite number of iterations) if the original problem is infeasible with high probability.
To the best of our knowledge, the proposed CUQB method is the first and only constrained grey-box BO method to provide theoretical guarantees on performance and infeasibility detection. To demonstrate the effectiveness of the proposed algorithm, we compare its performance to standard BO methods and COBALT on a variety of test problems. These results confirm that CUQB consistently finds solutions that are orders of magnitude better than that found by standard BO, and similar or better quality solutions than COBALT while simultaneously providing a guarantee of either declaring infeasibility or finding the globally optimal solution.










The remainder of this paper is organized as follows. In Section~\ref{sec:cobalt}, we provide an overview of the constrained optimization problems of interest in this work as well as summarize our previous method, COBALT, for solving such problems. In Section~\ref{sec:cuqb}, we present our new approach, CUQB, that can overcome some challenges of COBALT and show how it can be efficiently implemented in practice. 
We perform some theoretical analysis on the cumulative regret and violation probability, convergence, and infeasibility declaration in Section~\ref{sec:theory}. Section~\ref{sec:comp-setup} describes the computational setup for the numerical experiments including the practical implementation details of CUQB.
We then compare the performance of CUQB to traditional BO approaches and COBALT on a variety of benchmark problems in Section~\ref{sec:case-studies} and present some concluding remarks in Section~\ref{sec:conclusions}.


\section{Constrained Bayesian Optimization with Composite Functions}
\label{sec:cobalt}

This section first introduces the class of optimization problems considered in this work and then provides a summary of our recent approach, called COBALT \cite{paulson2022cobalt}, for tackling such problems. We end by discussing some potential challenges exhibited by COBALT that motivate our proposed work described in Section \ref{sec:cuqb}. 

\subsection{Problem formulation}

In this work, we consider constrained grey-box optimization problems over composite functions of the following form\footnote{Throughout the paper, we use standard BO convention to define our original optimization problem in \eqref{eq:grey-box-opt}, which uses a reward (max) perspective subject to $\geq 0$ inequalities. This is opposite to the standard convention in the optimization literature, which uses a cost (min) perspective subject to $\leq 0$ inequalities. Note that we can always convert between these representations by introducing a minus sign in the objective and constraints.}
\begin{subequations} \label{eq:grey-box-opt}
\begin{align}
    \max_{x \in \mathcal{X}} &~~ g_0(x, h(x)), \\
    \text{subject to:} &~~g_i(x, h(x)) \geq 0, ~~~ \forall i \in \mathbb{N}_1^n,
\end{align}
\end{subequations}
where $x \in \mathbb{R}^{d}$ is a set of $d$ decision variables that are constrained to some compact set of candidate solutions $\mathcal{X} \subset \mathbb{R}^d$; $\mathbb{N}_a^b = \{a,\ldots,b\}$ denotes the set of integers from $a$ to $b$; $h : \mathbb{R}^d \to \mathbb{R}^m$ is an expensive-to-evaluate black-box function with $m$ outputs whose structure is completely unknown and evaluations do not provide derivative information; and $g_i : \mathbb{R}^d \times \mathbb{R}^m \to \mathbb{R}$ are known, cheap-to-evaluate functions that define the objective ($i=0$) and constraints ($i \in \mathbb{N}_1^n$) of the problem. The functions $\{ g_i(x,h(x)) \}_{i \in \mathbb{N}_0^n}$ are \textit{composite} functions of known and unknown components. The known functions $\{ g_i \}_{i \in \mathbb{N}_0^n}$ can be interpreted as prior knowledge that can be exploited when developing an algorithm to solve \eqref{eq:grey-box-opt}.

Since the function $h$ is unknown, we must learn it from data. We focus on the so-called bandit feedback setting wherein, at every iteration $t$, an algorithm selects a query point $x_t \in \mathcal{X}$ for which noisy evaluations of $h(x_t)$ can be observed. That is, we receive an observation $y_{t} = h(x_t) + \epsilon_t$
where $\epsilon_t \sim \mathcal{N}(0, \sigma^2 I_m)$ is a zero-mean Gaussian noise term. 
To make progress toward an efficient sampling procedure, we need to make some ``regularity'' assumption for $h$. Here, we will focus on the case that $h$ satisfies some underlying smoothness properties such that the sequence of data points $\mathcal{D}_t = \{ (x_i, y_i) \}_{i=1}^t$ can be  correlated to future predictions. In particular, we assume $h$ is a sample from a multi-output Gaussian process (GP) distribution, $\mathcal{GP}(\mu_0, K_0)$, where $\mu_0 : \mathcal{X} \to \mathbb{R}^m$ is the prior mean function and $K_0 : \mathcal{X} \times \mathcal{X} \to \mathbb{S}^m_{++}$ is the prior covariance (or kernel) function with $\mathbb{S}_{++}^m$ denoting the set of positive definite matrices of dimension $m$ \cite{williams2006gaussian}. Given the current set of $t$ observations $\mathcal{D}_t$, the posterior remains a GP with closed form expressions for updated mean $\mu_t$ and covariance $K_t$ functions \cite{liu2018remarks}. As such, the posterior predictions of the unknown function $h(x)$ can be described by a multivariate normal distribution of the form
\begin{align} \label{eq:gp-posterior}
    h(x) \mid \mathcal{D}_t \sim \mathcal{N}( \mu_t(x), \Sigma_t(x) ),
\end{align}
where $\Sigma_t(x) = K_t(x,x)$ is the (marginal) covariance matrix at the test point $x$. Since \eqref{eq:grey-box-opt} is capable of flexibly fusing physical knowledge with data, it is a particuarly effective formulation for optimizing over ``hybrid'' models. 


\begin{remark}
There has been significant work on multi-output GPs, with the most straightforward method being to model each output as an independent GP with or without shared hyperparameters, e.g., \cite{eriksson2021scalable}. There are two main drawbacks of such approaches: (i) it neglects potentially important correlation between the outputs and (ii) constructing and storing a separate GP can be inefficient (from both a memory and computation point of view), especially when the number of outputs is large. As such, more recent work has focused on so-called multi-task GP (MTGP) models that use structured covariance functions of the form $k([x,i],[x',j]) = k_\text{data}(x,x')k_\text{task}(i,j)$ where $x$ is the input to task $i \in \mathbb{N}_1^m$ and $x'$ is the input to task $j \in \mathbb{N}_1^m$. This induces a Kronecker structure in the covariance matrix, which greatly simplifies the equations for $\mu_t(x)$ and $\Sigma_t(x)$, as shown in \cite{maddox2021bayesian}. Our proposed algorithm applies in all of these settings, as it only requires the ability to evaluate \eqref{eq:gp-posterior}. 
\end{remark}

\begin{remark}
The noise term $\epsilon_t$ is used to model errors associated with the measurement process used to obtain values of the unknown function $h(x)$, which is never perfect in practice. For a physical system, this noise would likely come from unmeasured fluctuations in the sensor. For a simulation, this noise could be due to stochastic fluctuations in the underlying model (e.g., random motion of particles in a molecular dynamics simulation) or other unknown sources of numerical error. It is worth noting that we assume that this noise follows a zero-mean Gaussian distribution to simplify the construction of the posterior distribution in \eqref{eq:gp-posterior}. Although our theoretical results in Section \ref{sec:theory} are based on this assumption, there are strategies that can be used to replace it with the more agnostic setting considered in \cite{srinivas09} in which the noise only needs to be i.i.d. $\sigma$-sub Gaussian. The proposed algorithm can be straightforwardly applied in this more general setting and we plan to study the theoretical implications of more general noise settings in future work.
\end{remark}

\subsection{The COBALT algorithm}

In principle, we could tackle \eqref{eq:grey-box-opt} with Bayesian optimization (BO) algorithms by neglecting the composite structure and treating it as a black-box problem 
\begin{subequations} \label{eq:black-box}
\begin{align}
    \max_{x \in \mathcal{X}} &~~ f_0(x), \\ 
    \text{subject to:} &~~ f_i(x) \geq 0, ~~ \forall i \in \mathbb{N}_1^n,
\end{align}
\end{subequations}
where $f_i(x) = g_i(x, h(x))$ for all $i \in \mathbb{N}_0^n$. Given GP surrogate models for the unknown functions, the second major component of BO is to choose an acquisition function $\alpha_t : \mathcal{X} \to \mathbb{R}$ that depends on the posterior distributions. When properly selected, the value of $\alpha_t(x)$ at any $x \in \mathcal{X}$ should be a good measure of the potential benefit of querying the unknown function at this point in the future \cite{astudillo2021thinking}. We can then define the BO framework as the sequential learning process that corresponds to selecting $x_{t+1}$ as the maximizer of $\alpha_t$, i.e., $x_{t+1} \in \argmax_{x \in \mathcal{X}} \alpha_t(x)$. The challenge is that, although we can place GP priors on $\{ f_i \}_{i \in \mathbb{N}_0^n}$ to derive $\alpha_t$, this is not an effective model in general for the case of composite functions. Consider the simple example $f_0(x) = h(x)^2$ -- we clearly see that $f_0$ must be positive for all $x$ but this is not enforced when $f_0$ is directly modeled as a GP. This observation was the main motivation for our previous algorithm, COBALT, that constructs an improved acquisition function given proper statistical representations of the posterior predictive distributions $\{ f_i(x) \mid \mathcal{D}_t \}_{i \in \mathbb{N}_0^n}$. 

COBALT can be derived using a similar thought experiment to the expected improvement (EI) acquisition function \cite{movckus1975bayesian,jones1998efficient} commonly used in the black box case. Suppose we have sampled at a set of $t$ points, the observations are noise-free,
and we are only willing return an evaluated point as our final solution. If we have no evaluations left to make, the optimal choice must correspond to the point with the largest feasible objective value. Let $f_{0,t}^\star$ denote the objective value of this point where $t$ is the number of times we have evaluated $h$ so far. That is,
\begin{align} \label{eq:fnstar}
    f_{0,t}^\star = \begin{cases}
        \max_{j \in S_t} f_0(x_j), &\text{if~} |S_t| > 0, \\
        M, &\text{otherwise}.
    \end{cases}
\end{align}
where $S_t = \{ j \in \mathbb{N}_{1}^t : f_i(x_j) \geq 0, ~\forall i \in \mathbb{N}_1^n \}$ is the set of feasible observations and $M \in \mathbb{R}$ is a penalty for not having a feasible solution (typically set to be smaller than the lowest possible objective value $M \leq \min_{x \in \mathcal{X}} f_0(x)$). 
Now suppose that we have one additional evaluation to perform and can perform it at any point $x \in \mathcal{X}$ for which we observe $h(x)$. After this new evaluation, the value of the best point will either be $f_0(x)=g_0(x,h(x))$ or $f_{0,t}^\star$. The former will occur if the point was feasible ($f_i(x)=g_i(x,h(x)) \geq 0, \forall i \in \mathbb{N}_1^n$) and it leads to an increase in the objective value $f_0(x) \geq f^\star_{0,t}$. The latter will occur if the point was infeasible ($\exists i \in \mathbb{N}_1^n : f_i(x) < 0$) or there was no improvement in the objective value $f_0(x) \leq f^\star_{0,t}$.
The increase/improvement in the best observed point due to this new evaluation can then be compactly represented as $[f_0(x) - f_{0,t}^\star]^+ \Delta(x)$ where $[a]^+ = \max\{a,0\}$ is the positive part and $\Delta(x)$ is a feasibility indicator that is equal to 1 if $f_i(x) \geq 0, \forall i \in \mathbb{N}_1^n$ or 0 otherwise. We would like to choose $x$ to make this improvement as large as possible; however, $h(x)$ is unknown until after the evaluation. A simple way to deal with this uncertainty is to take the expected value of the improvement, which is the basis for the COBALT acquisition function \cite{paulson2022cobalt}
\begin{align} \label{eq:cobalt}
    \text{COBALT}_t(x) = \mathbb{E}_t\left\lbrace [g_0(x,h(x)) - f_{0,t}^\star]^+ \Delta(x) \right\rbrace,
\end{align}
where $\mathbb{E}_t\{ \cdot \} = \mathbb{E}_t\{ \cdot \mid \mathcal{D}_t \}$ is the expectation under the posterior distribution of $h$ given the data $\mathcal{D}_t$. The posterior distribution is given in \eqref{eq:gp-posterior}, which is a multivariate normal with mean $\mu_t(x)$ and covariance $\Sigma_t(x)$. 

\subsection{Current challenges with COBALT}
\label{subsec:challenge-cobalt}

The COBALT acquisition function \eqref{eq:cobalt} is constructed to be one-step Bayes optimal under the two previously mentioned assumptions: (i) the function evaluations are noise-free and (ii) the final recommended point is restricted to a previously sampled point. However, there are some theoretical and computational challenges with COBALT that we briefly summarize below.

\subsubsection{Noisy observations}

In many applications, such as those involving physical experiments and/or stochastic simulations, the function evaluations are inherently noisy such that \eqref{eq:cobalt} is not directly applicable because $f_{0,t}^\star$ is not well-defined under noisy observations. In practice, one can replace $\max_{j \in S_t} f_0(x_j)$ in \eqref{eq:fnstar} with its noisy counterpart defined in terms of $y_t$, but this results in a loss in the Bayes optimal property as well as does not properly account for uncertainty in the prediction. Another alternative is to use a ``plug-in'' estimate by maximizing the mean objective function subject to the mean constraint function \cite{gardner2014bayesian}, however, this has a tendency to overestimate the incumbent, leading to insufficient exploration in practice. 

\subsubsection{Computation and constraint handling}

Another important challenge with \eqref{eq:cobalt} is that it does not have an analytic expression for general nonlinear functions $\{ g_i \}_{i \in \mathbb{N}_0^n}$. Instead, we must resort to a Monte Carlo (MC) sampling procedure to approximate it by performing the following ``whitening transformation'' \cite{wilson2017reparameterization,zhang2021constrained} (derived from the closure of Gaussian distributions under linear transformations)
\begin{align} \label{eq:whitening-transform}
    h(x) \mid \mathcal{D}_t = \mu_t(x) + C_t(x) Z,
\end{align}
where $C_t(x)$ is the lower Cholesky factor of $\Sigma_t(x)=C_t(x)C_t(x)^\top$ and $Z \sim \mathcal{N}(0,I_m)$ is an $m$-dimensional standard normal random vector. Using \eqref{eq:whitening-transform}, the MC approximation of \eqref{eq:cobalt} is given by
\begin{align} \label{eq:cobalt-approx}
    \text{COBALT}_t(x) \approx \frac{1}{L}\sum_{l=1}^L \left\lbrace \left[ g_0(x,\mu_t(x)+C_t(x)Z^{(l)}) - f_{0,t}^\star \right]^+ \Delta^{(l)}(x) \right\rbrace,
\end{align}
where $Z^{(1)},\ldots,Z^{(L)} \sim \mathcal{N}(0,I_m)$ are a set of $L$ i.i.d. random samples and $\Delta^{(l)}(x)$ is the evaluation of the feasibility function for the $l^\text{th}$ sample
\begin{align}
    \Delta^{(l)}(x) = \mathbb{I}\left( \min_{1\leq i\leq m} g_i(x,\mu_t(x)+C_t(x)Z^{(l)}) \geq 0)
 \right),
\end{align}
for indicator function $\mathbb{I}(A)$ that is equal to 1 if logical proposition $A$ is true and 0 otherwise. The presence of the indicator function in \eqref{eq:cobalt-approx} results in a highly non-smooth function that is difficult to maximize using standard gradient-based methods. As such, we previously suggested to apply an explicit linearization technique to overcome this challenge in practice (see \cite{paulson2022cobalt} for details). Not only can this linearization technique degrade when $\{ g_i \}_{i=1}^n$ are highly nonlinear, this requires us to select so-called ``trust level'' parameters related to the desired satisfaction probability. In \cite{paulson2022cobalt}, we develop a heuristic procedure to design these trust levels, which performed well in practice, though their impact on performance remains an open question. 



\subsubsection{Theoretical guarantees}

Due to the practical approximations mentioned in the previous section, no theoretical performance guarantees have been established for the current implementation of COBALT. The heuristic noise and constraint handling methods further complicate this task. It is worth noting that, in the absence of constraints, one could follow a similar approach to \cite[Proposition 2]{astudillo2019bayesian} to establish convergence of the exact COBALT acquisition function, i.e., $f_{0,t}^\star \to \max_{x \in \mathcal{X}}f_0(x)$ when $x_{t+1} \in \argmax_{x \in \mathcal{X}} \text{COBALT}_t(x)$ as $t \to \infty$. However, this result only provides a guarantee of convergence in the limit of infinite samples and does not establish any sort of bounds on the rate of convergence for a finite $t$. This serves as a key source of inspiration for this work, as we would like to come up with an algorithm for which we can derive convergence rate bounds. We present such an algorithm in the next section. It also turns out that our proposed approach can straightforwardly handle noisy observations and does not require any constraint linearization. 


\section{Proposed Constrained Upper Quantile Bound (CUQB) Method}
\label{sec:cuqb}

This section summarizes our proposed approach, called constrained upper quantile bound (CUQB), for solving \eqref{eq:grey-box-opt}. The idea is a generalization of the upper confidence bound (UCB) method \cite{srinivas09,chowdhury2017kernelized,vakili2021information} for solving black-box optimization problems of the form \eqref{eq:black-box} when all functions $\{ f_i \}_{i \in \mathbb{N}_0^n}$ are fully unknown. UCB relies on the ability to calculate explicit expressions for the ``confidence bounds'' for GP models that give us a range of possible observations. As we show next, we can still compute such bounds for composite functions $\{ g_i(x,h(x)) \}_{i \in \mathbb{N}_0^n}$ (with non-Gaussian predictions) by determining quantiles of the implied random variable.

\subsection{Composite quantile functions}

Let $Y_{i,t}(x)$ denote the conditional random variable corresponding to the prediction of the $i^\text{th}$ function in \eqref{eq:grey-box-opt} at point $x$ given data $\mathcal{D}_t$ for all $i \in \mathbb{N}_0^n$. 
Using the transformation \eqref{eq:whitening-transform}, we can express the target random variable as follows
\begin{align} \label{eq:Yit}
    Y_{i,t}(x) = g_i(x, \mu_t(x) + C_t(x)Z), ~~ \forall i \in \mathbb{N}_0^n, \forall t \geq 0.
\end{align}
In the remainder of this section we will drop subscripts $i,t$ for notational simplicity and refer to $Y_{i,t}(x)$ as simply $Y(x)$, which is a random variable parametrized by the test point location $x$. Let $F_{Y(x)}(y) = \text{Pr}\{ Y(x) \leq y \}$ denote the cumulative distribution function (CDF) of $Y(x)$. The quantile function $Q_{Y(x)} : [0,1] \to \mathbb{R}$ then maps its input $p$ (a probability level) to a threshold value $y$ such that the probability of $Y(x)$ being less than or equal to $x$ is $p$. Therefore, we can think of the quantile as simply the generalized inverse CDF, i.e.,
\begin{align} \label{eq:quantile}
Q_{Y(x)}(p) = F_{Y(x)}^{-}(p) = \inf\{ y \in \mathbb{R} : F_{Y(x)}(y) \geq p \}.
\end{align}
Note that the generalized inverse exists for any transformation $g_i$. When the function $g_i(x,y)$ is differentiable in its second argument $y$, the CDF will have a unique inverse though we will not require this assumption in our analysis.
The quantile function can be derived in closed form for a Gaussian distribution but in general will have a more complicated representation. The main value of the quantile function in our context is that it can be immediately used to bound the range of possible predictions by properly selecting the probability levels $p$. For example, the following bound holds by construction for any $\alpha \in (0,1)$
\begin{align}
    \text{Pr}\left\lbrace Y(x) \in \left[ Q_{Y(x)}(\textstyle\frac{\alpha}{2}), Q_{Y(x)}(1-\textstyle\frac{\alpha}{2}) \right] \right\rbrace = 1 - \alpha.
\end{align}
The upper quantile bound $Q_{Y(x)}(1-\textstyle\frac{\alpha}{2})$ and lower quantile bound $Q_{Y(x)}(\textstyle\frac{\alpha}{2})$, respectively, represent our optimistic and pessimistic estimate of the composite function $g(x,h(x))$ given all available data. We will show later that both bounds are important: the upper bound is important for search while the lower bound is important for recommendation in the presence of noise.

\subsection{The CUQB sequential learning algorithm}

We are now in a position to state our proposed algorithm, which we refer to as constrained upper confidence bound (CUQB). To facilitate the description of CUQB, we formally define the upper and lower quantile bound functions next.

\begin{definition}[Quantile bound functions]
Let $Y_{i,t}(x)$ refer to the random variable in \eqref{eq:Yit} corresponding to the posterior prediction of the objective ($i=0$) or a constraint ($i \in \mathbb{N}_1^n$) function given data $\mathcal{D}_t$. The lower and upper quantile bounds for $Y_{i,t}(x)$ are then denoted by
\begin{subequations} \label{eq:quantile-bounds}
\begin{align}
    l_{i,t}(x) &= Q_{Y_{i,t}(x)}(\textstyle\frac{\alpha_{i,t}}{2}), \\
    u_{i,t}(x) &= Q_{Y_{i,t}(x)}(1-\textstyle\frac{\alpha_{i,t}}{2}),
\end{align}
\end{subequations}
where $\{ \alpha_{i,t} \}_{i\in\mathbb{N}_0^n,t \geq 0}$ is a known sequence of probability values $\alpha_{i,t} \in (0,1)$. 
\end{definition}

A complete description of the proposed CUQB method is provided in Algorithm \ref{alg:cuqb} in terms of these quantile bounds. This algorithm is conceptually simple and only requires one to solve a single auxiliary problem with the grey-box objective and constraint functions replaced by their upper quantile bound functions (Line 5). Since we have assumed that $h(x)$ is expensive to query, the cost of solving this auxiliary problem should be much smaller than the original problem \eqref{eq:grey-box-opt}, even though it is generally nonconvex. A practical approach to solving this optimization problem is discussed in the next section.\footnote{Throughout the manuscript, we refer to ``CUQB'' as the conceptual method in Algorithm \ref{alg:cuqb} that requires the optimization problem \eqref{eq:cuqb_subproblem} to be solved to global optimality at every iteration. In practice, it can be difficult to guarantee that one has found a global maximizer and so one must often resort to a heuristic global optimization procedure (e.g., multi-start of a local gradient-based optimization algorithm) in practice. Although the theoretical analysis does not hold for the practical implementation of CUQB, our extensive numerical results show that strong performance results can still be achieved in practice under such approximations.}

\begin{algorithm}[tb!]
\caption{The \textbf{C}onstrained \textbf{U}pper \textbf{Q}uantile \textbf{B}ound (CUQB) Grey-Box Optimization Algorithm for Noisy and Expensive Functions}
\textbf{Input:} Compact domain $\mathcal{X}$; GP prior mean function $\mu_0(x)$ and covariance function $K_0(x,x')$; probability values $\{ \alpha_{i,t} \}_{i\in\mathbb{N}_0^n,t \geq 0}$; recommendation penalty parameter $\rho > 0$; and total number of iterations $T$.  
\begin{algorithmic}[1]
\For{$t=0$ to $T-1$}
\If{$\exists i \in \mathbb{N}_1^n$ such that $\max_{x \in \mathcal{X}} u_{i,t}(x) < 0$}
\State Declare that the original problem \eqref{eq:grey-box-opt} is infeasible and stop. 
\EndIf
\State Solve the following auxiliary constrained optimization problem to determine the next best sample point $x_{t+1}$ given our current posterior distribution
\begin{align} \label{eq:cuqb_subproblem}
    x_{t+1} \in \argmax_{x \in \mathcal{X}} \left\lbrace u_{0,t}(x) ~~ \text{s.t.} ~~ u_{i,t}(x) \geq 0, ~~ \forall i \in \mathbb{N}_1^n \right\rbrace,
\end{align}
\State Query unknown function $h$ at $x_{t+1}$ to get noisy observation $y_{t+1}$.
\State Update data $\mathcal{D}_{t+1} = \mathcal{D}_t \cup \{ (x_{t+1}, y_{t+1}) \}$.
\State Update multi-output GP model for $h$ in \eqref{eq:gp-posterior} with new data $\mathcal{D}_{t+1}$. 
\EndFor
\State Return the point $x_{t^\star}$ with the largest penalized lower quantile bound as our best guess for the optimal solution
\begin{align} \label{eq:recommended-point}
    t^\star = \argmax_{t \in \{ 1,\ldots, T \}}  \left\lbrace l_{0,t-1}(x_t) - \rho \sum_{i=1}^n [-l_{i,t-1}(x_t)]^+ \right\rbrace.
\end{align}
\end{algorithmic}
\label{alg:cuqb}
\end{algorithm}

There are two other features of Algorithm \ref{alg:cuqb} that deserve further elaboration. First, in Lines 1-4, we have included an infeasibility detection scheme. Due to the consideration of potentially unbounded observation noise and finite quantile bounds, it is possible that infeasibility is declared even when the original problem \eqref{eq:grey-box-opt} is feasible. Luckily, we can control the probability of a false declaration through a proper selection scheme for $\{ \alpha_{i,t} \}_{i\in\mathbb{N}_0^n,t \geq 0}$, meaning we can ensure this is triggered only when the actual problem is infeasible with high probability, which is discussed further in Section \ref{sec:theory}. Second, in Line 10, we have incorporated a recommendation procedure for returning the best point in the sequence $\{ x_1,\ldots, x_T \}$. Although there are many potential choices (e.g., simply return the final point $x_T$), this approach is particularly robust to the noise in the observations and allows for convergence rate bounds to be established, which is also discussed in detail in Section \ref{sec:theory}.

Figure \ref{fig:illustrative_example} provides a comparison of the classical EI and UCB acquisition functions to COBALT and UQB when there are no constraints, $h(x) = \sin(x) + \cos(x)$ is scalar valued, $g(x,h(x)) = h(x)^2$, and we have evaluated $h(x)$ at five points. The left-hand column shows the hybrid (grey-box) modeling case wherein we compute a GP model for $h(x)$ and propagate the uncertainty through the known function $g(x, y)$ to compute an implied posterior for $f(x) = g(x, h(x))$. The right-hand column shows the standard black-box modeling case wherein we directly compute a GP model for $f(x)$. The first main observation is that the black-box GP posterior has much larger 95\% confidence intervals than the grey-box posterior including predicting the potential for negative values. 
Both COBALT and UQB are able to take advantage of this valuable information to more quickly hone in on good sampling regions. Note that, in addition to more naturally handling noise than COBALT, UQB does not have flat regions, which makes it a bit easier to practically optimize.\footnote{We developed a heuristic scaling procedure in \cite{paulson2022cobalt} to overcome the challenges with the flat regions of COBALT that have gradient equal to zero.}  

\begin{figure}[ht!] 
    \centering
    \begin{subfigure}[b]{0.46\textwidth}
        \centering
        \includegraphics[width=\textwidth]{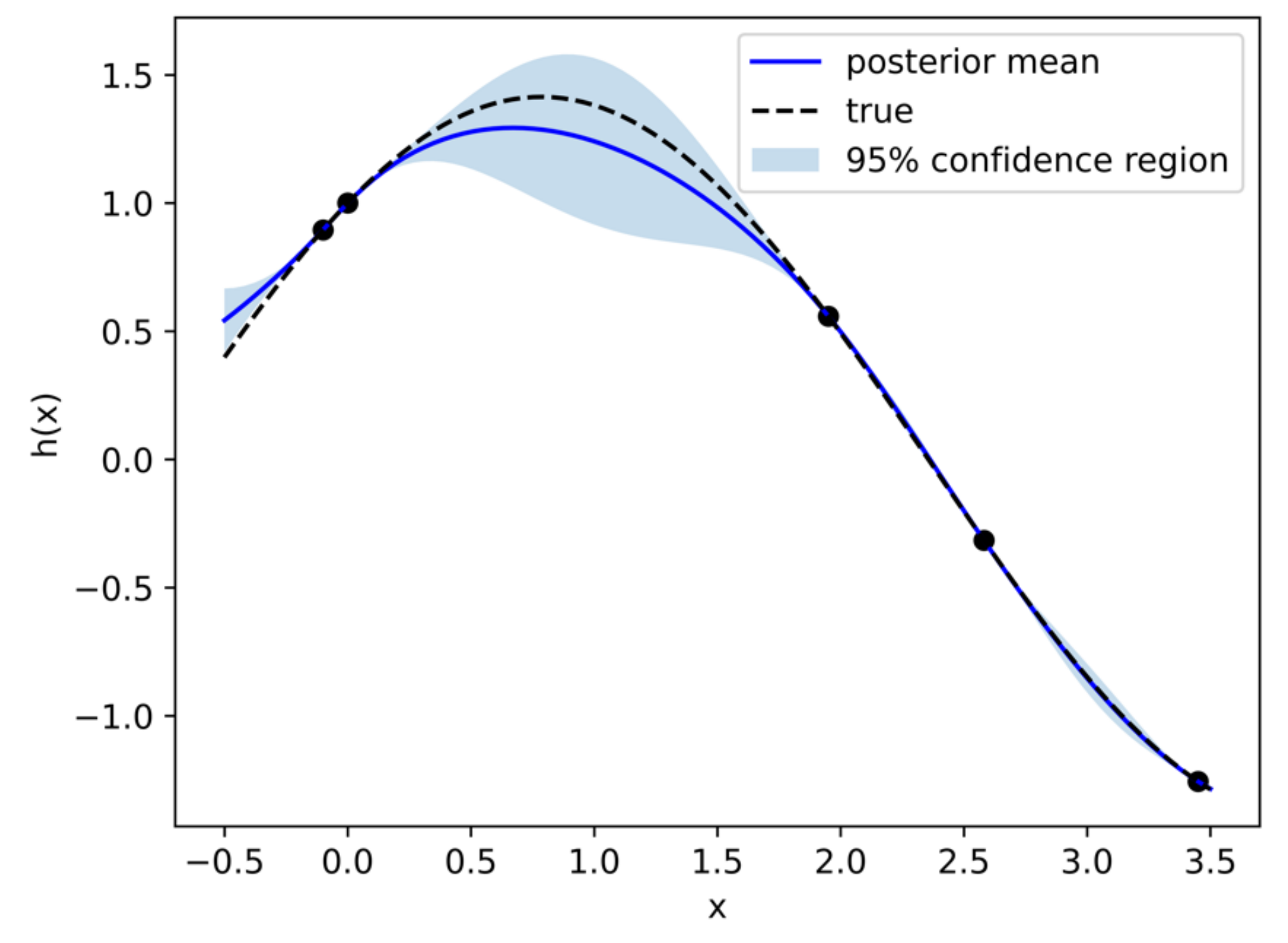}
        \caption{Posterior for $h(x)$ used by hybrid models}
    \end{subfigure}
    \hfill
    \begin{subfigure}[b]{0.46\textwidth}
        \centering
    \end{subfigure}
    \vskip\baselineskip
    \vspace{-2mm}
    \begin{subfigure}[b]{0.46\textwidth}
        \centering
        \includegraphics[width=\textwidth]{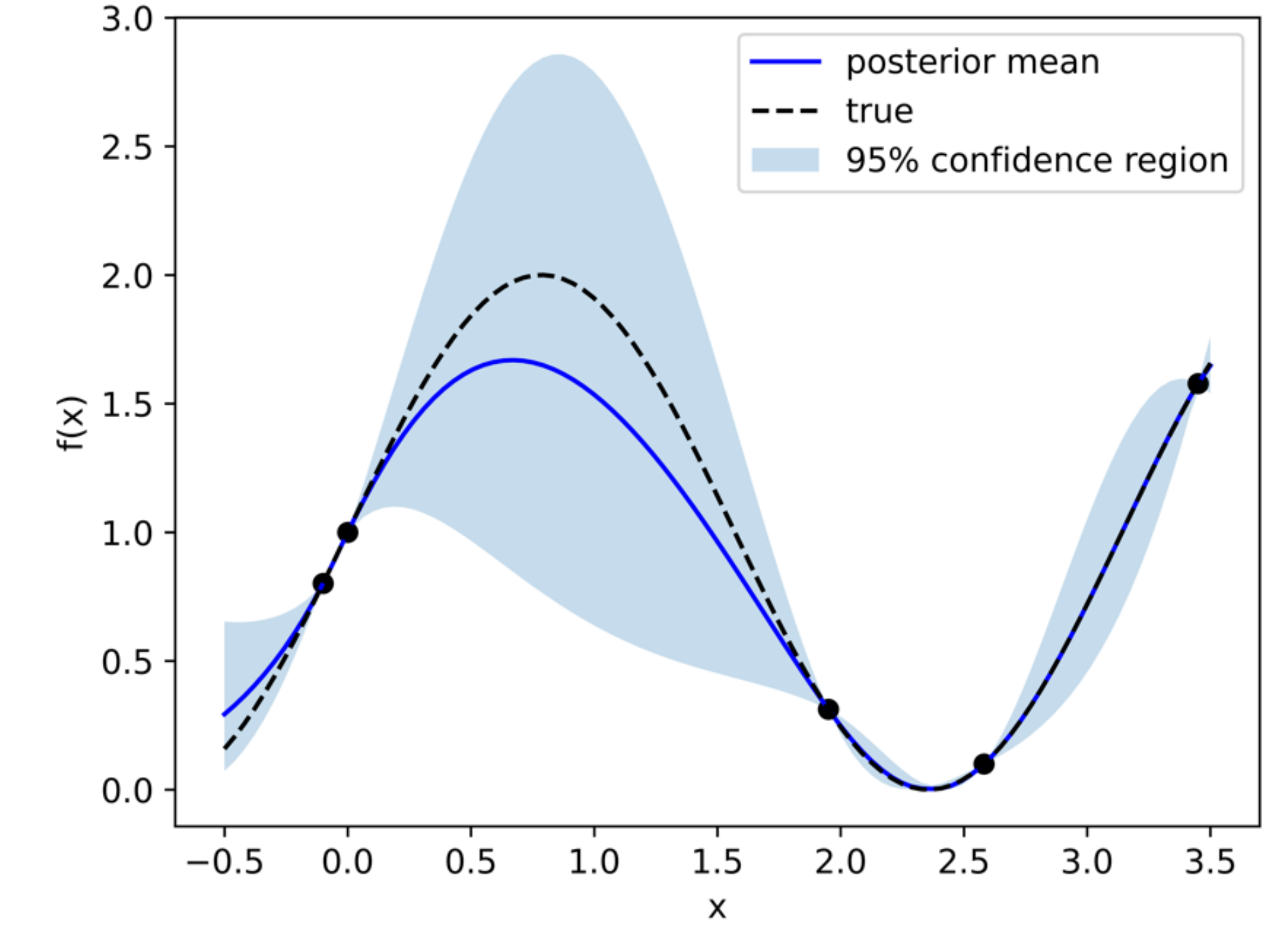}
        \caption{Implied posterior for $f_0(x)$}
    \end{subfigure}
    \hfill
    \begin{subfigure}[b]{0.46\textwidth}
        \centering
        \includegraphics[width=\textwidth]{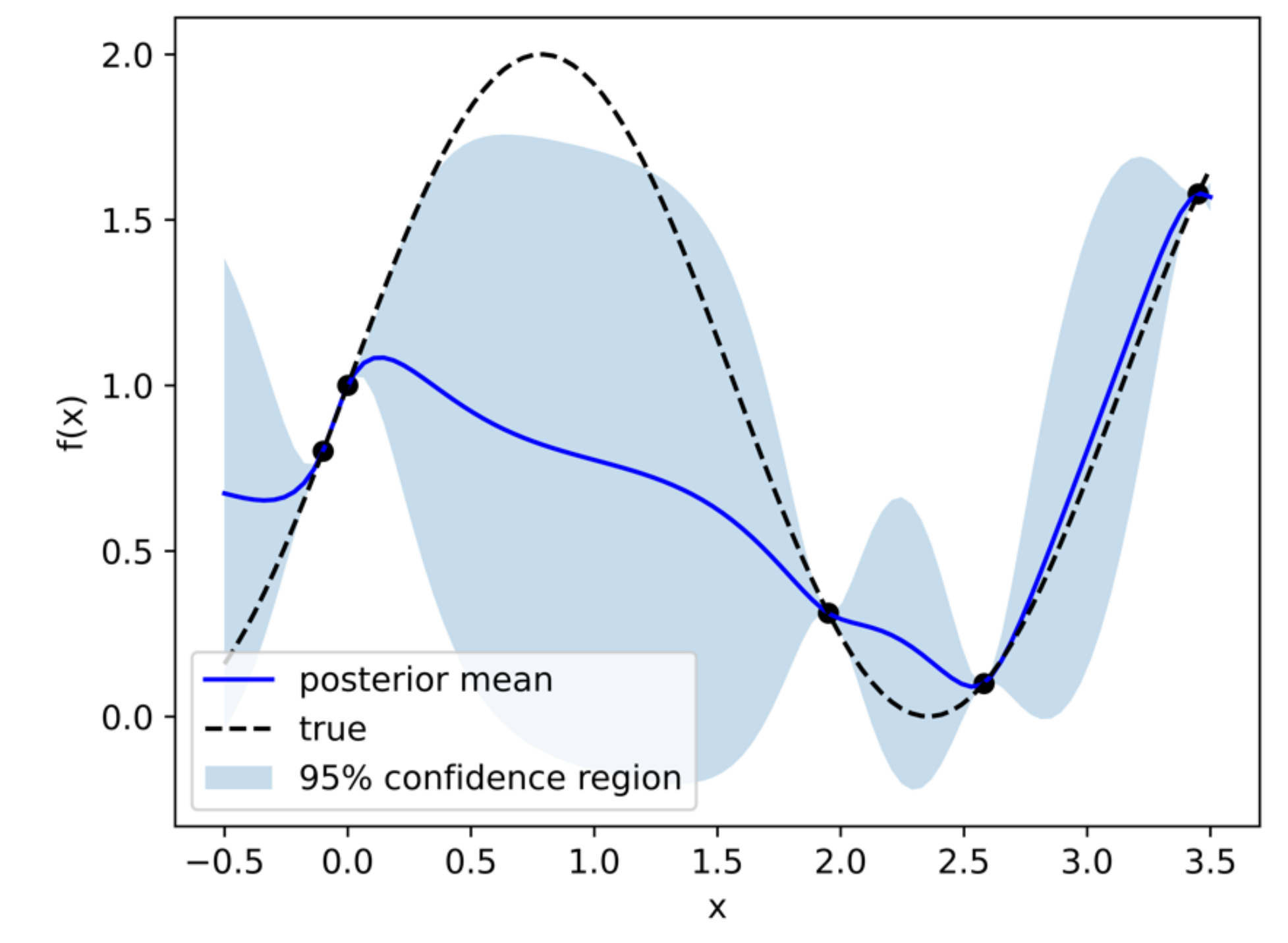}
        \caption{Posterior for $f_0(x)$ using black-box model}
    \end{subfigure}    
    \vskip\baselineskip
    \vspace{-2mm}
    \begin{subfigure}[b]{0.46\textwidth}
        \centering
        \includegraphics[width=\textwidth]{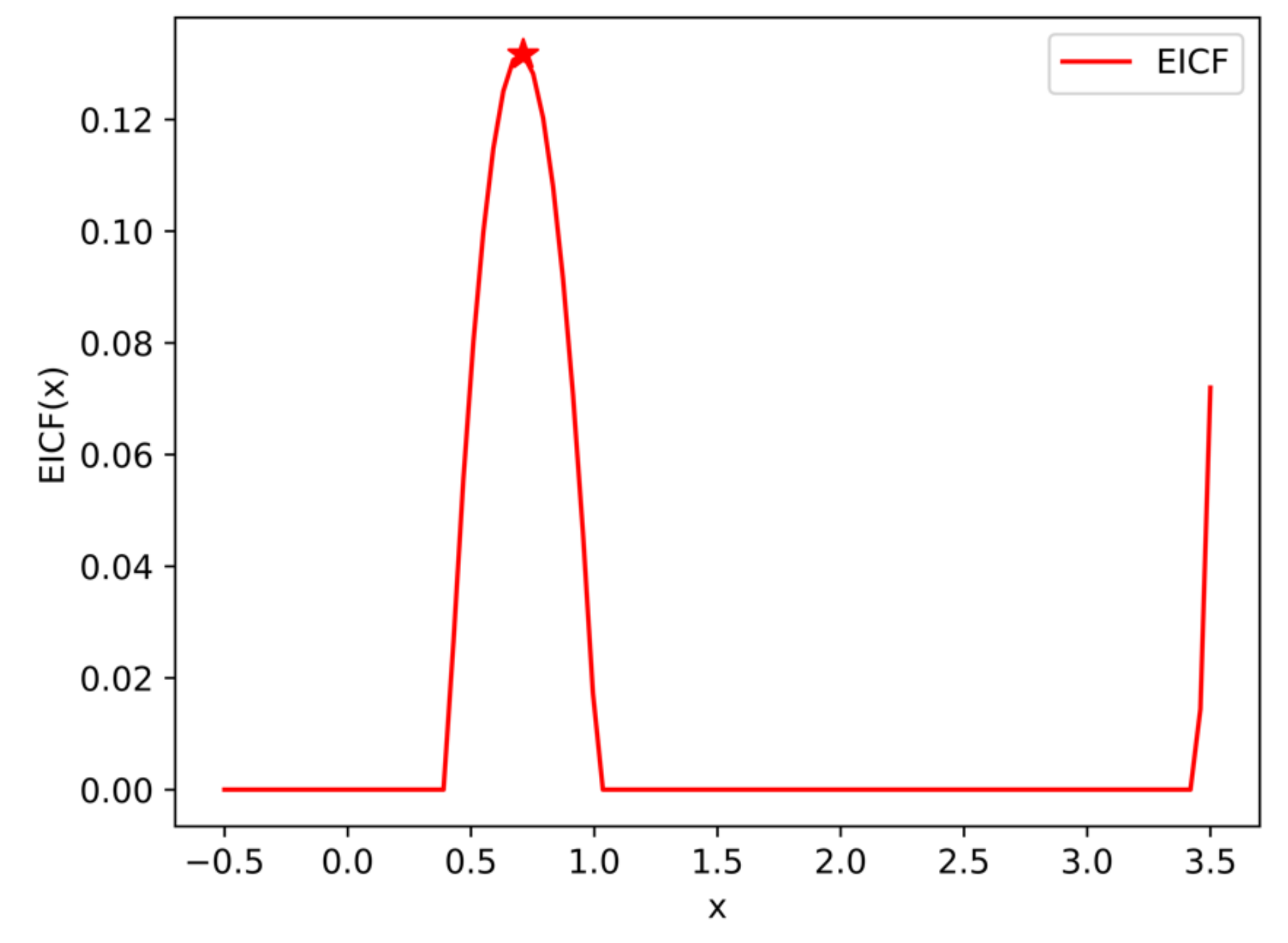}
        \caption{COBALT acquisition function}
    \end{subfigure}
    \hfill
    \begin{subfigure}[b]{0.46\textwidth}
        \centering
        \includegraphics[width=\textwidth]{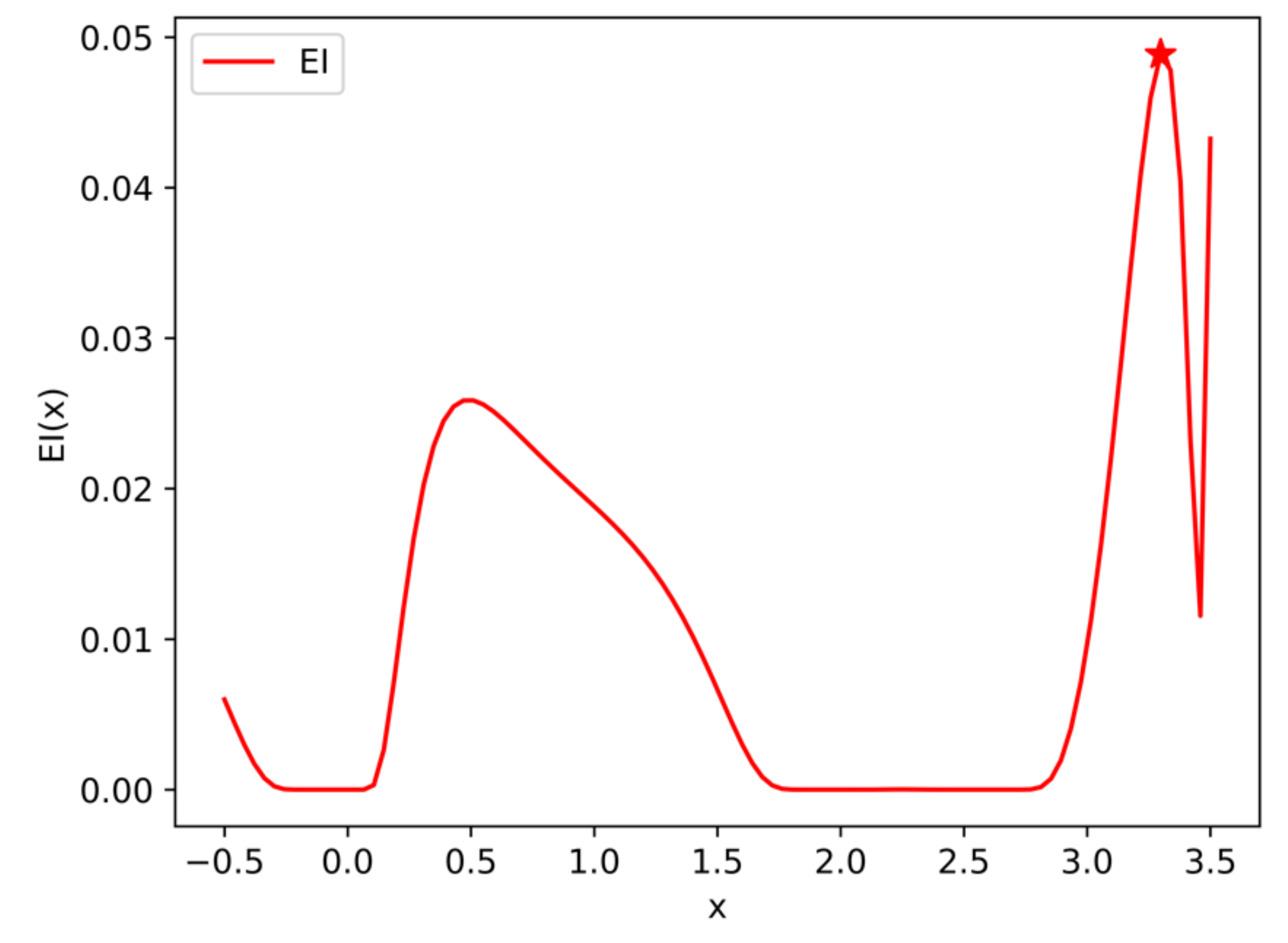}
        \caption{Classical EI acquisition function}
    \end{subfigure}
    \vskip\baselineskip
    \vspace{-2mm}
    \begin{subfigure}[b]{0.46\textwidth}
        \centering
        \includegraphics[width=\textwidth]{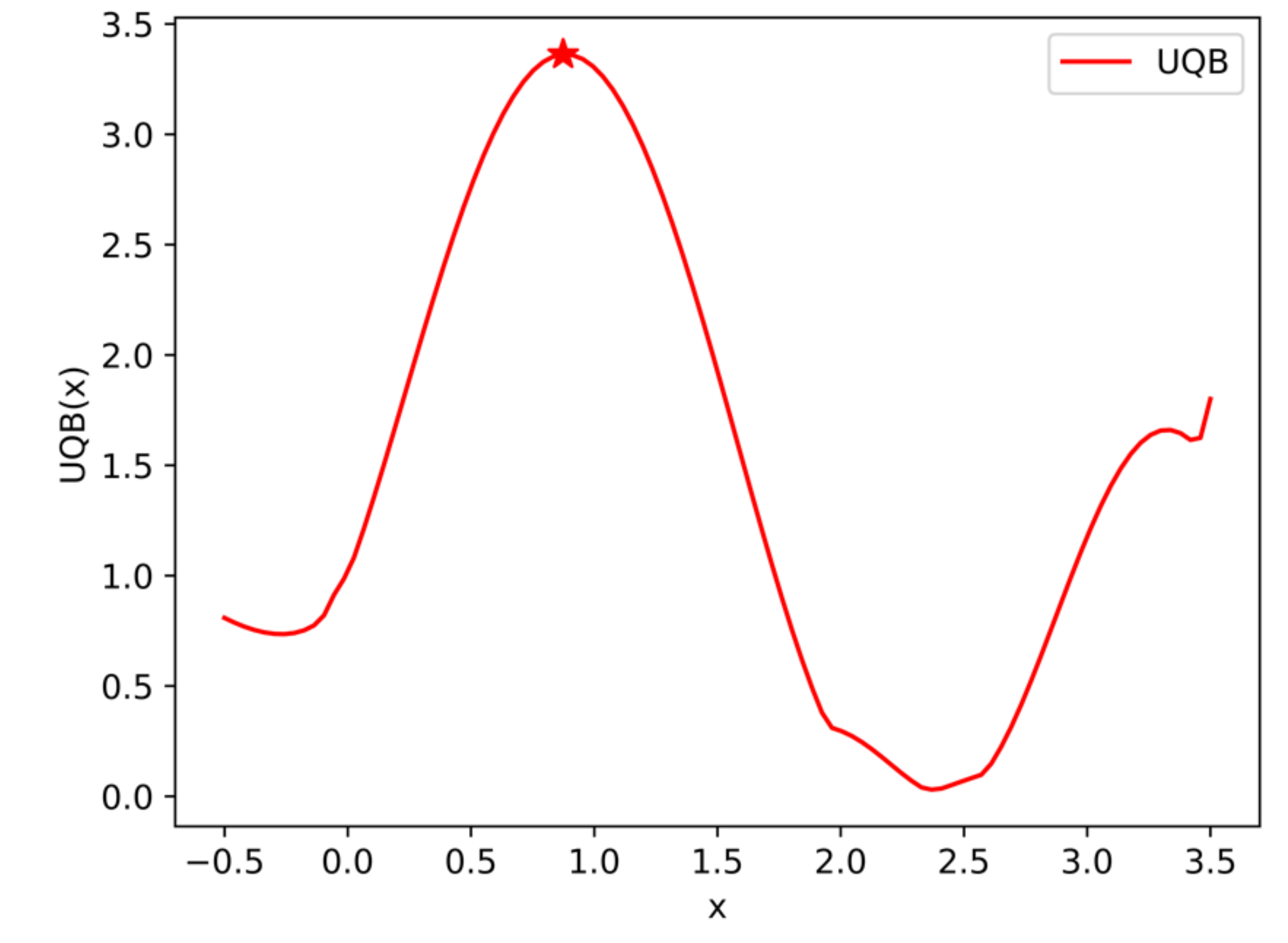}
        \caption{Proposed UQB acquisition function}
    \end{subfigure}
    \hfill
    \begin{subfigure}[b]{0.46\textwidth}
        \centering
        \includegraphics[width=\textwidth]{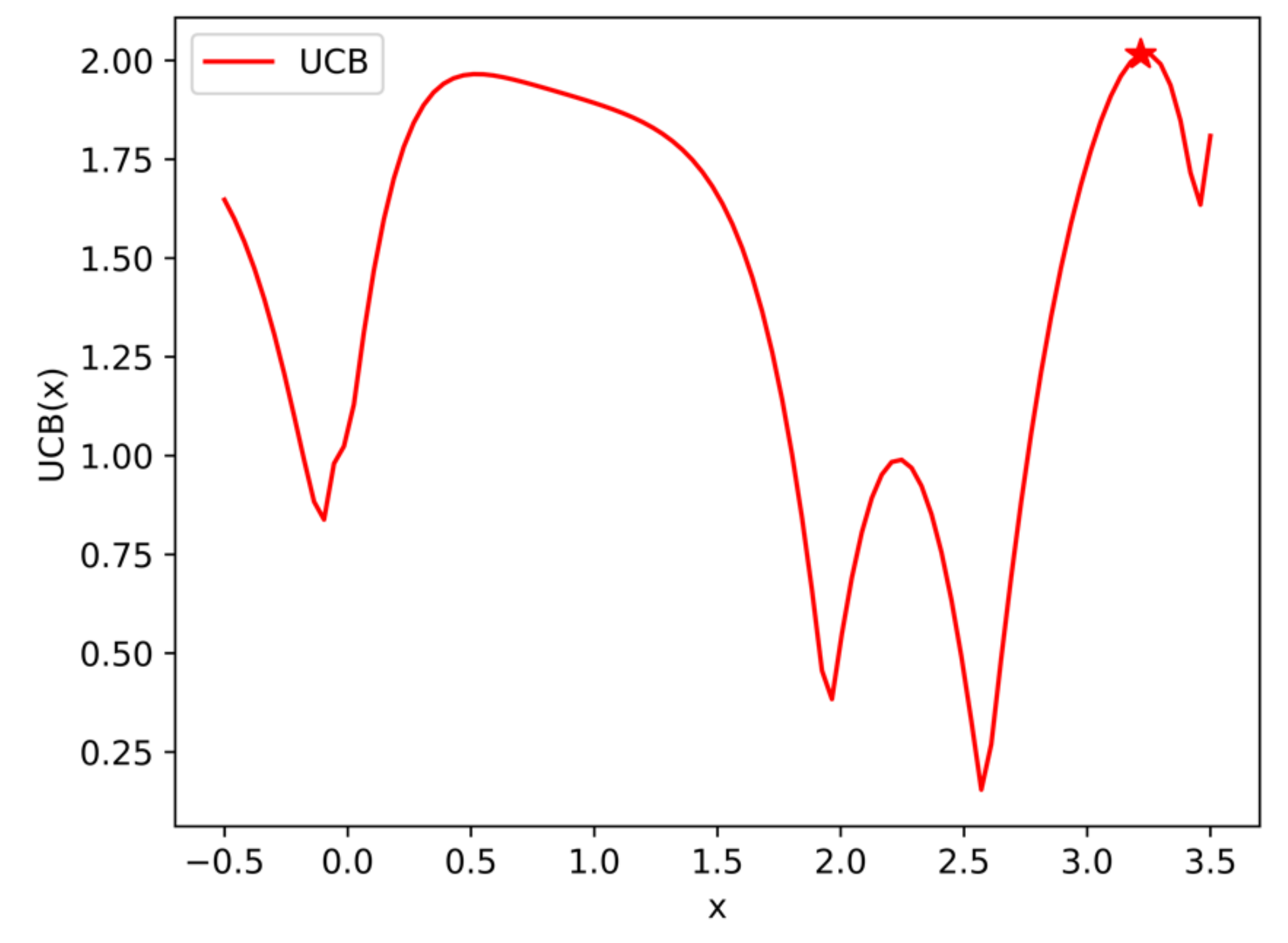}
        \caption{Classical UCB acquisition function}
    \end{subfigure}    
    \caption{Illustrative example comparing classical EI and UCB acquisition functions to their composite (grey-box) counterparts COBALT and UQB for a problem with scalar $h(x)$ and $f(x) = g(x,h(x)) = h(x)^2$. The observations of $h(x)$ provide more useful information about the location of potential global maxima, so that COBALT and UQB choose to evaluate points much closer to the global maxima.}
\label{fig:illustrative_example}
\end{figure}

Figure \ref{fig:constraints} demonstrates the value of CUQB from a constraint handling point of view by comparing its predicted feasible region (shown in red) to that of COBALT and EPBO (black-box version CUQB) for the following two-dimensional constraint
\begin{align}
    \{ x = (x_1,x_2) \in [-4,4]^2 : g(h(x_1,x_2)) \leq 100 \},
\end{align}
where $h(x_1,x_2) = (x_1^2 + x_2 - 11)^2 + (x_1 + x_2^2 - 7)^2$ is the unknown function and $g(h(x_1,x_2)) = h(x_1,x_2)^2$ is the known function.
The key observation is that the feasible region for COBALT is not guaranteed to contain the true feasible region (whose boundary is shown with white lines). The feasible region for EPBO does contain the true feasible region, however, due to the black-box representation of the functions, it has a significantly larger region than CUQB as a consequence of its higher degree of prediction uncertainty. CUQB thus achieves the best of both worlds: guaranteed containment of the true feasible region (with high probability) and reduced levels of uncertainty by exploiting prior structural knowledge. 

\begin{figure}[ht!] 
    \centering
    \begin{subfigure}[b]{0.49\textwidth}
        \centering
        \includegraphics[width=\textwidth]{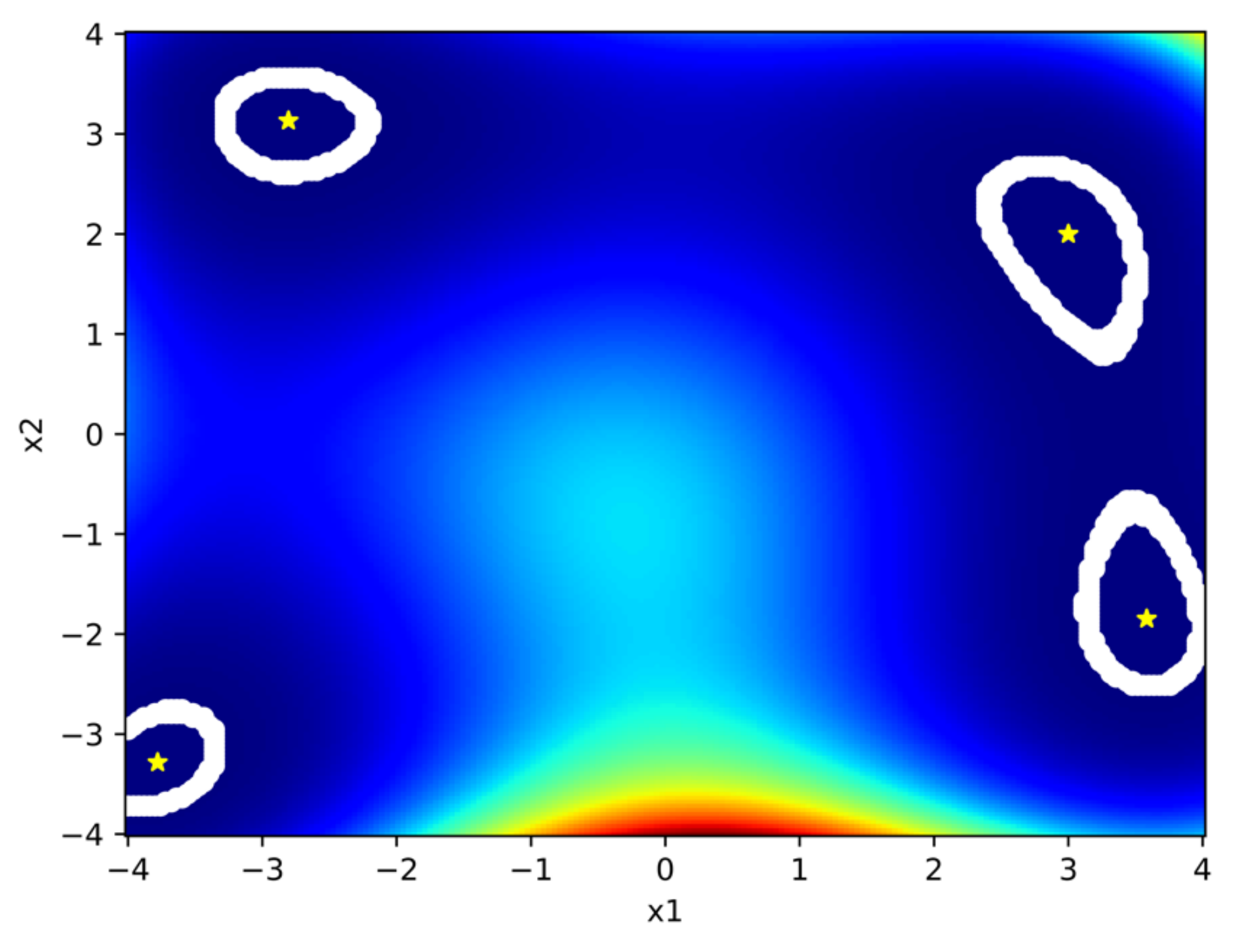}
        \caption{True feasible region}
    \end{subfigure}
    \hfill
    \begin{subfigure}[b]{0.49\textwidth}
        \centering
        \includegraphics[width=\textwidth]{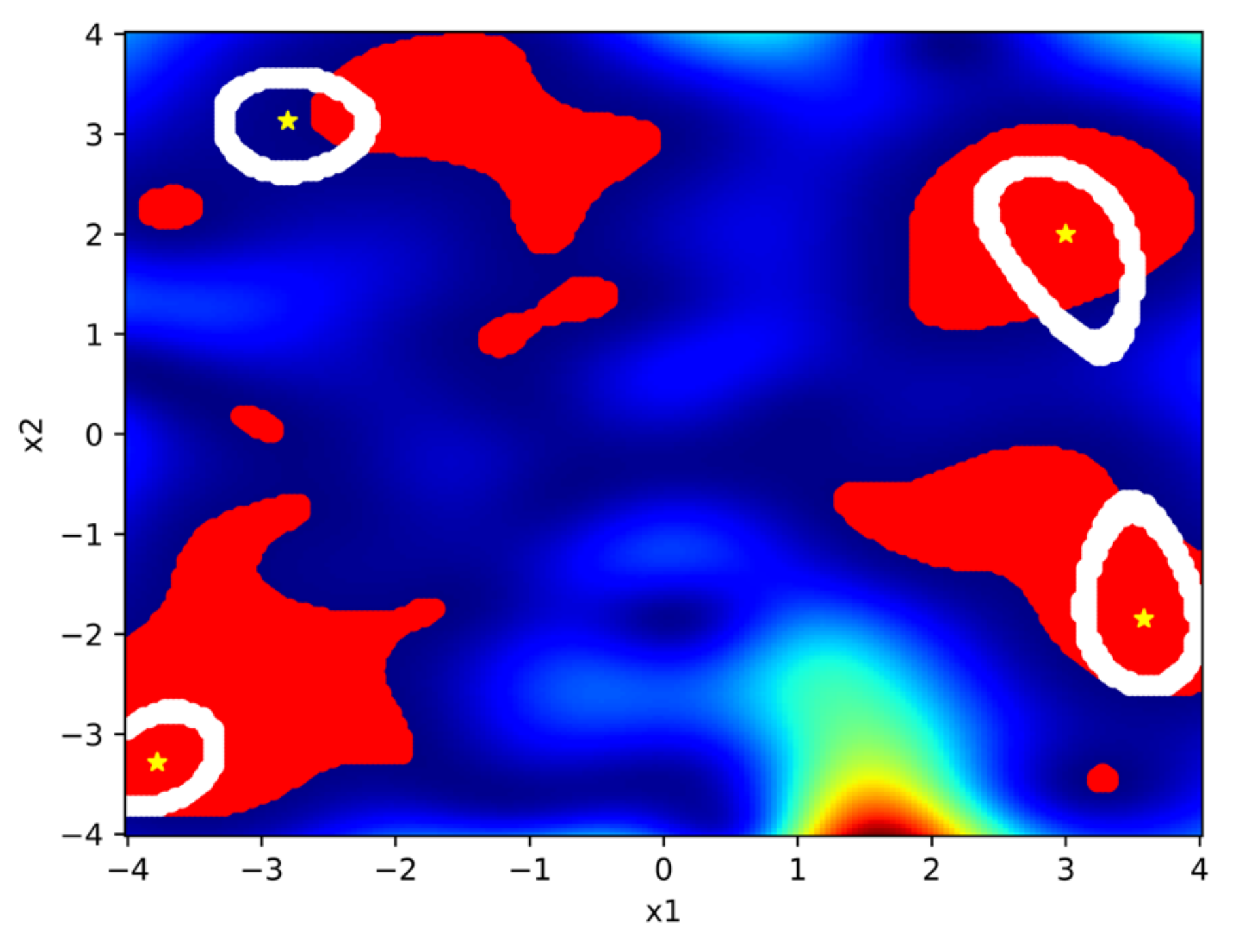}
        \caption{Predicted feasible region with COBALT}
    \end{subfigure}
    \vskip\baselineskip
    \vspace{-2mm}
    \begin{subfigure}[b]{0.49\textwidth}
        \centering
        \includegraphics[width=\textwidth]{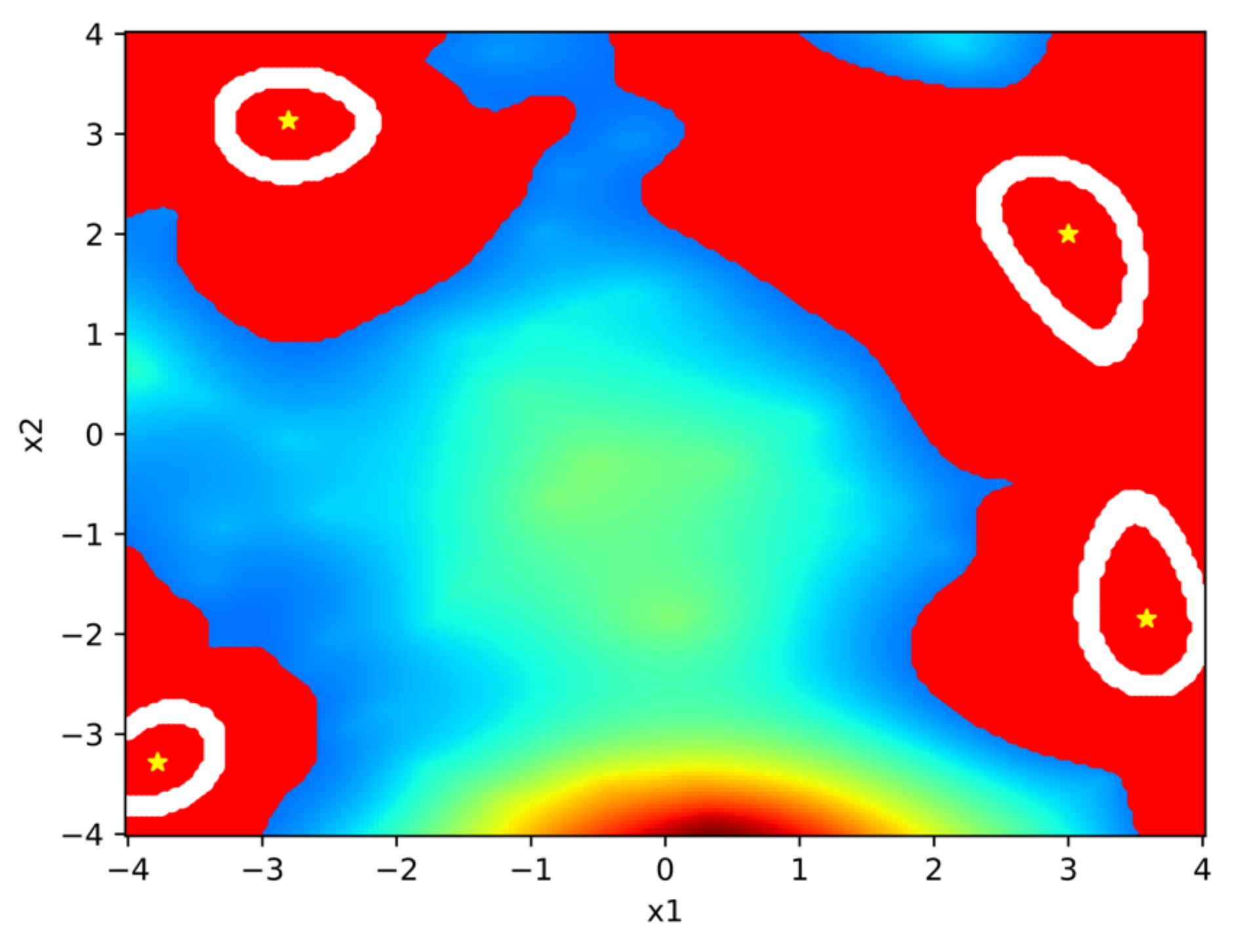}
        \caption{Predicted feasible region with EPBO}
    \end{subfigure}
    \hfill
    \begin{subfigure}[b]{0.49\textwidth}
        \centering
        \includegraphics[width=\textwidth]{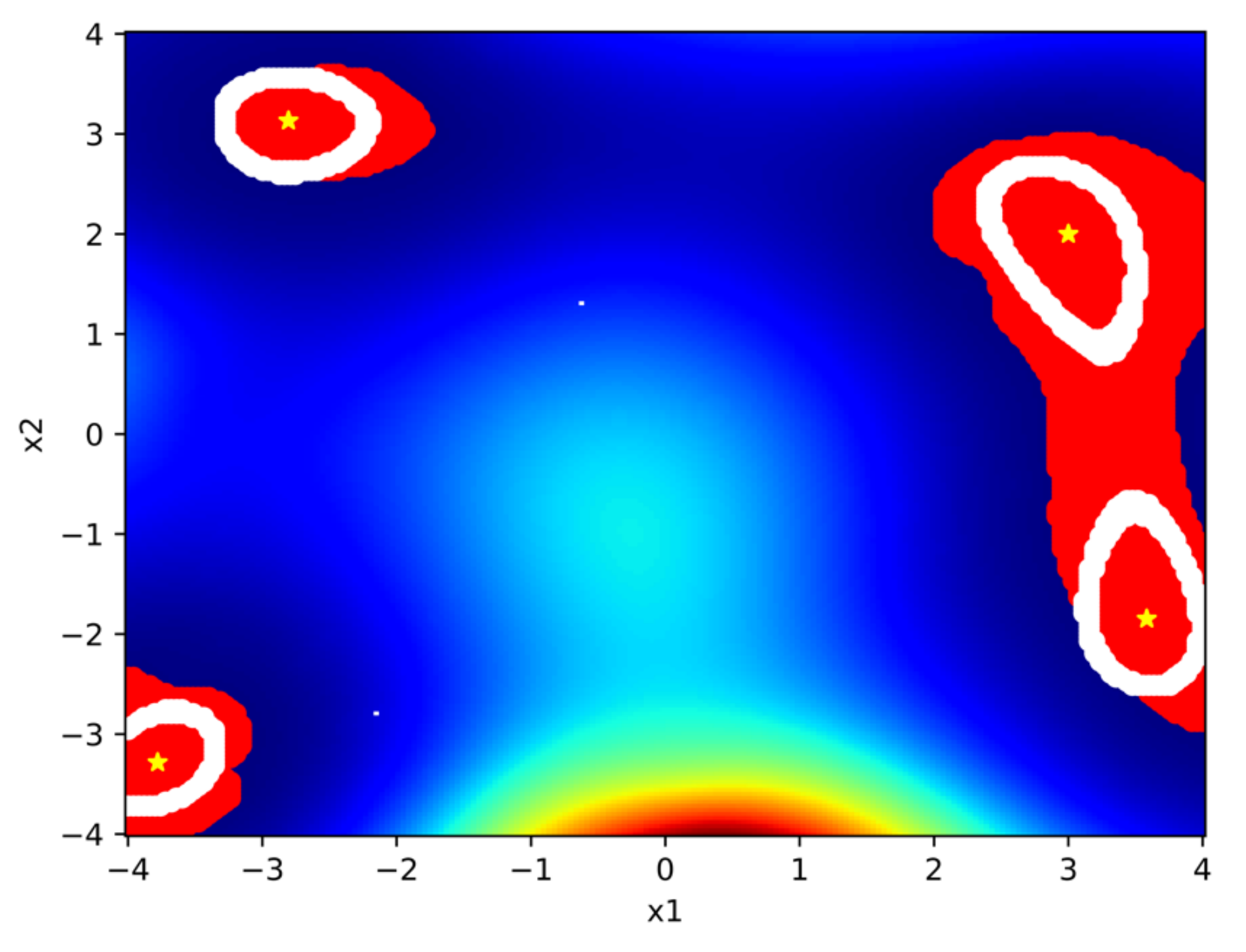}
        \caption{Predicted feasible region with CUQB}
    \end{subfigure}    
    \caption{Illustrative example comparing the constraint handling approaches in COBALT and EPBO to the proposed CUQB method. The filled red area corresponds to the predicted feasible region while the white lines represent the boundary of the true feasible region. We see that CUQB is able to more accurately contain the true feasible region than the alternative methods.}
\label{fig:constraints}
\end{figure}


\subsection{Practical optimization of CUQB using differentiable sorting}
\label{subsec:practicaloptcuqb}

The subproblem \eqref{eq:cuqb_subproblem} is difficult to solve since we do not even have an explicit expression for the quantile functions $\{ u_{i,t} \}_{i \in \mathbb{N}_0^n}$, meaning we cannot easily evaluate them for any specific $x \in \mathcal{X}$. This presents a huge barrier when attempting to apply guaranteed global optimization methods, such as spatial branch and bound \cite{epperly1997reduced}, that typically exploit structure of the function to compute upper and lower bounds in different domains of the search space. 
Instead, we resort to a more practical (heuristic) global optimization procedure. Multi-start methods that randomly sample initial conditions used to initialize a local gradient-based optimization algorithm are popular alternatives since they have been empirically shown to provide a good tradeoff between performance and computational cost. However, it turns out that multi-start is still challenging to apply in this context due to lack of a closed-form expression for the quantile functions. As such, we need to introduce two additional approximations that will allow us to build easy-to-differentiate representations of $\{ u_{i,t} \}_{i \in \mathbb{N}_0^n}$. In this section, we show that these approximations are convergent in the sense that they can be made to exactly match the true quantile functions as two independent parameters are refined. 

First, we must resort to some MC sampling approximation similar to that used for COBALT in \eqref{eq:cobalt-approx} since the distribution of $Y_{i,t}(x)$ is not known in closed-form. Estimating quantiles of a target response, however, is fundamentally different  than estimating its mean (expected value), as it provides additional information about the distribution of $Y_{i,t}(x)$. 
The most straightforward approach for estimating quantiles is based on the order statistics (sorted values) of a set of i.i.d. samples of $Y_{i,t}(x)$. Consider a set of $L$ i.i.d. samples from \eqref{eq:Yit}, i.e., $Y_{i,t}^{(l)}(x) = g_i(x, \mu_t(x) + C_t(x) Z^{(l)})$ for all $l \in \mathbb{N}_1^L$ where $Z^{(1)},\ldots,Z^{(L)} \sim \mathcal{N}(0,I_m)$. Let $\text{Sort} : \mathbb{R}^L \to \mathbb{R}^L$ denote an operator that sorts these samples in ascending order
\begin{align} \label{eq:sort}
    Y_{i,t}^{(1,L)}(x),\ldots, Y_{i,t}^{(L,L)}(x) = \text{Sort}(Y_{i,t}^{(1)}(x), \ldots Y_{i,t}^{(n)}(x)),
\end{align}
such that
\begin{align} \label{eq:sorted-sample}
    Y_{i,t}^{(1,L)}(x) \leq Y_{i,t}^{(2,L)}(x) \leq \cdots \leq Y_{i,t}^{(L,L)}(x).
\end{align}
We can then define the estimated (or empirical) quantile as 
\begin{align} \label{eq:estimated-quantile}
    \hat{Q}_{Y_{i,t}(x)}(p) = Y_{i,t}^{(\lceil pL \rceil, L)}(x),
\end{align}
where $\lceil a \rceil$ denotes the ceil operator (smallest integer greater than or equal to $a$). Therefore, the order $p$ approximated quantile in \eqref{eq:estimated-quantile} corresponds to the $\lceil pL \rceil^\text{th}$ sample in the sorted list \eqref{eq:sorted-sample}. 

One challenge remains when replacing the true UQB functions $\{ u_{i,t}(x) \}_{i \in \mathbb{N}_0^n}$ with their empirical counterparts $\{ \hat{u}^{(L)}_{i,t}(x) = \hat{Q}_{Y_{i,t}(x)}(1-\textstyle\frac{\alpha_{i,t}}{2}) \}_{i \in \mathbb{N}_0^n}$; the sorting operator in \eqref{eq:sort} is not differentiable everywhere such that it cannot be directly incorporated into gradient-based solvers. To better see this, we note that the sort operator can be formulated as a linear program (LP), as shown in \cite{blondel2020fast}. To keep the notation consistent, we introduce $s(\boldsymbol{\theta})$ as the sort operator in \textit{descending order} for any sequence $\boldsymbol{\theta} = (\theta_1, \ldots, \theta_L) \in \mathbb{R}^L$. This can be related to the required ascending sort operator in \eqref{eq:sort} via  $\text{Sort}(\boldsymbol{\theta}) = -s(-\boldsymbol{\theta})$. The LP formulation of $s(\boldsymbol{\theta})$ is
\begin{align} \label{eq:lp-sort}
    s( \boldsymbol{\theta} ) = \argmax_{ \boldsymbol{y} \in \mathcal{P}(\boldsymbol{\theta}) } \boldsymbol{y}^\top \boldsymbol{\rho},
\end{align}
where $\mathcal{P}(\boldsymbol{\theta})$ is the permutahedron of $\boldsymbol{\theta}$ (i.e., a convex polytope whose vertices correspond to permutations of $\boldsymbol{\theta}$) \cite{bowman1972permutation} and $\boldsymbol{\rho} = (L, L-1, \cdots, 1)$ is a vector of descending integers from $L$ to 1. The main idea for making a differentiable version of \eqref{eq:lp-sort} is to introduce a quadratic regularization term leading to the following soft sort operator
\begin{align}
    s_\varepsilon(\boldsymbol{\theta}) = \argmax_{ \boldsymbol{y} \in \mathcal{P}(\boldsymbol{\theta}) } \boldsymbol{y}^\top \boldsymbol{\rho} - \frac{\varepsilon}{2}\| \boldsymbol{y} \|^2 = \argmax_{ \boldsymbol{y} \in \mathcal{P}(\boldsymbol{\theta}) } \frac{1}{2}\left\| \boldsymbol{y} - \frac{\boldsymbol{\rho}}{\varepsilon} \right\|^2,
\end{align}
which corresponds to the Euclidean projection of $\boldsymbol{\rho}/\varepsilon$  onto $\mathcal{P}(\boldsymbol{\theta})$ for a given regularization strength $\varepsilon > 0$. By replacing the exact sort $s$ with the soft soft $s_\varepsilon$, we can efficiently compute derivatives of $\hat{u}^{(L)}_{i,t}(x)$ with respect to $x$, which enables \eqref{eq:cuqb_subproblem} to be solved relatively quickly using local gradient-based methods. We summarize the convergence properties of $\hat{u}^{(L)}_{i,t}(x)$ in the following proposition.
\begin{proposition} \label{prop:1}
The empirical quantile function $\hat{u}^{(L)}_{i,t}(x)$ in \eqref{eq:estimated-quantile} converges in probability to the true quantile function $u_{i,t}(x)$ for all $x \in \mathcal{X}$ as $L \to \infty$ and $\varepsilon \to 0$.
\end{proposition}

The proof of this result is given in Appendix \ref{appendix:new}. We investigate the impact of these approximations in Section \ref{sec:case-studies}; empirically we found that the approximations are quite accurate for $L$ values $\sim 50$ and $\varepsilon$ values $\sim 0.1$. Since the soft sort operator only requires $O(L\log L)$ complexity for the forward pass and $O(L)$ complexity for the backward differentiation pass, the cost of a single local optimization algorithm remains quite low in practice. 



\subsection{Simplification for linear transformations}

For general nonlinear transformation functions $\{ g_i \}_{i \in \mathbb{N}_0^n}$, we need to use the empirical quantile function \eqref{eq:estimated-quantile} described in the previous section. However, whenever these functions are linear with respect to their second argument, we can develop closed-form solutions similar to the classical UCB and its constrained extensions. 
We summarize this result in the following proposition.
\begin{proposition} \label{prop:2}
Suppose that $g_i(x, y) = a_i^\top(x) y + b_i(x)$ for some known functions $a_i : \mathbb{R}^d \to \mathbb{R}^m$ and $b_i : \mathbb{R}^d \to \mathbb{R}$ for all $i \in \mathbb{N}_0^n$. Then, the lower and upper quantile bound functions for $Y_{i,t}(x)$ in \eqref{eq:quantile-bounds} have the following closed-form representation
\begin{subequations}
\begin{align}
    l_{i,t}(x) &= \mu^Y_{i,t}(x) - \beta_{i,t}^{1/2} \sigma_{i,t}^Y(x), \\
    u_{i,t}(x) &= \mu^Y_{i,t}(x) + \beta_{i,t}^{1/2} \sigma_{i,t}^Y(x),
\end{align}
\end{subequations}
where $\mu_{i,t}^Y(x) = a_i^\top(x) \mu_{t}(x) + b_i(x)$, $\sigma_{i,t}^Y(x) = \sqrt{a_i^\top(x) \Sigma_{t}(x) a_i(x)}$, $\beta_{i,t}^{1/2} = \Phi^{-1}(1-\textstyle\frac{\alpha_{i,t}}{2})$, and $\Phi^{-1}$ is the inverse standard normal cumulative distribution function. 
\end{proposition}

The proof of Proposition \ref{prop:2} is given in Appendix \ref{appendix:new}. Proposition \ref{prop:2} does not imply that our approach is equivalent to constrained UCB methods, even when $\{ g_i \}_{i \in \mathbb{N}_0^n}$ are linear. This is because the posterior distribution of the function $a^\top(x) h(x) + b(x)$ given observations of $h(x)$ is different from the one given only lumped observations of $a^\top(x) h(x) + b(x)$. In fact, it has been shown that the former approach necessarily leads to a reduction in variance compared to the latter approach \cite[Theorem 1]{wang2020improving}, such that this allows us to determine tighter confidence bounds enabling improved search performance.


\section{Theoretical Analysis of CUQB}
\label{sec:theory}

In this section, we analyze the theoretical performance of CUQB (Algorithm \ref{alg:cuqb}). Our goal is to establish bounds on two important performance metrics: cumulative regret and cumulative constraint violation.
We then derive bounds on the cumulative performance metrics as a function of the total number of iterations $T$, which allow us to establish bounds on the rate of convergence to the optimal solution. Finally, we show that, with high probability (at least $1-\delta$ where $\delta \in (0,1)$ is a parameter chosen by the user), CUQB will declare infeasibility in a finite number of iterations if the original problem is infeasible. We note that the results presented in this section hold under the assumption that we identify a global maximizer of \eqref{eq:cuqb_subproblem} at every iteration. In Section \ref{sec:case-studies}, we empirically show that strong performance results can still be achieved using effective heuristic methods to solve this problem.

Note that the proofs of all results in this section are provided in Appendix \ref{appendix:A} to streamline the presentation.


\subsection{Performance metrics}

As in the standard bandit feedback setting, we want to minimize the gap between our queried point $f_0(x_{t+1})$ and the true solution $f_0^\star = \max_{x \in \mathcal{X}, f_i(x) \geq 0, \forall i \in \mathbb{N}_1^n} f_0(x)$, which is often referred to as the \textit{instantaneous regret}
\begin{align} \label{eq:instant-regret}
    r_{t+1} = f_0^\star - f_0(x_{t+1}),
\end{align}
where $x_{t+1}$ is the point queried by Algorithm \ref{alg:cuqb} at step $t \geq 0$. In the presence of unknown constraints, we may sample infeasible points that have a larger value than $f_0^\star$ such that we further define the positive instantaneous regret as follows
\begin{align}
    r_{t+1}^+ = \left[ f_0^\star - f_0(x_{t+1}) \right]^+ = \max\{ f_0^\star - f_0(x_{t+1}), 0 \}.
\end{align}
Since we are also interested in finding points that satisfy the unknown constraints, we define the violation of constraint $i \in \mathbb{N}_1^n$ at step $t \geq 0$ as
\begin{align} \label{eq:instant-violation}
    v_{i, t+1} = \left[ -f_i(x_{t+1}) \right]^+ = \max\{ -f_i(x_{t+1}),0\} = -\min\{ f_i(x_{t+1}),0 \}.
\end{align}
Ideally, we would achieve zero regret and constraint violation in a single step; however, this is only possible when $h(x)$ is perfectly known. Instead, we will focus on analyzing their cumulative values as the algorithm progresses.

\begin{definition}[Cumulative regret]
    Given a total of $T$ steps of CUQB, we define the cumulative regret and cumulative positive regret over $T$ to be
    \begin{subequations}        
    \begin{align}
        R_T = \textstyle\sum_{t=0}^{T-1} r_{t+1}, \\
        R_T^+ = \textstyle\sum_{t=0}^{T-1} r^+_{t+1}.
    \end{align}
    \end{subequations}
\end{definition}

\begin{definition}[Cumulative constraint violation]
    Given a total of $T$ steps of CUQB, we define the cumulative constraint violation of constraint $i$ over $T$ to be
    \begin{align}
        V_{i,T} = \textstyle\sum_{t=0}^T v_{i,t+1}, ~~~ \forall i \in \mathbb{N}_1^n.
    \end{align}
\end{definition}


\subsection{Required assumptions}

To derive bounds on $R_T$ and $V_{i,T}$ based on the sequence of samples selected by CUQB, we must first make some assumptions on our problem setup.

\begin{assumption} \label{assump:1}
The unknown function $h(x)$ is a sample of a GP with mean function $\mu_0(x)$ and covariance function $K_0(x,x')$ and its observations are corrupted by zero mean Gaussian noise. Without loss of generality, we will assume that the prior mean function is zero, i.e., $\mu_0(x) = 0$. 
\end{assumption}

Assumption \ref{assump:1} is required to ensure that the data collected on $h$ can be used to construct a reasonable surrogate model. As noted previously, the choice of the kernel function is very flexible, which enables such multi-output GP models to capture a large class of functions, as long as they satisfy some degree of smoothness/continuity. We will establish convergence for the most popular kernel choices later. 

\begin{assumption} \label{assump:2}
The feasible set $\mathcal{X}$ has finite cardinality $| \mathcal{X} | < \infty$.
\end{assumption}

Assumption \ref{assump:2} is made to simplify the construction of high-probability bounds on the predictions of the objective and constraint functions. This is a common assumption in the BO literature, as discussed in \cite{srinivas09}. There has been work on relaxing this assumption in the case of black-box functions, e.g., \cite{chowdhury2017kernelized,vakili2021information,wuthrich2021regret}, though how to effectively extend these bounds to the case of composite functions remains an open question that we plan to address in our future work. Although this may seem like a very restrictive assumption, it is only needed to simplify the theoretical analysis and thus does not prevent CUQB from being practically applied to continuous sets. Furthermore, we will show that the cumulative regret and constraint violation bounds exhibit a logarithmic dependence on the cardinality $| \mathcal{X} |$ such that we do not expect substantial performance losses in the case of continuous $\mathcal{X}$.

\begin{assumption} \label{assump:3}
The optimization problem \eqref{eq:grey-box-opt} is feasible and has an optimal solution $x^\star$ with corresponding optimal objective value $f^\star_0 = f_0(x^\star)$.
\end{assumption}

We initially make Assumption \ref{assump:3} to focus on the feasible case first. The behavior of CUQB when this assumption is relaxed is studied in Section \ref{subsec:infeas}. 

\begin{assumption} \label{assump:4}
The functions $\{ g_i(x,y) \}_{i \in \mathbb{N}_0^n}$ are Lipschitz continuous with respect to their second argument, i.e., 
\begin{align}
    | g_i(x, y) - g_i(x, z) | \leq L_i \| y - z \|_1, ~~\forall x \in \mathbb{R}^d,~ \forall y,z \in \mathbb{R}^m, ~\forall i \in \mathbb{N}_0^n,
\end{align}
where $L_i < \infty$ is a finite Lipschitz constant.
\end{assumption}

Assumption \ref{assump:4} ensures that the outer transformations of the unknown function $h(x)$ are well-behaved in the sense that the distance between the outputs are bounded some constant times the distance between the inputs. 

\subsection{Cumulative regret and violation probability bounds}

Our derived bounds depend on the \textit{maximum information gain} (MIG), which is a fundamental quantity in Bayesian experimental design \cite{chaloner1995bayesian} that provides a measure of informativeness of any finite set of sampling points. 

\begin{definition}[Maximum information gain]
Let $\mathcal{A} \subset \mathcal{X}$ denote any potential subset of points sampled from $\mathcal{X}$. The maximum information gain for the $j^\text{th}$ element of $h$ given $t$ noisy measurements is
\begin{align} \label{eq:information-gain}
\gamma_{j,t} &= \max_{ \mathcal{A} \subset \mathcal{X} : |\mathcal{A}| = t} \frac{1}{2} \log \det\left( I + \sigma^{-2} \boldsymbol{K}_{j,\mathcal{A}} \right),
\end{align}
where $\boldsymbol{K}_{j,\mathcal{A}} = [k_{j,0}(x, x')]_{x, x' \in \mathcal{A}}$ where $k_{j,0}$ is the $j^\text{th}$ diagonal element of $K_0$. 
\end{definition}

The first key step is to show that the lower and upper quantile bounds can be guaranteed to bound the true function with high probability for proper selection of the quantile parameters $\{ \alpha_{i,t} \}$. 

\begin{lemma} \label{lemma:1}
Let Assumptions \ref{assump:1} and \ref{assump:2} hold. Then,
\begin{align}
    \textnormal{Pr} \left\lbrace g_i(x, h(x)) \in [ l_{i,t}(x) , u_{i,t}(x) ], ~~\forall x \in \mathcal{X}, \forall t \geq 0, \forall i \in \mathbb{N}_0^n \mid \mathcal{D}_t \right\rbrace \geq 1 - \delta,
\end{align}
where $l_{i,t}(x)$ and $u_{i,t}(x)$ are given by \eqref{eq:quantile-bounds} with $\alpha_{i,t} = 6\delta/(\pi^2 (t+1)^2 (n+1) | \mathcal{X} |)$.
\end{lemma}


Next, we can use Lemma \ref{lemma:1} to derive bounds on the instantaneous regret \eqref{eq:instant-regret} and constraint violation \eqref{eq:instant-violation} terms that are related to the width between the upper and lower quantiles for the given sample point $x_{t+1}$.

\begin{lemma} \label{lemma:2}
    Let Assumptions \ref{assump:1}, \ref{assump:2}, and \ref{assump:3} hold. Then, with probability at least $1-\delta$, feasibility is never declared in Line 3 of Algorithm \ref{alg:cuqb} and for all $t \geq 0$
    \begin{subequations}        
    \begin{align}
        r_{t+1} \leq r_{t+1}^+ &\leq w_{0,t}(x_{t+1}), \\
        v_{i,t+1} &\leq w_{i,t}(x_{t+1}), ~~~ \forall i \in \mathbb{N}_1^n, 
    \end{align}
    \end{subequations}
    where $w_{i,t}(x_{t+1}) = u_{i,t}(x_{t+1}) - l_{i,t}(x_{t+1})$ is the width of the $\alpha_{i,t}$-level quantile bound evaluated at the sample point $x_{t+1}$.
\end{lemma}


Lemma \ref{lemma:2} shows that both the regret and violation must be less than the quantile width $w_{i,t}(x_{t+1})$ at every $t \geq 0$ of Algorithm \ref{alg:cuqb}. However, it is not obvious if $w_{i,t}(x_{t+1})$ is decaying at every iteration, which would imply that we are effectively ``closing in'' on the optimal point. Therefore, we next establish a relationship between $w_{i,t}(x_{t+1})$ and the standard deviations of the underlying GP predictions, which will be easier to show must decay as iterations increase.

\begin{lemma} \label{lemma:3}
    Let $Y = g(X)$ be a nonlinear transformation of a normal random vector $X \sim \mathcal{N}(\mu_X, \Sigma_X)$ where $g : \mathbb{R}^m \to \mathbb{R}$ is Lipschitz continuous with constant $L_g$. Then, 
    \begin{align}
        Q_Y(1 - \textstyle\frac{\alpha}{2}) - Q_Y(\textstyle\frac{\alpha}{2}) \leq 2L_g \beta^{1/2} \sum_{j=1}^m \sigma_{X,j},
    \end{align}
    where $\sigma_{X,j}$ is the $j^\text{th}$ diagonal element of $\Sigma_X$, $\alpha \in (0,1)$, and $\beta = 2\log(2m/\alpha)$.
\end{lemma}


Now that we have a connection between quantiles of a transformed random variable and the standard deviation of the underlying Gaussian random vector in Lemma \ref{lemma:3}, it is useful to establish a bound on the cumulative sum of the posterior standard deviations in terms of the MIG. 

\begin{lemma}[Lemma 4, \cite{chowdhury2017kernelized}] \label{lemma:4}
    Let $x_1,\ldots, x_T$ be a sequence of query points selected by an algorithm. Then, the following scaling laws hold for common choices of kernels
    \begin{align}
        \textstyle\sum_{t=0}^{T-1} \sigma_{j,t}(x_{t+1}) \leq 2\sqrt{(T+2)\gamma_{j,T}}, ~~ \forall j \in \mathbb{N}_1^m.
    \end{align}
\end{lemma}

The relationship in Lemma \ref{lemma:4} is important because there are known iteration-dependent bounds on the MIG that have been established for the commonly used kernels. These bounds were derived in \cite{srinivas09}, which we repeat below for completeness.

\begin{lemma}[Theorem 5, \cite{srinivas09}] \label{lemma:5}
Let $\mathcal{X} \subset \Omega$ be a subset of a compact and convex set $\Omega$, $d \in \mathbb{N}$, and assume $k_{j,0}(x,x') \leq 1$. Then,
\begin{itemize}
\item Linear: $\gamma_{j,T} = \mathcal{O}(d \log T)$; 
\item Squared exponential: $\gamma_{j,T} = \mathcal{O}((\log T)^{d+1})$; 
\item Matern ($\nu > 1$): $\gamma_{j,T} = \mathcal{O}(T^{d(d+1)/(2\nu+d(d+1))}(\log T))$.
\end{itemize}
\end{lemma}

We are now in a position to state our main theorem, which establishes explicit bounds on the cumulative regret and constraint violation as a function of the total number of iterations $T$. 

\begin{theorem} \label{thm:2}
    Under Assumptions \ref{assump:1}, \ref{assump:2}, \ref{assump:3}, and \ref{assump:4}, we have, with probability at least $1 - \delta$, that the sample points of CUQB (Algorithm \ref{alg:cuqb}) satisfy
    \begin{subequations}        
    \begin{align}
        R_T \leq R_T^+ &\leq 4 L_{0} \beta^{1/2}_{T} \Psi_T = \tilde{\mathcal{O}}(\sqrt{\gamma^\textnormal{max}_{T} T}), \\
        V_{i,T} &\leq 4 L_{i} \beta^{1/2}_{T} \Psi_T = \tilde{\mathcal{O}}(\sqrt{\gamma^\textnormal{max}_{T} T}), ~~ \forall i \in \mathbb{N}_1^n,
    \end{align}
    \end{subequations}
    where $\Psi_T = \textstyle\sum_{j=1}^m \sqrt{(T+2)\gamma_{j,T}}$, $\beta_{t} = 2\log( m\pi^2t^2(n+1)|\mathcal{X}| / (3\delta) )$, and $\gamma^\textnormal{max}_{T} = \max_{j \in \mathbb{N}_1^m} \gamma_{j,T}$ is the worst-case (slowest converging) MIG for unknown function $h$. Here, the notation $\tilde{\mathcal{O}}$ refers to a variation of the standard order of magnitude that suppresses logarithmic factors. 
\end{theorem}


We can combine the results shown in Theorem \ref{thm:2} with Lemma \ref{lemma:5} to derive kernel specific bounds on the growth of the cumulative regret and constraint violation as the number of iterations $T$ increases. Intuitively, one would expect that a \textit{sublinear} growth in $T$ would allow for convergence since this implies the regret/violation gap must be shrinking, i.e., $R_T/T \to 0$ and $V_{i,T}/T \to 0$ as $T \to \infty$. We more formally study the impact of Theorem \ref{thm:2} on convergence rate with and without noise next. 

\subsection{Convergence rate bounds}
\label{subsec:convergence-rate-bounds}

With slight abuse of notation, we let $r(x_t) = f_0^\star - f_0(x_t)$ and $v_{i}(x_t)$ denote the instantaneous regret and constraint violation as a function of $x_t$ (as opposed to being indexed by $t$). We can then establish the following bounds on the rate of convergence of CUQB to the optimal solution.

\begin{theorem} \label{thm:3}
    Let Assumptions \ref{assump:1}, \ref{assump:2}, \ref{assump:3}, and \ref{assump:4} hold and $\mathcal{L} = \sum_{i=0}^n L_i$. Then, with probability at least $1 - \delta$, there must exist a point $\tilde{x}_T \in \{ x_1,\ldots,x_T \}$ contained in the sequence generated by CUQB such that the following inequalities hold
    \begin{subequations}        
    \begin{align}
        f_0^\star - f_0(\tilde{x}_T) &\leq \frac{4\mathcal{L}\beta_{T}^{1/2} \Psi_T}{T} = \tilde{\mathcal{O}}\left( \sqrt{\frac{\gamma^\textnormal{max}_{T}}{T}} \right), \\
        \left[ -f_i(\tilde{x}_T) \right]^+ &\leq \frac{4\mathcal{L}\beta_{T}^{1/2} \Psi_T}{T} = \tilde{\mathcal{O}}\left( \sqrt{\frac{\gamma^\textnormal{max}_{T}}{T}} \right), ~~ \forall i \in \mathbb{N}_1^n.
    \end{align}
    \end{subequations}
\end{theorem}


Theorem \ref{thm:3} provides explicit convergence rate bounds to the constrained optimal solution in terms of the MIG for general kernels. However, one important point is that the result holds for $\tilde{x}_T = x_{t^\star}$ with $t^\star = \argmin_{t \in \{ 1, \ldots, T \}} ( r^+_{t+1} + \sum_{i=1}^n v_{i,t+1} )$. In practice, we cannot identify $t^\star$ since it depends on $f_0^\star$, which is unknown. To address this challenge, we proposed a practical recommendation procedure in \eqref{eq:recommended-point}. This procedure is derived from the fact, in the absence of noise, we can identify the index that minimizes a penalized version of this sequence, i.e., 
\begin{align} \label{eq:identify-tstar}
    \argmin_{t \in \{ 1, \ldots, T \}} \left( r_t + \rho \sum_{i=1}^n v_{i,t} \right) = \argmax_{t \in \{1,\ldots,T \}} \left( f_{0}(x_t) - \rho \sum_{i=1}^n [-f_{i}(x_t)]^+ \right),
\end{align}
for any penalty factor $\rho > 0$ since $r_t = f_0^\star - f_0(x_t)$, so $f_0^\star$ appears as a constant that will not change the location of the minimum. The result in \eqref{eq:identify-tstar} can only be used when observations of $h$ are noise-free since, otherwise, we do not have direct access to $\{ f_{i}(x_t) \}_{i \in \mathbb{N}_0^n}$. Therefore, we would like to understand the impact of this recommendation procedure on convergence, which we do in the following theorem.

\begin{theorem} \label{thm:4}
    Let Assumptions \ref{assump:1}, \ref{assump:2}, \ref{assump:3}, and \ref{assump:4} hold and $\bar{\rho}$ be large enough such that $f_0(x) - \rho\sum_{i=1}^n [-f_i(x)]^+$ is an exact penalty function for \eqref{eq:grey-box-opt} for any $\rho \geq \bar{\rho}$. Then, the recommended point \eqref{eq:recommended-point}, $\tilde{x}_T^r$, by CUBQ (Algorithm \ref{alg:cuqb}) must satisfy
    \begin{align} \label{eq:recommend-converge}
        0 \leq f_0^\star - f_0(\tilde{x}_T^r) + \rho\sum_{i=1}^n [-f_i(\tilde{x}_T^r)]^+ \leq \frac{4 \mathcal{L}_r(\rho) \beta_T^{1/2} \Psi_T}{T},~~~ \forall \rho \in [\bar{\rho}, \infty),
    \end{align}
    with probability at least $1 - \delta$, where $\mathcal{L}_r(\rho) = L_0 + \rho\sum_{i=1}^n L_i$.
\end{theorem}


It is interesting to note that the established bounds in Theorems \ref{thm:3} and \ref{thm:4} have a similar structure, with the main difference being the constant out front ($\mathcal{L}$ versus $\mathcal{L}_r(\rho)$). Theorem \ref{thm:3}, however, provides a stronger result in the sense that both regret and constraint violation can be individually bounded by a smaller constant. The main value of Theorem \ref{thm:4} is that it allows us to establish convergence of our recommended point, i.e., $\tilde{x}_T^r \to x^\star$ as long as the bound in \eqref{eq:recommend-converge} decays to 0 as $T \to \infty$. From Lemma \ref{lemma:5}, we see that this holds for the three most popular kernel types including linear, squared exponential, and Matern with $\nu > 1$. The convergence rates of CUQB for each of these cases are summarized in Table \ref{tab:kernel} assuming that the elements of $h$ are modeled independently with the same choice of kernel. 

\begin{table}[ht!]
    \centering
    \caption{Convergence rate for specific covariance kernel choices for the proposed CUQB method.}
    \begin{tabular}{|c|c|}
    \hline
        \textbf{Kernel} & \textbf{Convergence rate}  \\ \hline
        Linear & $\tilde{\mathcal{O}}( T^{-1/2} )$ \\ \hline
        Squared exponential & $\tilde{\mathcal{O}}( T^{-1/2} \log(T)^{(d+1)/2} )$ \\ \hline
        Matern, $\nu > 1$ & $\tilde{\mathcal{O}}( T^{-\nu/(2\nu+d(d+1))} \log(T)^{1/2} )$ \\ \hline
    \end{tabular}
    \label{tab:kernel}
\end{table}

Theorem \ref{thm:4} holds for any large enough $\rho \geq \bar{\rho}$ but does not discuss how to practically compute a $\rho$ value satisfying this requirement. It turns out that we can use the quantile bounds established in Lemma \ref{lemma:1} to compute such a $\rho$, which we summarize in the following corollary. 

\begin{corollary} \label{cor:1}
Let $\rho \geq 0$ be any scalar value that satisfies the following inequality
\begin{align} \label{eq:rho-tilde}
    \max_{x \in \mathcal{X}} \{ u_{0,0}(x) - \rho\textstyle\sum_{i=1}^n [-u_{i,0}(x)]^+ \} \leq l_{0,0}^\star,
\end{align}
where $l_{0,0}^\star = \max_{x \in \mathcal{X}, l_{i,0}(x) \geq 0, \forall i \in \mathbb{N}_1^n} l_{0,0}(x)$ is a lower bound on $f_0^\star$. Then, $\rho$ must satisfy the assumption in Theorem \ref{thm:4}, i.e., $\rho \geq \bar{\rho}$ with probability at least $1-\delta$.
\end{corollary}

The inequality \eqref{eq:rho-tilde} is defined completely in terms of the initial quantile bounds (at iteration $t=0$), which can be efficiently evaluated and optimized, as discussed in Section \ref{sec:cuqb}. As such, we can practically find $\rho$ using \eqref{eq:rho-tilde} before selecting any new samples using Algorithm \ref{alg:cuqb}. In particular, one can start from an initial guess of $\rho$ and keep iteratively increasing it until \eqref{eq:rho-tilde} is satisfied. Under Assumptions \ref{assump:1}-\ref{assump:4}, a finite value for $\rho$ must exist that satisfies \eqref{eq:rho-tilde}.
Note that the result holds at every iteration and so a new $\rho$ can be computed as the algorithm proceeds at additional cost, though there is little benefit to updating this choice once a sufficiently large value has been found. Furthermore, we have found that a fixing $\rho$ to be a large value (on the order of $10^5$) is often sufficient in practice as long as the objective and constraint functions are reasonably scaled. 










\subsection{Infeasibility declaration}
\label{subsec:infeas}

Theorems \ref{thm:3} and \ref{thm:4} assume the original problem is feasible; however, this may not always hold in practice and thus we have incorporated an infeasibility detection scheme in Algorithm \ref{alg:cuqb}, similar to \cite{xu2022constrained}. Lemma \ref{lemma:2} showed that, given the problem is feasible, infeasibility will not be declared with high probability. It turns out that we can further show that this step will be triggered with high probability if the original problem is infeasible as shown next. 

\begin{theorem} \label{thm:5}
    Let Assumptions \ref{assump:1}, \ref{assump:2}, and \ref{assump:4} hold, assume the original problem \eqref{eq:grey-box-opt} is infeasible, and that $\gamma_T^\text{max} = o(\log(T)/T)$ (i.e., $\lim_{T \to \infty} \gamma_T^\text{max}\log(T) / T = 0$). Then, given a desired confidence level $\delta \in (0,1)$, the CUQB method (Algorithm \ref{alg:cuqb}) will declare infeasibility in Line 3 within a finite number of iterations
    \begin{align}
        \overline{T} = \min_{T \geq 1} \left\lbrace T : \sqrt{\frac{\log(T) \gamma_{T}^\text{max}}{T}} < C \epsilon \right\rbrace,
    \end{align}
    with probability at least $1 - \delta$, where $\epsilon = -\max_{x \in \mathcal{X}} f_i(x) > 0$ for some $i \in \mathbb{N}_1^n$ and $C$ is some positive constant.
\end{theorem}





\section{Computational Setup}
\label{sec:comp-setup}

\subsection{Test problems}

We consider a total of 22 test problems that have been aggregated from the optimization and control literature. The problems have different number of input dimensions (ranging from 2 to 10) with bound constraints, different number of internal black-box functions (ranging from 1 to 24), and different number of constraints (ranging from 0 to 3). The test problems are distributed into two major groups: (i) synthetic problems and (ii) realistic application problems. The synthetic problems can be further divided into a set of 10 unconstrained and 10 constrained problems that are summarized in Tables \ref{tab:unconstrained_testproblems} and \ref{tab:constrained_testproblems}, respectively. Note that these synthetic problems have been modified from the original form to exhibit some degree of composite (grey-box) structure by splitting up the objective and/or constraint functions into known and unknown components. The full problem descriptions can be found in Appendices \ref{appendix:B} and \ref{appendix:C}. The realistic application problems include an environmental 
model calibration problem and a real-time reactor optimization problem. The first problem is meant to demonstrate that CUQB can scale to a problem with a large number of black-box functions. The second problem is meant to show how CUQB can be effective in the presence of measurement noise. 

\begin{table}[ht!]
\caption{Summary of the unconstrained set of synthetic test problems with input dimension $d$ and number of black-box functions $m$. The detailed formulation and solution of each problem is provided in Appendix \ref{appendix:B}.}
\begin{center}
\begin{tabular}{l c c c c l} \hline
Name &$d$ &$m$ &Appendix &Reference \\
\hline
Booth &2 &1 &\ref{subsec:booth} &\cite{subotic2011different} \\
Wolfe &3 &1 &\ref{subsec:wolfe} &\cite{jamil2013literature} \\
Rastrigin &3 &2 &\ref{subsec:rastrigin} &\cite{rastrigin1974systems} \\
Colville &4 &1 &\ref{subsec:colville} &\cite{rahnamayan2007novel} \\
Friedman &5 &1 &\ref{subsec:friedman} &\cite{gramacy2012cases} \\
Dolan &5 &2 &\ref{subsec:dolan} &\cite{jamil2013literature} \\
Rosenbrock &6 &4 &\ref{subsec:rb} &\cite{rosenbrock60} \\
Zakharov &7 &1 &\ref{subsec:zakharov} &\cite{chelouah2000tabu} \\
Powell &8 &4 &\ref{subsec:powell} &\cite{laguna2005experimental} \\
Styblinski-Tang &9 &4 &\ref{subsec:st} &\cite{styblinski1990experiments} \\
\hline
\end{tabular}
\label{tab:unconstrained_testproblems}
\end{center}
\end{table}

\begin{table}[ht!]
\caption{Summary of the constrained set of synthetic test problems with input dimension $d$, number of black-box functions $m$, and number of nonlinear constraints $n$. The detailed formulation and solution of each problem is provided in Appendix \ref{appendix:C}.}
\begin{center}
\begin{tabular}{l c c c c c} \hline
Name &$d$ &$m$ &$n$ &Appendix &Reference \\
\hline
Bazaraa &2 &2 &2 &\ref{subsec:bazaraa} &\cite{bazaraa2013nonlinear} \\
Spring  &3 &2 &4 &\ref{subsec:spring} &\cite{nekoo2022search} \\
Ex314  &3 &2 &3 &\ref{subsec:ex314} &\cite{floudas2013handbook} \\
Rosen-Suzuki  &4 &2 &3 &\ref{subsec:rosen-suzuki} &\cite{hock80} \\
st\_bpv1  &4 &3 &4 &\ref{subsec:st-bpv1} &\cite{tawarmalani2013convexification} \\
Ex211  &5 &2 &1 &\ref{subsec:ex211} &\cite{floudas2013handbook} \\
Ex212  &6 &2 &2 &\ref{subsec:ex212} &\cite{floudas2013handbook} \\
g09  &7 &2 &4 &\ref{subsec:g09} &\cite{karaboga2011modified} \\
Ex724  &8 &3 &4 &\ref{subsec:ex724} &\cite{floudas2013handbook} \\
Ex216  &10 &4 &5 &\ref{subsec:ex216} &\cite{floudas2013handbook} \\
\hline
\end{tabular}
\label{tab:constrained_testproblems}
\end{center}
\end{table}






\subsection{Baseline methods for comparison}
\label{subsec:baseline}

We compare a collection of seven different solvers to the proposed CUQB method, which are summarized in detail below. These solvers represent a mixture of the top performing DFO solvers from the comprehensive study performed in \cite{rios13} and more recently published state-of-the-art constrained BO methods. 

\begin{enumerate}
    \item \textbf{SNOBFIT \cite{huyer2008snobfit}:} This is a deterministic global search DFO solver that uses a combination of search spacing partitioning and surrogate model building. It uses the global selection rule to split the search domain and builds a quadratic model around the current best iterate and linear models for all other evaluated points. We use the Python implementation of SNOBFIT that has been made available in the \texttt{scikit-quant} package \cite{scikit-quant}. 
    \vspace{1mm}
    \item \textbf{DIRECT \cite{jones1993}:} This is a deterministic global search DFO solver that systematically divides the search domain into smaller rectangles and eliminates regions that are unlikely to contain the global optimum. The elimination process is based on concept of Lipschitz continuity, which ensures that the function's behavior within a rectangle is predictable. We use the Python implementation of DIRECT in the \texttt{scipy} package \cite{2020SciPy-NMeth}.    
    \vspace{1mm}
    \item \textbf{BOBYQA \cite{powell2009bobyqa}:} This is a deterministic local search DFO solver that belongs to the class of trust region algorithms. It iteratively constructs local quadratic models  within the trust region to guide the search toward an optimum. We use the Python implementation of BOBYQA in the \texttt{Py-BOBYQA} package \cite{cartis2019improving}. 
    \vspace{1mm}
    \item \textbf{CMA-ES \cite{hansen2003reducing}:} This is a stochastic global search DFO solver that belongs to the class of evolutionary algorithms. It maintains a population of candidate solutions and iteratively updates the mean and covariance matrix of a multivariate normal distribution to generate new candidate solutions. The algorithm adaptively adjusts the covariance matrix based on the success of previous iterations, allowing it to explore the search space efficiently. We use the Python implementation of CMA-ES in the \texttt{pycma} package \cite{hansen2019pycma}.
    \vspace{1mm}
    \item \textbf{EIC \cite{gardner2014bayesian}:} Expected improvement with constraints (EIC) is a deterministic global search constrained DFO solver that belongs to the family of black-box BO methods. It modifies the popular EI acquisition function to account for black-box constraints. In the absence of constraints, it simply reduces to the standard EI acquisition function \cite{jones1998efficient}. One can interpret EIC as a special case of the COBALT acquisition function in \eqref{eq:cobalt} when $g_i(x, h(x)) = h_i(x)$ (i.e., fully black-box instead of grey-box) for all $i \in \mathbb{N}_0^n$. In this case, one can derive a closed-form expression for the expectation in \eqref{eq:cobalt}, as shown in \cite{gardner2014bayesian}. We implemented our own version of EIC in the \texttt{BoTorch} package \cite{balandat2020botorch}.
    \vspace{1mm}
    \item \textbf{EPBO \cite{lu2022no}:} Exact penalty Bayesian optimization (EPBO) is a deterministic global search constrained DFO solver that belongs to the family of black-box BO methods. It is an extension of the well-known upper confidence bound (UCB) acquisition function \cite{srinivas09} to handle black-box constraints in a rigorous manner. One can show that the proposed CUQB method (without the infeasibility detection scheme) is equivalent to EPBO when $g_i(x, h(x)) = h_i(x)$ (i.e., fully black-box instead of grey-box) for all $i \in \mathbb{N}_0^n$. We implemented EPBO using the \texttt{BoTorch} package \cite{balandat2020botorch}.
    \vspace{1mm}
    \item \textbf{EIC-CF \cite{paulson2022cobalt}:} EIC with composite functions (EIC-CF) is a deterministic global search constrained DFO solver that belongs to the family of grey-box BO methods. It is an extension of the black-box EIC method to account for the known (grey-box) composite structure present in \eqref{eq:grey-box-opt}. EIC-CF is closely related to the COBALT method developed by the authors; the main difference is how one goes about practically optimizing over \eqref{eq:cobalt}. EIC-CF directly optimizes over \eqref{eq:cobalt-approx} with the non-smooth indicator function replaced by a differentiable sigmoid approximation whereas COBALT employs several heuristics that we have found lead to significant practical improvements (see numerical experiments in \cite[Section 4]{paulson2022cobalt}). We opted to use the vanilla EIC-CF for a fairer comparison to CUQB that does not incorporate any additional heuristics to improve the optimization of the acqusition function. We implemented EIC-CF using the \texttt{BoTorch} package \cite{balandat2020botorch}. 
\end{enumerate}


SNOBFIT, DIRECT, BOBYQA, and CMA-ES are not directly applicable to problems with unknown (black- or grey-box) constraints. 
To make them applicable to problems of the form \eqref{eq:black-box}, we use a standard penalty method (see, e.g., \cite{larson2019derivative}). In this approach, the following penalty function is defined
\begin{align}
    F(x) = f_0(x) - \rho \textstyle\sum_{i=1}^m [-f_i(x)]^+,
\end{align}
such that one can solve the box-constrained problem $\max_{x \in \mathcal{X}} F(x)$ in place of \eqref{eq:black-box}. We used a penalty weight factor of $\rho = 10^{5}$, which was found to provide a reasonable tradeoff between efficiency and constraint satisfaction. The other methods (including CUQB) handle constrained problems directly and so do not require further modifications. Unless otherwise stated, all solvers use their default settings. As discussed more in the next section, all constrained BO methods (EIC, EPBO, EIC-CF, and CUQB) were implemented by the authors in the same environment and therefore use the exact same GP model settings and same optimization strategy for tackling their respective acquisition subproblem. 

\subsection{Implementation details}

All BO methods assume that the black-box functions are independent and have a GP prior with a zero mean function and a 3/2 Matern covariance function with automatic relevance determination (ARD) \cite{williams2006gaussian}. The hyperparameters of the covariance function were estimated using the standard maximum likelihood estimation (MLE) approach in \cite[Chapter 2]{williams2006gaussian} using the \texttt{fit\_gpytorch\_mll} function in \texttt{BoTorch}.

The complete list of hyperparameters for the practical implementation of CUQB are summarized in Table \ref{tab:cuqb-param} that are divided into 4 major groups including termination. Since the focus of this work is on expensive functions, we impose an overall budget of only 100 function evaluations in every run and all algorithms are compared based on the quality of the recommended solution at every iteration. We briefly discuss the hyperparameters in the other three groups below.

We note that these practical choices are made to achieve a good tradeoff between computational cost of solving \eqref{eq:cuqb_subproblem} and high quality sampling performance. One could easily update these parameters to generate a more accurate solution to \eqref{eq:cuqb_subproblem} at additional cost. The theory in Section \ref{sec:theory} suggests that, the closer one gets to exactly solving \eqref{eq:cuqb_subproblem}, the better results one would observe in practice. This means that the ideal CUQB method in Algorithm \ref{alg:cuqb} should only further improve upon the performance results shown in Section \ref{sec:case-studies}.

A Python implementation of the proposed CUQB method, along with the other baseline solvers described previously, is available at \cite{lu23github}.

\addtolength{\tabcolsep}{+2pt}    
\begin{table}[ht!]
\caption{Default hyperparameter values for practical implementation of CUQB algorithm.}
\begin{center}
\begin{tabular}{p{0.25\linewidth} p{0.1\linewidth} p{0.35\linewidth} p{0.1\linewidth}} \hline
    &Parameter &Description &Value \\
\hline
Quantile approx. &$L$ &Number of randomly drawn Monte Carlo samples &50 \\
    &$\varepsilon$ &Regularization strength of soft sorting operator &0.1 \\
    &$\alpha$       &Probability level     &0.95 \\     
Optimization &$N_0$ &Number of initial random samples used to select multi-start candidates &8192 \\
&$N_\text{opt}$ &Number of multi-start initialization candidates for L-BFGS-B &3 \\
&$\rho$ &Penalty factor used for recommendation procedure and approximately solving \eqref{eq:cuqb_subproblem} &$10^5$ \\
Budget allocation &$T_0$ &Budget for initial randomly drawn samples &$2d+1$ \\
Termination &$T$ &Overall evaluation budget &100 \\
\hline
\end{tabular}
\label{tab:cuqb-param}
\end{center}
\end{table}
\addtolength{\tabcolsep}{-3pt}

\subsubsection{Quantile approximation hyperparameters}

As discussed in Section \ref{subsec:practicaloptcuqb}, the quantile functions $\{ u_{i,t}(x) \}_{i \in \mathbb{N}_0^n}$ must be estimated using some number of MC samples $L$ and sorting regularization parameter $\varepsilon$. Proposition \ref{prop:1} shows that the empirical estimate will converge to the true quantile function as $L \to \infty$ and $\varepsilon \to 0$; however, it is not clear what values will be work well enough in practice. 
Through numerical tests, we found that the results are fairly insensitive to values of $\varepsilon < 1$, which matches the observations from \cite{blondel2020fast}. This result was consistent across $L$ and $\alpha$ values, which led us to select a default value of $\varepsilon = 0.1$ that provides a nice tradeoff between accuracy and useful gradient information for the optimization process. In Figure \ref{fig:mc-prob}, we plot the empirical quantile function for a given test problem for a range of $L \in \{50, 100, 1000 \}$ and $\alpha \in \{ 0.75, 0.85, 0.95 \}$ with $\varepsilon=0.1$. Due to the randomness of the MC samples, we plot the resulting empirical quantile function for ten independent replicates. 
As expected, the variance in the estimator reduces as $L$ increases. Although we would like to select a large value for $L$, this increases the computational complexity of the optimization in \eqref{eq:cuqb_subproblem} such that it is practically useful to select as small an $L$ as possible without degrading the quality of the solution to \eqref{eq:cuqb_subproblem}. Even for $L=50$, the variability is reasonably small and does not change the location of the maximizer for this test problem. We found similar results across other problems, which led us to select 50 as the default value. 

Lemma \ref{lemma:1} provides a rigorous way to select $\alpha_{i,t}$; however, it tends to be conservative as it was derived through use of the union bound. A simpler choice that often works well in practice is to set $\alpha_{i,t} = \alpha$ to a constant $\alpha \in (0,1)$. The value $\alpha$ is a key factor impacting overall performance of CUQB. Empirically, if $\alpha$ is set to a small value, then there will be limited exploration and one is more likely to miss finding the globally optimal solution. On the other hand, if $\alpha$ is set too close to one, then the algorithm will overemphasize exploration, which may cause a reduction in the rate of convergence. 
Through the results of our performance tests in Section \ref{subsec-quantileapprox}, we found that setting $\alpha = 0.95$ works well across a variety of problems.

\begin{figure}[ht!] 
    \centering
    \begin{subfigure}[b]{0.32\textwidth}
        \centering
        \includegraphics[width=\textwidth]{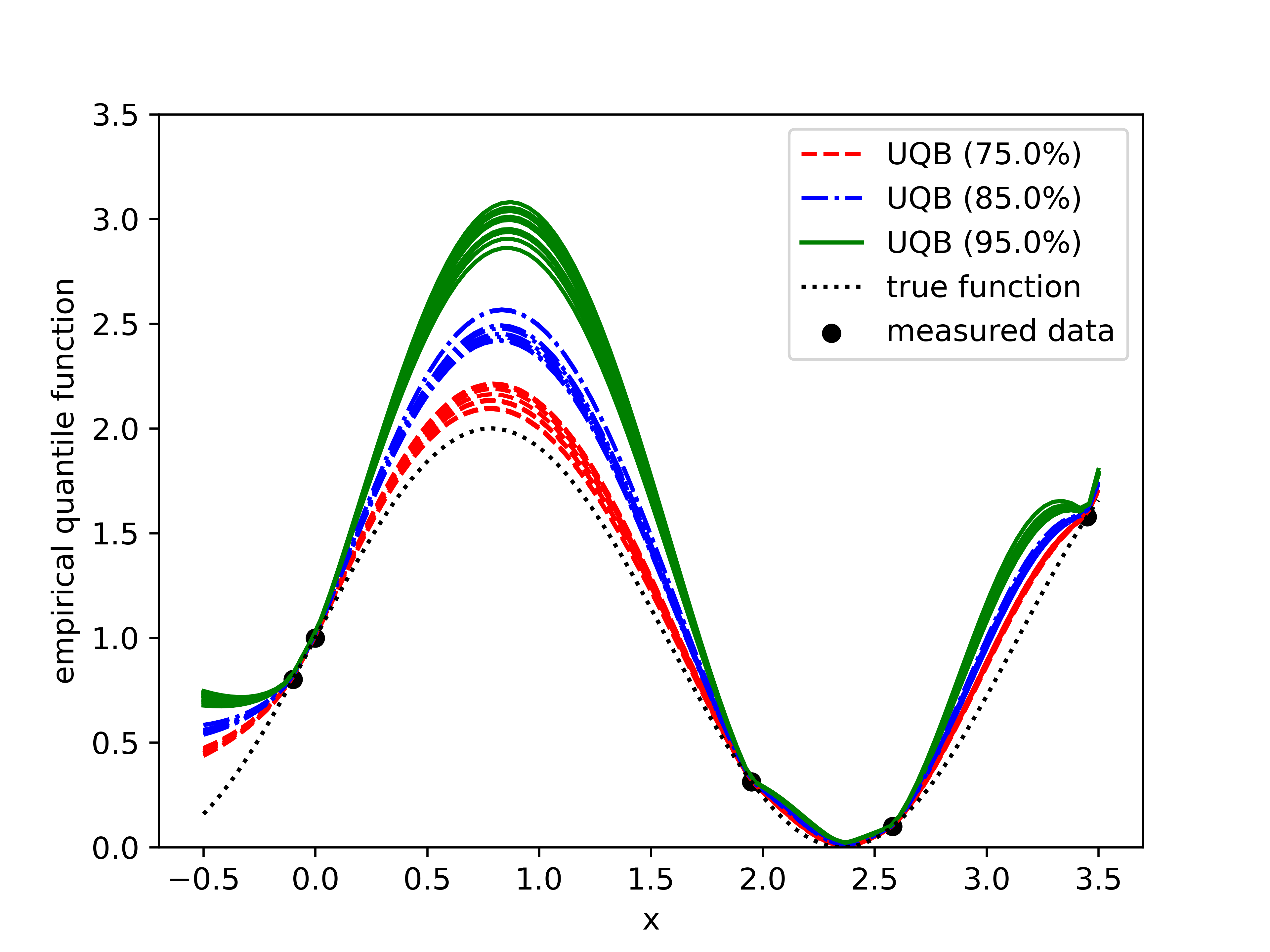}
        \caption{$L=50$}
    \end{subfigure}
    \hfill
    \begin{subfigure}[b]{0.32\textwidth}
        \centering
        \includegraphics[width=\textwidth]{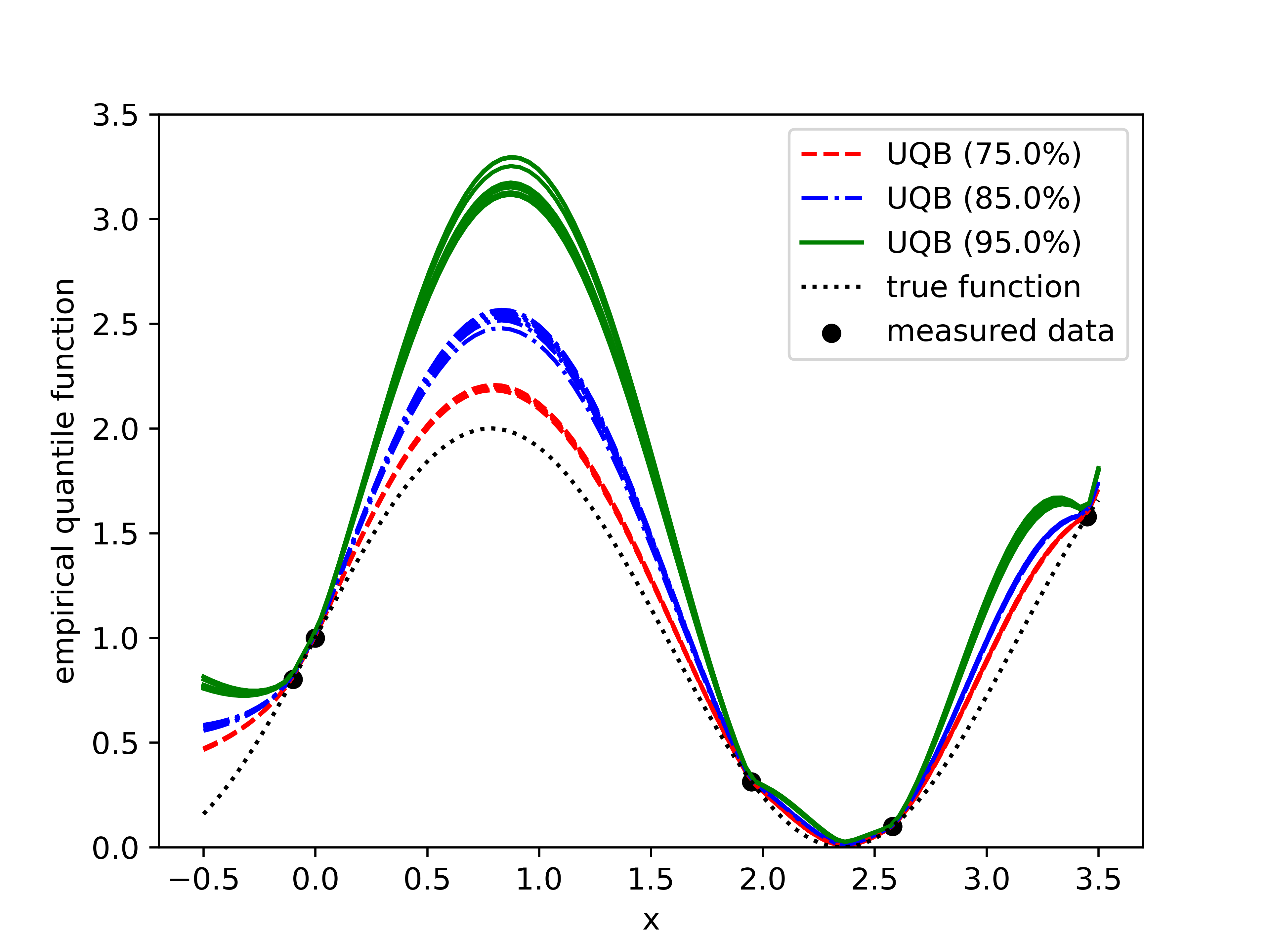}
        \caption{$L=100$}
    \end{subfigure}
    \hfill
    \begin{subfigure}[b]{0.32\textwidth}
        \centering
        \includegraphics[width=\textwidth]{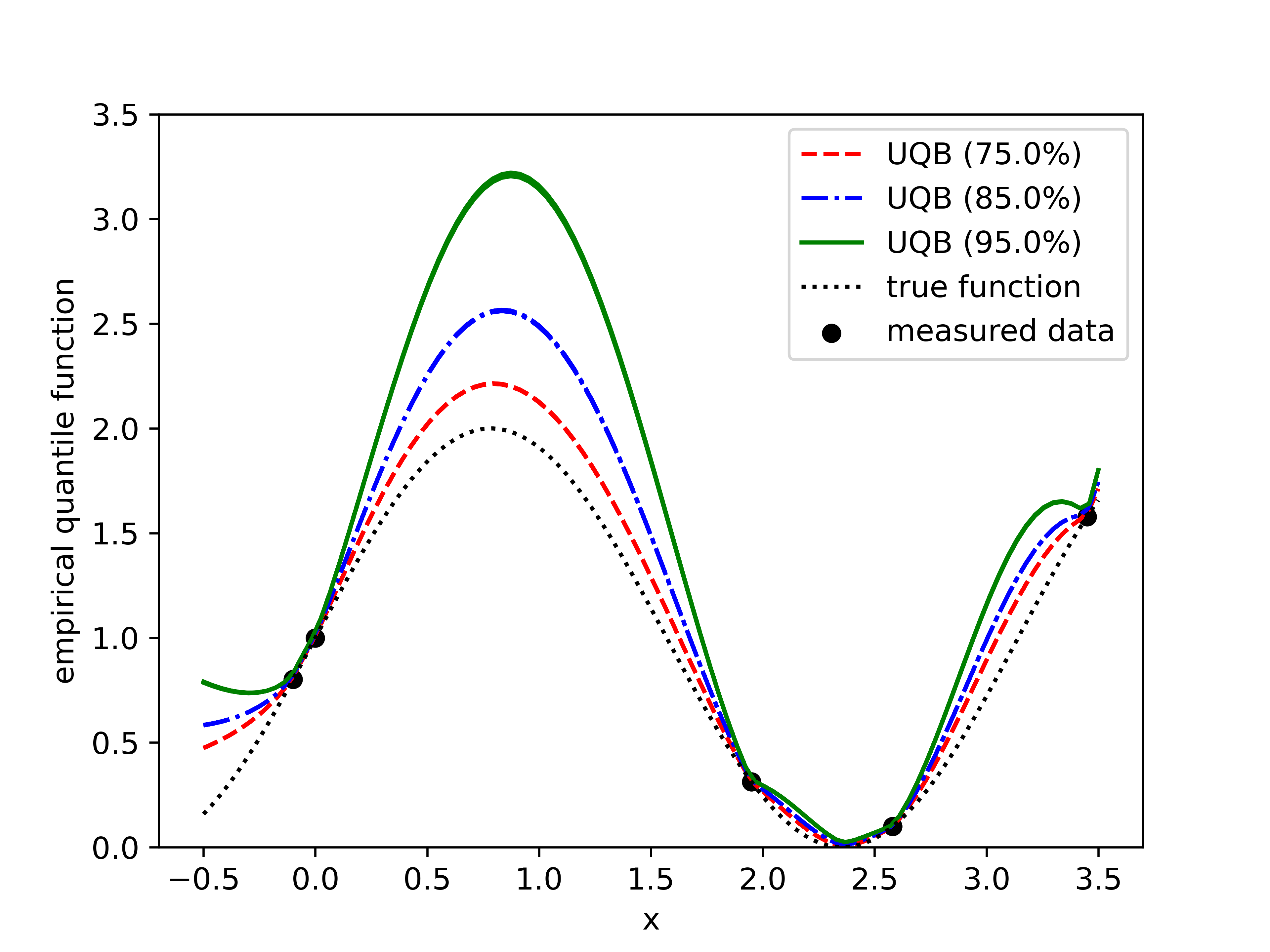}
        \caption{$L=1000$}
    \end{subfigure}    
    \vskip\baselineskip
    \vspace{-4mm}
    \caption{Empirical quantile function for a test problem estimated using $L = 50, 100, 1000$ Monte Carlo samples and $\alpha = 0.75, 0.85, 0.95$ probability levels for 10 independent replicates.}
\label{fig:mc-prob}
\end{figure}

\subsubsection{Optimization hyperparameters}

Although we can solve \eqref{eq:cuqb_subproblem} with a guaranteed global optimization method (such as branch and bound), it can be expensive in practice due to the generally nonconvex and multi-modal structure of the acquisition function. There are a large number of alternatives that one can pursue to tackle \eqref{eq:cuqb_subproblem}, with the simplest being to use DFO solvers that do not require gradient information. These are feasible for low-dimensional problems but do not scale to higher dimensions. A more scalable approach, used in \cite{balandat2020botorch}, is to rely on quasi-second order methods to provide strong convergence properties to a local solution. For simplicity, we focus on a penalty method for solving \eqref{eq:cuqb_subproblem} in practice (though this could easily be replaced by interior point or sequential quadratic programming methods)
\begin{align} \label{eq:penalty-subproblem}
    x_{t+1} \in \argmax_{x \in \mathcal{X}} a_t(x),
\end{align}
where $a_t(x) = u_{0,t}(x) - \rho \textstyle\sum_{i=1}^n [-u_{i,t}(x)]^+$.
As long as the upper quantile functions are monotonically non-increasing (i.e., $u_{i,t-1}(x) \geq u_{i,t}(x)$ for all $t \geq 0$), then \eqref{eq:penalty-subproblem} is a solution to the original problem \eqref{eq:cuqb_subproblem} for the $\rho$ derived in Corollary \ref{cor:1}. We can practically compute gradients of $a_t(x)$ since it is differentiable everywhere except the kink points in the non-smooth max operator $[-u_{i,t}(x)]^+$. As such, \eqref{eq:penalty-subproblem} can be solved to local optimality using the unconstrained L-BFGS-B solver \cite{zhu1997algorithm}, which is supported in \texttt{BoTorch} (thus directly enabling GPU-accelerated computations).

Due to the structure of the acquisition function $a_t(x)$, the initial conditions supplied to L-BFGS-B are very important since one may get stuck in highly suboptimal local solutions. To reduce this risk, it is common to employ multi-start optimization (i.e., start the solver from multiple initial conditions and select the best of the final converged solutions). We use the following effective heuristic developed in \cite{balandat2020botorch} to select a set of $N_\text{opt}$ initialization candidates:
\begin{enumerate}[1.]
    \item Sample $N_0$ points $\{ x_0^{(i)} \}_{i=1}^{N_0}$ from $\mathcal{X}$ using quasi-random Sobol sequences;
    \item Evaluate acquisition function at these candidate points $\{ v_t^{(i)} = a_t( x_0^{(i)} ) \}_{i=1}^{N_0}$;
    \item Let $\hat{\mu}$ and $\hat{\sigma}$ denote the empirical mean and standard deviation of the population $\{ v_t^{(i)} \}_{i=1}^{N_0}$, respectively. Sample $N_\text{opt}$ points from the set $\{ x_0^{(i)} \}_{i=1}^{N_0}$ where each point is selected with probability $p_i \propto \exp( \eta ( v_t^{(i)} - \hat{\mu}) / \hat{\sigma} )$ for some $\eta > 0$.
\end{enumerate}
This sampling procedure ensures that the initial conditions are selected to tradeoff exploration/exploitation based on the magnitude of $\eta$. This procedure converges to standard Sobol sampling as $\eta \to 0$ and corresponds to greedy sampling as $\eta \to \infty$. The latter approach is known to suffer from the clustering problem wherein all the high-valued points are near each other and thus converge to the same local solution. We therefore stick with the default value of $\eta = 1$. Interested readers are referred to \cite[Appendix E]{balandat2020botorch} for detailed comparison of this practical optimization approach to various alternative methods. 





\subsubsection{Budget allocation hyperparameters}

CUQB can be easily applied in the absence of initial data when an accurate GP prior is known. Since this is not the case in many real-world applications, we instead rely on the MLE training procedure discussed previously, which can be inaccurate when there is little-to-no initial data. Therefore, it is common in the BO literature to select a set of $T_0$ points uniformly at random to act as purely explorative samples in the early phase of the algorithm. There have been limited studies on the impact of $T_0$ on performance, though it is commonly suggested to use $d+1$ random points. We study the impact of $T_0$, along with the choice of $\alpha$, in Section \ref{subsec-quantileapprox} and find that a value of $2d+1$ provides a better overall balance between quick performance gains without getting stuck in local solutions. 

\subsubsection{Estimating average performance}

The performance of most solvers used in this study are affected by the selected set of starting points (e.g., $N_0$ initial samples for the BO-related methods). Therefore, each problem is solved for ten different randomly generated starting points. The same set of points are used for all considered solvers. We then compare the average behavior of each algorithm by the median recommended feasible objective value from the ten runs for each problem. 

\subsection{Performance metric}

To evaluate the performance of the different algorithms, we modify the so-called ``performance profiles'' suggested in \cite{more2009benchmarking} to work for constrained problems. Specifically, for a given tolerance $\tau \in [0,1]$ and starting point $x_0$, the following test is used to measure the increase in a penalized version of the objective function value:
\begin{align} \label{eq:perf-test}
    F_\omega - F(x_0) \geq (1 - \tau)( F_U - F(x_0) ),
\end{align}
where $F(x) = f_0(x) - \rho \sum_{i=1}^n [-f_i(x)]^+$ is the penalized objective function, $F(x_0)$ is the value of the penalized objective function at the starting point $x_0$, $F_\omega$ is the value of the penalized objective function at the solution reported by solver $\omega$, and $F_U$ is the largest penalized objective function obtained by any of the solvers in the set of interest within the budget of $T$ function evaluations. This condition implies that a problem is considered ``solved'' by $\omega$ if the median solution improved the starting point by at least a fraction of $(1-\tau)$ of the largest attained reduction.

For the set of test problems $\mathcal{K}$, the performance profile of any solver $\omega$ is defined as the fraction of problems that can be solved for a given tolerance $\tau$ within a given number of function evaluations $t$, i.e., 
\begin{align}
    \varrho_\omega(t) = \frac{\text{card}\{ \kappa \in \mathcal{K} : \pi_{\kappa, \omega} \leq t \}}{| \mathcal{K} |},
\end{align}
where $\pi_{\kappa, \omega}$ denotes the minimum number of function evaluations required by solver $\omega$ to satisfy \eqref{eq:perf-test} for problem $\kappa$.

\section{Results and Discussion}
\label{sec:case-studies}

\subsection{Choice of initial sampling and probability level parameter}
\label{subsec-quantileapprox}

The initial sample budget $T_0$ and the probability level $\alpha$ are key hyperparameters in CUQB. To illustrate their influence, we conduct numerical experiments by varying them over $T_0 \in \{ d+1, 2d+1, 3d+1 \}$ and $\alpha \in \{ 0.75, 0.85, 0.95 \}$. Figure \ref{fig:alpha-perf} shows the performance profiles for CUQB under these different settings for the set of 10 unconstrained problems (Table \ref{tab:unconstrained_testproblems}) with $\tau = 0.01$. The first major increase in performance happens around iteration 15, with $T_0 = d+1$ performing better than the alternatives that devote more of the initial budget to random samples. However, we see $T_0 = d+1$ has slightly worse performance in the later iterations, which is likely due to inaccurate estimates of the GP hyperparameters causing it to under-explore. Similarly, the lower $\alpha$ values of 0.75 and 0.85 show slightly worse performance after iteration 40 than 0.95 as a result of under-exploration (due to the quantile function not providing an accurate upper bound the true function). As such, the combination of $T_0 = 2d+1$ and $\alpha = 0.95$ appears to due to the best at providing balanced data for learning accurate covariance hyperparameters while not under- or over-exploring the design space. 
We emphasize that, under these settings, the median CUQB result was within 1\% of the best improvement for 9 out of 10 problems of varying size and difficulty within only 40 function evaluations.


\begin{figure}[ht!] 
    \centering
    \begin{subfigure}[b]{0.49\textwidth}
        \centering
        \includegraphics[width=\textwidth]{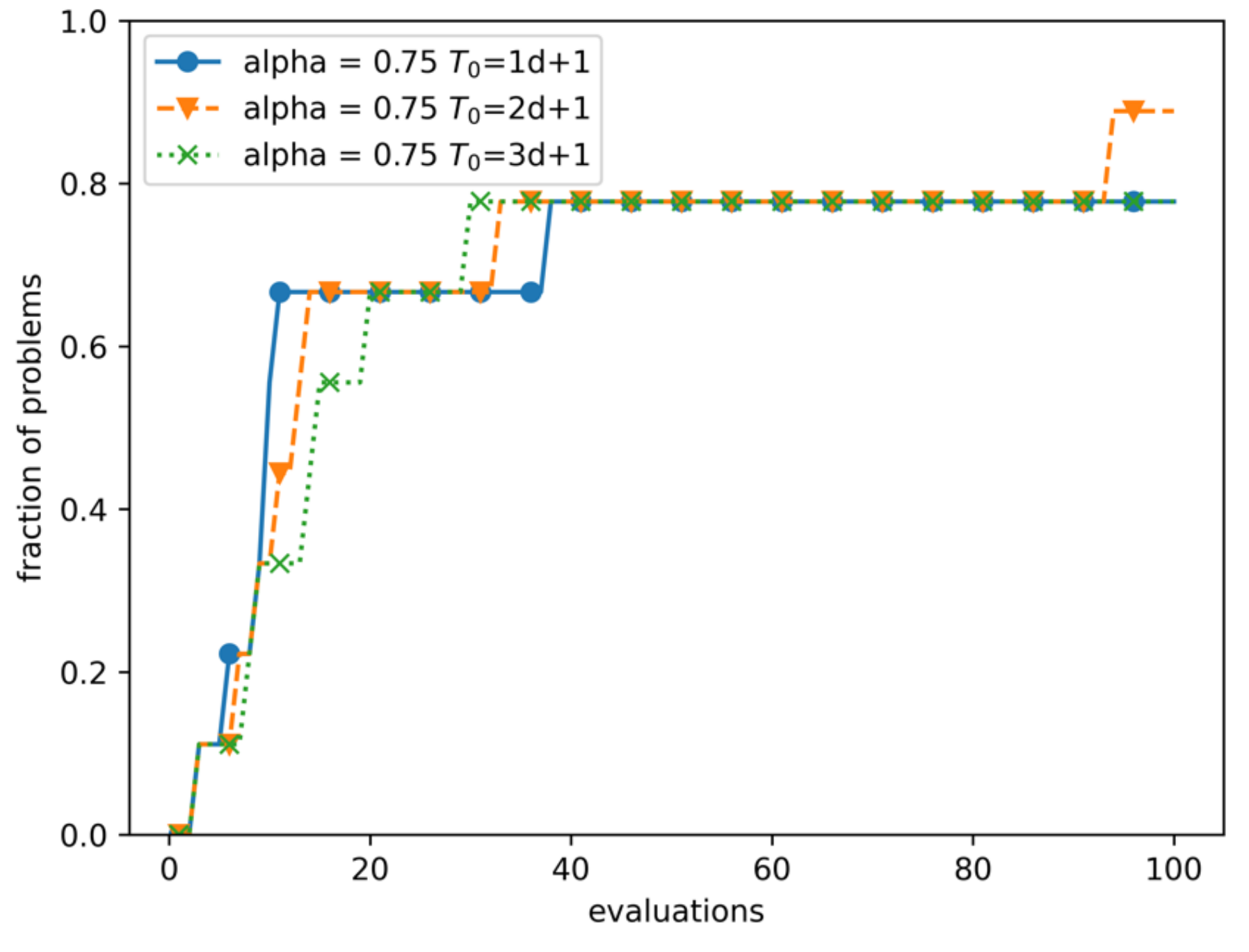}
        \caption{$\alpha = 0.75$}
    \end{subfigure}
    \hfill
    \begin{subfigure}[b]{0.49\textwidth}
        \centering
        \includegraphics[width=\textwidth]{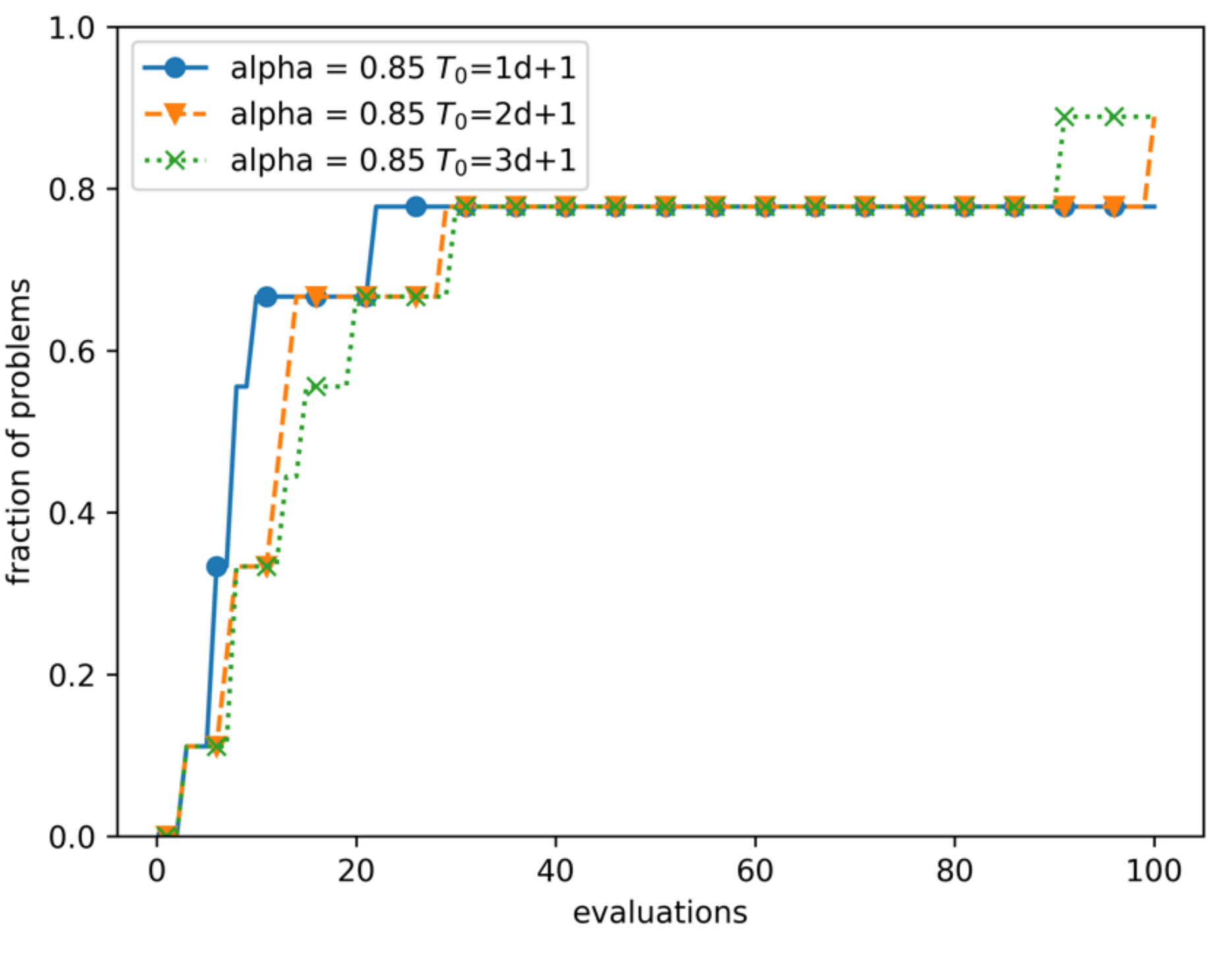}
        \caption{$\alpha = 0.85$}
    \end{subfigure}
    \hfill
    \begin{subfigure}[b]{0.49\textwidth}
        \centering
        \includegraphics[width=\textwidth]{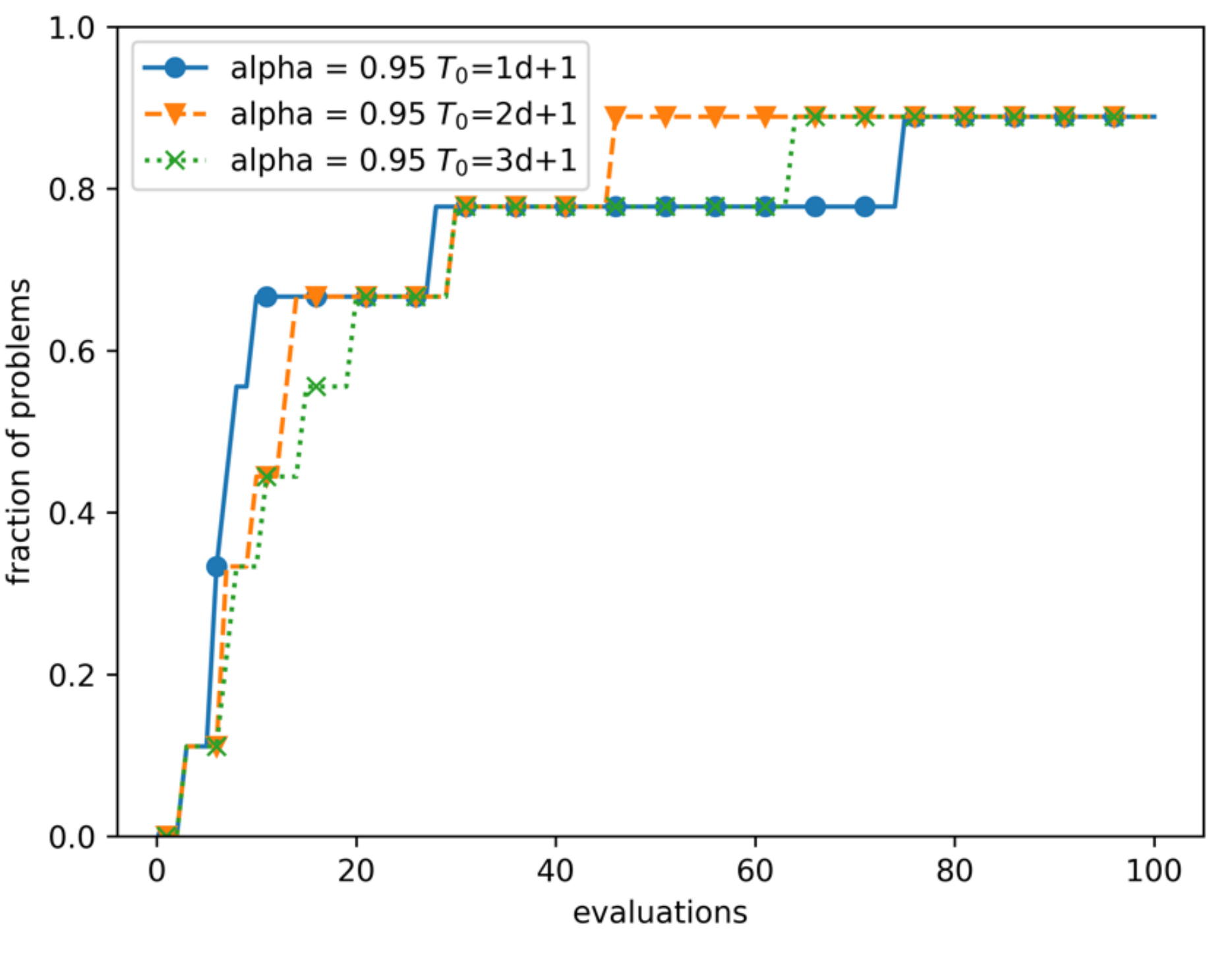}
        \caption{$\alpha = 0.95$}
    \end{subfigure}    
    \vskip\baselineskip
    \vspace{-4mm}
    \caption{Performance profiles comparing CUQB under different initial sampling and probability level settings over the set of unconstrained test problems for $\tau = 0.01$.}
\label{fig:alpha-perf}
\end{figure}

\subsection{Computational results for unconstrained problems}

Figure \ref{fig:unconstrained-perf} compares the performance of CUQB to the solvers described in Section \ref{subsec:baseline} for the set of unconstrained test problems for $\tau = 0.01$. We see that CUQB substantially outperforms all tested solvers, as it is the only solver capable of solving $>70$\% of the test problems (according to \eqref{eq:perf-test}). All BO-related methods (CUQB, EIC-CF, EIC, EPBO) perform similarly for the first 25 iterations and clearly outperform the non-BO solvers, solving $70$\% of the test problems within around 15 iterations. SNOBFIT and BOBYQA, which are both model-based DFO solvers, end up solving $50$\% and $60$\% of test problems, respectively, compared to only $40$\% for DIRECT and CMA-ES. This is not unexpected since model-based DFO solvers are known to be quite efficient for small- to medium-sized problems. It is interesting to note that the BO methods show faster convergence than deterministic model-based methods likely due to a better balance between exploration and exploitation. However, EIC-CF, EIC, and EPBO are unable to solve $30\%$ of the test problems whereas CQUB ends up solving $100\%$ of them. EIC and EPBO do not take advantage of prior knowledge and so their performance is fundamentally limited especially on higher-dimensional problems. Although EIC-CF does account for prior knowledge, the acquisition surface is much more challenging to optimize, which leads to poor performance in the later iterations. CQUB is able to overcome both of these limitations, making it the best solver among the tested solvers. 


\begin{figure}[tb!]
\centering
\includegraphics[width=0.8\linewidth]{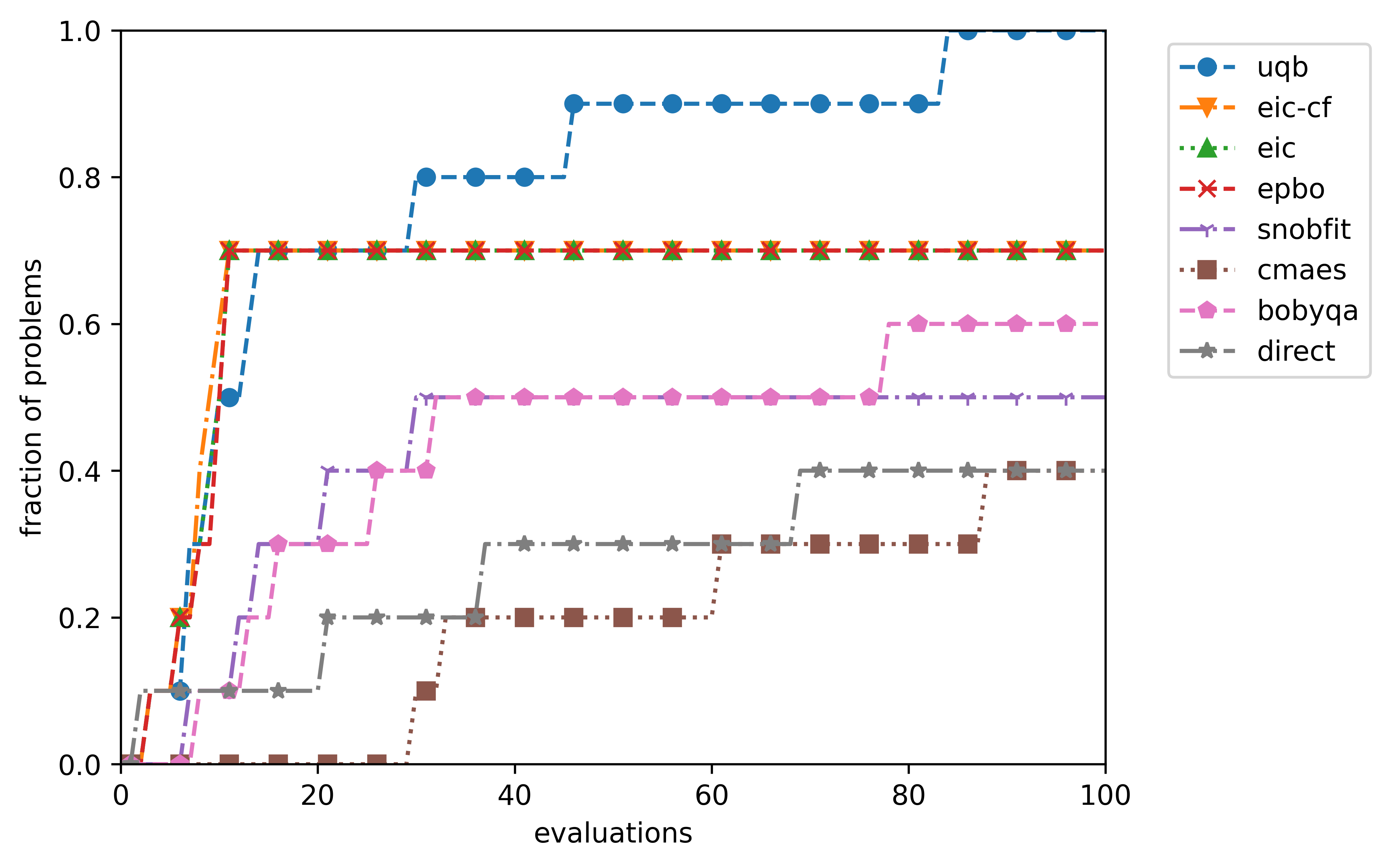}
\caption{Performance profiles comparing CUQB and the selected set of comparison solvers in Section \ref{subsec:baseline} over the set of unconstrained test problems for $\tau = 0.01$.}
\label{fig:unconstrained-perf}
\end{figure}

\subsection{Computational results for constrained problems}

Next, we analyze the performance of CUQB on problems with black- and grey-box constraints. Figure \ref{fig:constrained-perf} shows the performance profiles for CUQB and the selected baseline solvers for the set of constrained test problems with $\tau = 0.01$. CUQB remains the best solver in terms of both speed of finding a solution and number of problems solved, being the only solver to satisfy \eqref{eq:perf-test} for $> 80\%$ of the test problems. Similarly to the unconstrained case, we see that the BO-related solvers continue to outperform the non-BO solvers by a significant margin. Furthermore, due to the lumped penalty approach employed by SNOBFIT, DIRECT, BOBYQA, and CMA-ES, all of these methods struggle to select feasible solutions during the search process, resulting in much worse performance. In particular, DIRECT and CMA-ES are unable to solve any constrained test problem well enough to satisfy \eqref{eq:perf-test} while BOBYQA and SNOBFIT can only solve 10\% and 20\%, respectively. All BO-related methods, on the other hand, are able to solve $\geq 50\%$ of the constrained test problems since they explicitly model the constraints to better shape the acquisition function to identify feasible points. An important difference between the results for the constrained and unconstrained test problems is that there is a much bigger gap between the black-box methods (EIC and EPBO) and the grey-box methods (EIC-CF and CUQB) on the constrained problems. This suggests that prior knowledge is even more helpful when it comes to finding feasible points that satisfy unknown constraint functions, which is a key advantage of the grey-box perspective adopted in this work. EIC-CF continues to underperform relative to CUQB, likely due to the complexity of the acquisition surface such that CUQB excels over all other tested solvers. CUQB solves $80\%$ of the constrained test problems within the first 25 iterations and $100\%$ of the constrained problems in the 100 evaluation budget.

\begin{figure}[tb!]
\centering
\includegraphics[width=0.8\linewidth]{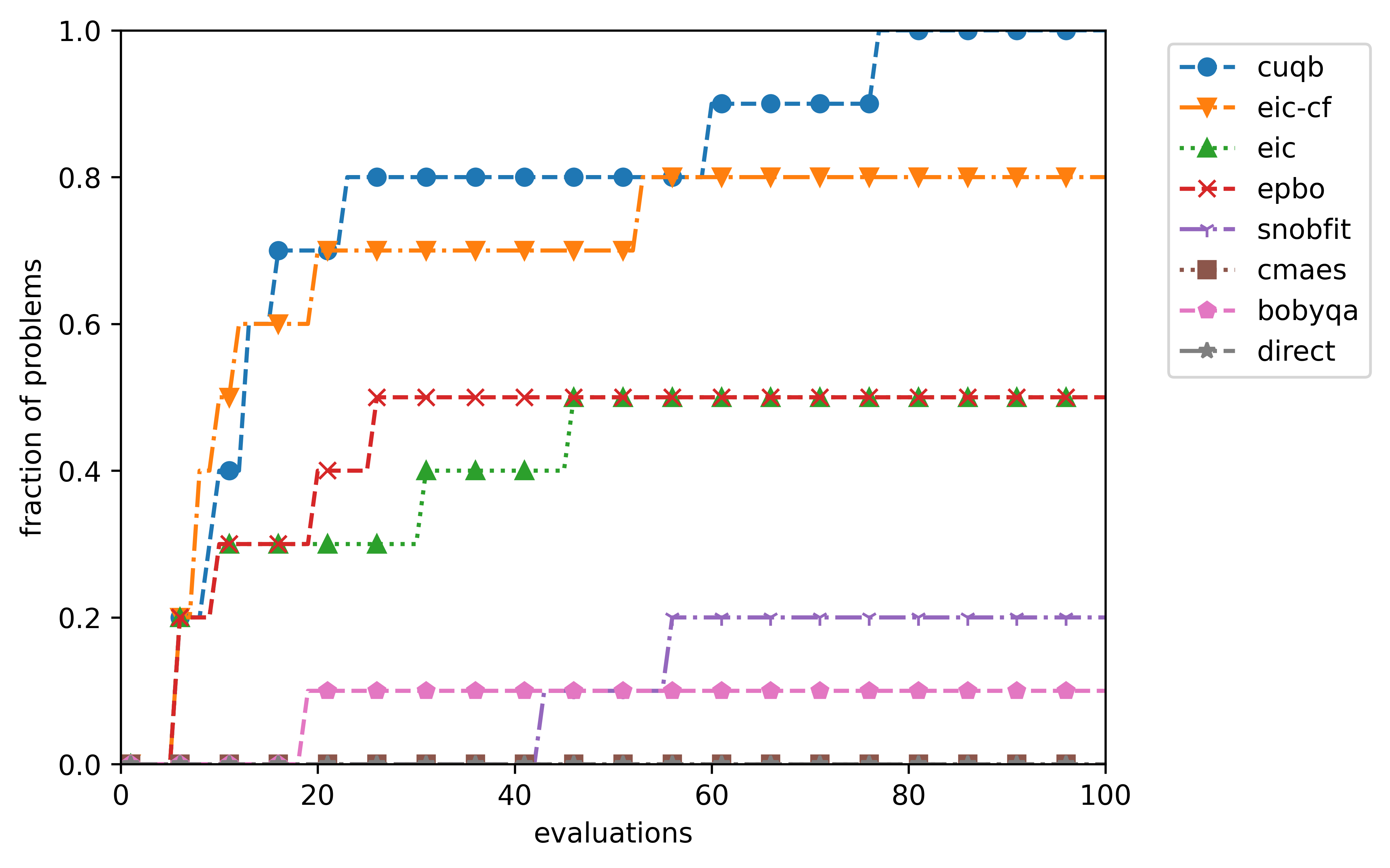}
\caption{Performance profiles comparing CUQB and the selected set of comparison solvers in Section \ref{subsec:baseline} over the set of constrained test problems for $\tau = 0.01$.}
\label{fig:constrained-perf}
\end{figure}




\subsection{Large-scale environmental model function}

We further compare CUQB and the selected solvers on a more realistic case related to Bayesian calibration of expensive computer models, originally proposed in \cite{bliznyuk2008bayesian}. The main purpose of this study is to show that CUQB can effectively scale to problems with a large number of black-box functions. 

Consider the spill of a chemical pollutant at two locations along a long narrow channel, which leads to the following concentration profile 
\begin{align}
    c(s,t ; M, D, \Lambda, \tau) = \frac{M}{\sqrt{4\pi{Dt}}}\text{exp}\left( \frac{-s^2}{4Dt} \right) + \frac{M \mathbb{I}(t > \tau)}{\sqrt{4\pi D(t-\tau)}}\text{exp}\left( \frac{-(s-\Lambda)^2}{4D(t-\tau)} \right),
\end{align}
where $s$ and $t$ are, respectively, the space and time points, $M$ the mass of the pollutant spilled at each location, $D$ the diffusion coefficient, $\Lambda$ is the location of the second spill, and $\tau$ the time of the second spill.
The true underlying parameters $M_0 = 10$, $D_0 = 0.07$, $\Lambda_0 = 1.505$, and $\tau_0 = 30.1525$ are unknown; however, we can estimate these parameters based on observations obtained from various monitoring stations. Specifically, we observe $c(s,t ; M_0, D_0, \Lambda_0, \tau_0)$ on a $4 \times 6$ grid of values at $s \in \mathcal{S} = \{ 1, 1.5, 2.5, 3 \}$ and $t \in \mathcal{T} = \{ 10, 20, 30, 40, 50, 60 \}$. Therefore, our goal is to find the parameter values $x = (M,D,\Lambda,\tau)$ that minimize the mean squared error between our model predictions and available observations
\begin{align} \label{eq:env-composite}
    g_0(x, h(x)) = -\sum_{(s,t) \in \mathcal{S} \times \mathcal{T}} \left( c(s,t ; M_0, D_0, \Lambda_0, \tau_0) - c(s,t ; M, D, \Lambda, \tau) \right)^2,
\end{align}
where $h(x) = \{ c(s,t ; M, D, \Lambda, \tau) : (s,t) \in \mathcal{S} \times \mathcal{T} \} \in \mathbb{R}^{24}$ is our black-box function. 
In summary, the environmental calibration problem has the following dimensions: $d=4$, $m=24$, and $n=0$.

The average simple regret values $\min_{t' = 1,\ldots, t} \{ r_{t'} = f_0^\star - f(x_{t'}) \}$ versus the number of iterations $t$ for all solvers are shown in Figure \ref{fig:env-regret}. CUQB remains the best solver converging to a regret value less than $10^{-6}$ within the first 20 simulations and a value of less than $10^{-8}$ by the end of 100 simulation budget (implying it quickly and accurately identified the true global solution). EIC-CF, which also exploits the composite structure of the objective function, performs similarly to CUQB, with only slightly worse performance after 20 iterations. The same trend remains for the other solvers, with EIC and EPBO coming in third and fourth, respectively, with a substantially slower rate of convergence than CUQB and EIC-CF due to neglected problem structure. However, these model-based solvers outperform the alternatives (DIRECT, CMA-ES, SNOBFIT, and BOBYQA), which end up with around 6 orders of magnitude worse regret values than CQUB after 100 simulations. These results demonstrate CUQB's ability to effectively exploit problem structure to reduce simulation cost in an important and common application problem.

\begin{figure}[tb!]
\centering
\includegraphics[width=0.8\linewidth]{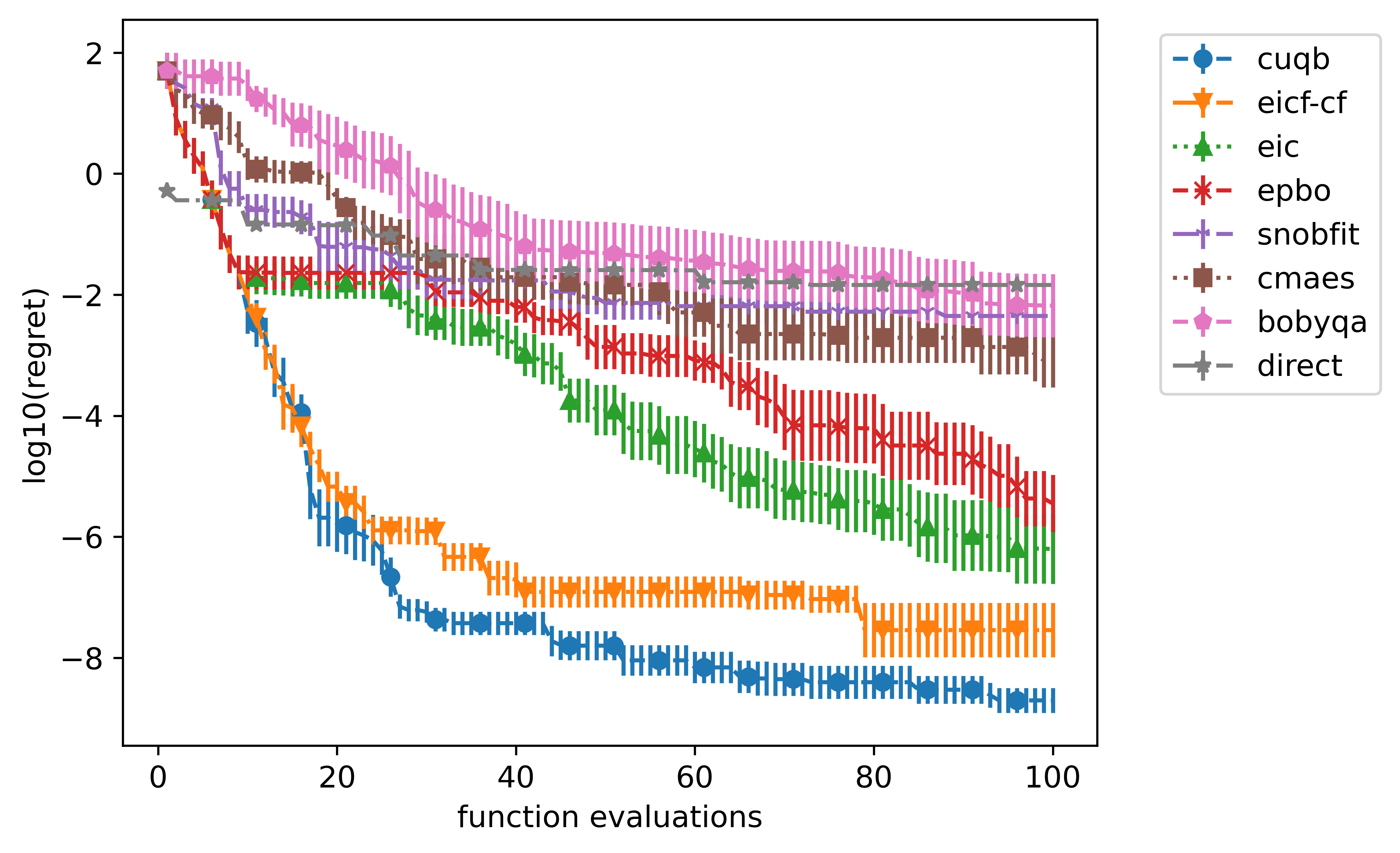}
\caption{Expected simple regret (on a logarithmic scale) for the environmental model problem with approximate confidence bounds estimated from 10 independent realizations. Results are shown for all comparison solvers discussed in Section \ref{subsec:baseline}.}
\label{fig:env-regret}
\end{figure}

\subsection{Application to real-time reactor optimization}

As a final case study, we consider the application of CUQB to the real-time optimization (RTO) of a chemical reactor system. The main purpose of this study is to demonstrate CUQB's ability to effectively handle measurement noise, which commonly occurs in RTO and other process control applications. Noise affects both the sample selection \eqref{eq:cuqb_subproblem} and recommendation \eqref{eq:recommended-point} processes, so we analyze both impacts in the results below, after providing a short description of the problem. 

We consider the Williams-Otto system from \cite{del2021real}, which is a benchmark problem in the RTO literature, in which a continuously stirred tank reactor (CSTR) is fed with two streams of pure reactants $A$ and $B$ with inlet flowrates $F_A$ and $F_B$. The reactor operates at a fixed temperature $T_r$, with the following three reactions occurring
\begin{align*}
A+B &\longrightarrow C \\
B+C &\longrightarrow P + E \\
C+P &\longrightarrow G 
\end{align*}
where $C$ is an intermediate, $G$ is an undesirable byproduct, and $P$ and $E$ are two main products. The complete set of mass balance and kinetic rate equations can be found in \cite{mendoza2016assessing}. 
Let $X_A$, $X_G$, $X_P$, and $X_E$ denote the mass fraction of species $A$, $G$, $P$, and $E$, respectively, which depend on the two input variables $x = (F_B, T_r)$. The unknown black-box function is $h(x) = ( X_A(x), X_G(x), X_P(x), X_E(x) ) \in \mathbb{R}^4$. We assume that we only have access to noisy measurements of these mass fractions
\begin{align}
y = h(x) + \epsilon, ~~~ \epsilon \sim \mathcal{N}(0, \sigma^2 I_4),
\end{align}
where $\sigma$ denotes the standard deviation of the noise. 
We are interested in maximizing the economic profit of this system by manipulating $F_B$ and $T_r$ subject to operating constraints on the mass fraction of unreacted $A$ and byproduct $G$. We can formulate these as composite functions as follows
\begin{subequations}
\begin{align}
    g_0(x,y) &= (1043.38X_P + 2092X_E)(F_A+F_B) - 79.23F_A - 118.34F_B, \\
    g_1(x,y) &= 0.12 - X_A, \\
    g_2(x,y) &= 0.08 - X_G.
\end{align}
\end{subequations}
In summary, the Williams-Otto RTO problem has the following dimensions: $d=2$, $m=4$, and $n=2$.

For comparison purposes, we first present results for the ideal perfect measurement case, i.e., $\sigma = 0$. As discussed in Section \ref{subsec:convergence-rate-bounds}, an advantage of the noise-free case is that we can easily identify the point in the sample sequence $\{ x_1, \ldots, x_t \}$ that has minimum penalty-based regret according to \eqref{eq:identify-tstar}. Figure \ref{fig:wo-noiseless-regret} shows the average simple penalty-based regret (minimum over the sequence) versus the number of iterations for all solvers. CUQB continues to perform the best, achieving penalty-based regret values of $\sim 10^{-2}$ within 50 iterations, which is more than a factor of 100 better than any other tested solver. EIC-CF performs significantly worse than CUQB in this case, which is likely due to a combination of how it handles constraints and the complexity of optimizing the acquisition function. Another interesting observation is that the black-box BO methods (EIC and EPBO) are outperformed by SNOBFIT, which is likely due to the problem being low-dimensional ($d=2$) and the sensitivity of the constraint functions with respect to the inputs that results in over-exploration of points with relatively small amounts of constraint violation. 

\begin{figure}[tb!]
\centering
\includegraphics[width=0.8\linewidth]{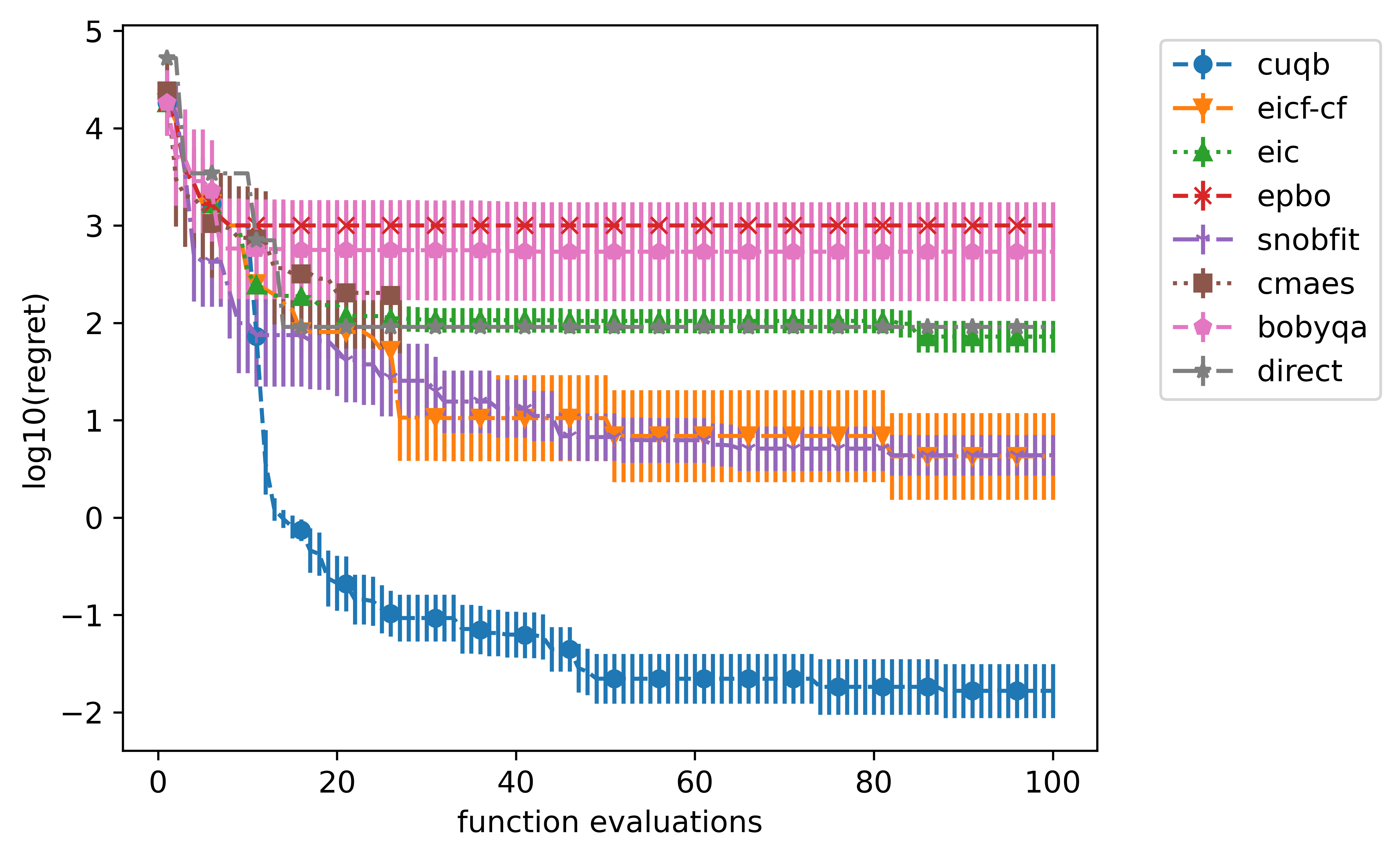}
\caption{Expected simple regret (on a logarithmic scale) for the Williams-Otto RTO problem without measurement noise ($\sigma = 0$) with approximate confidence bounds estimated from 10 independent realizations. Results are shown for all comparison solvers discussed in Section \ref{subsec:baseline}.}
\label{fig:wo-noiseless-regret}
\end{figure}

Next, we look at how the performance changes when a small ($\sigma = 0.01$) and large ($\sigma = 0.05$) amount of noise is included in the observations. Since we can no longer identify the best sample point in the sequence using \eqref{eq:identify-tstar}, we need an updated recommendation process. We proposed a quantile bound recommender \eqref{eq:recommended-point} for which we establish convergence rates in Theorem \ref{thm:3}. To demonstrate the practical importance of the proposed recommender, we also compare it to a ``naive recommender'' that replaces the noise-free values in \eqref{eq:identify-tstar} with their noisy counterparts. The true penalty-based regret values (evaluated at the recommended point) versus number of iterations for both recommenders and both noise levels are shown in Figure \ref{fig:wo-noise-regret}. The CUQB search process shows the best performance (lowest true penalty-based regret values) in all cases. We see a considerable increase in the regret values (more than a factor of 10) when using the naive recommender (Figure \ref{fig:wo-noise-regret}ac) compared to the proposed quantile bound recommender (Figure \ref{fig:wo-noise-regret}bd) -- this highlights the quantile bound recommender's ability to adequately filter the measurement noise when making a final recommendation, which is critical in the presence of constraints. As expected, we do see a drop in performance of all algorithms due to the presence of noise; however, CUQB appears to be the least impacted by noise since the mean and standard deviation remain the smallest out of all tested solvers for both noise levels. 
These results demonstrate CUQB's unique capability of simultaneously exploiting problem structure and filtering noise in an industrially-relevant application. 

\begin{figure}[ht!] 
    \centering
    \begin{subfigure}[b]{0.49\textwidth}
        \centering
        \includegraphics[width=\textwidth]{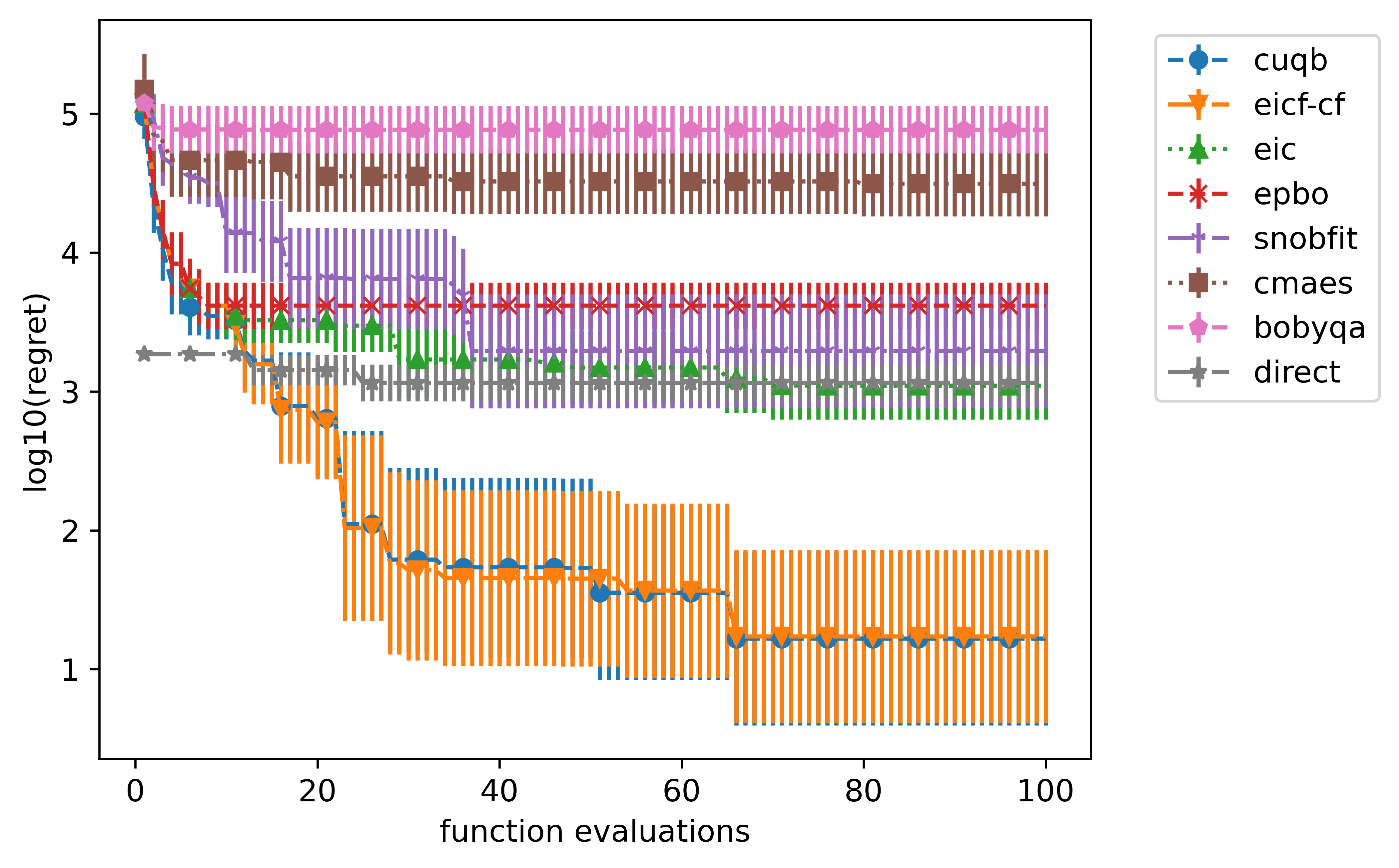}
        \caption{$\sigma = 0.01$, naive recommender}
    \end{subfigure}
    \hfill
    \begin{subfigure}[b]{0.49\textwidth}
        \centering
        \includegraphics[width=\textwidth]{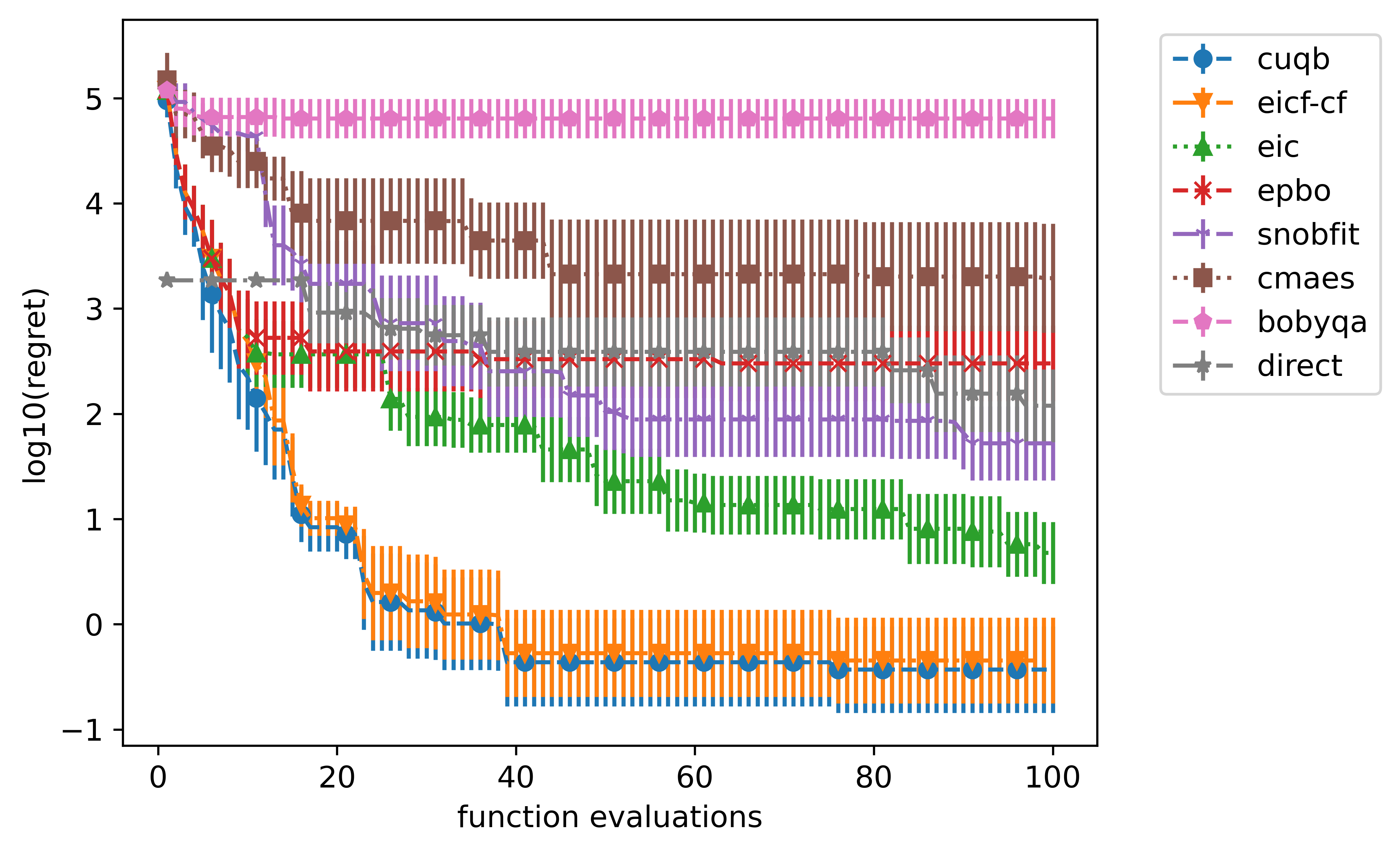}
        \caption{$\sigma = 0.01$, proposed recommender}
    \end{subfigure}
    \hfill
    \begin{subfigure}[b]{0.49\textwidth}
        \centering
        \includegraphics[width=\textwidth]{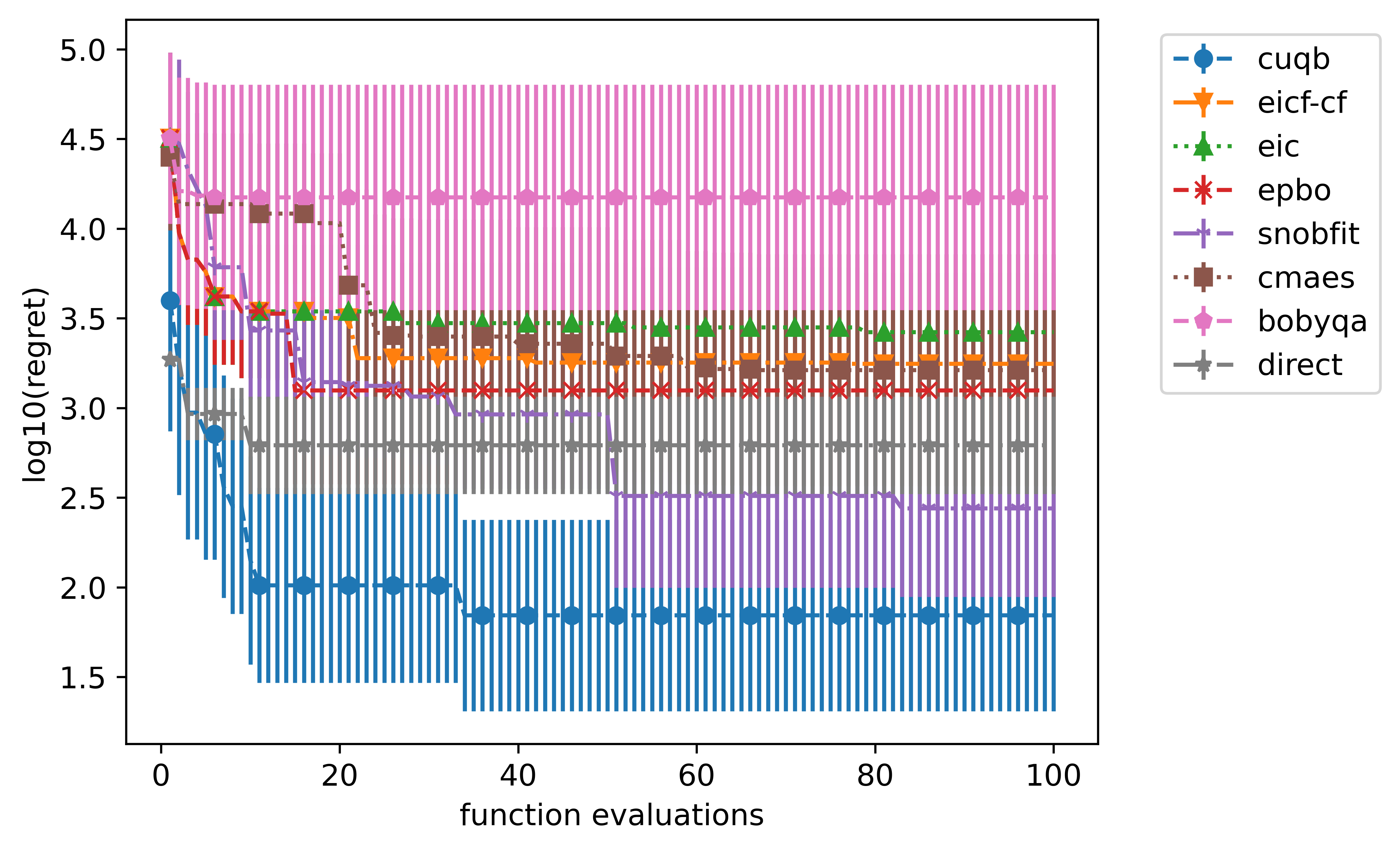}
        \caption{$\sigma = 0.05$, naive recommender}
    \end{subfigure}    
    \hfill
    \begin{subfigure}[b]{0.49\textwidth}
        \centering
        \includegraphics[width=\textwidth]{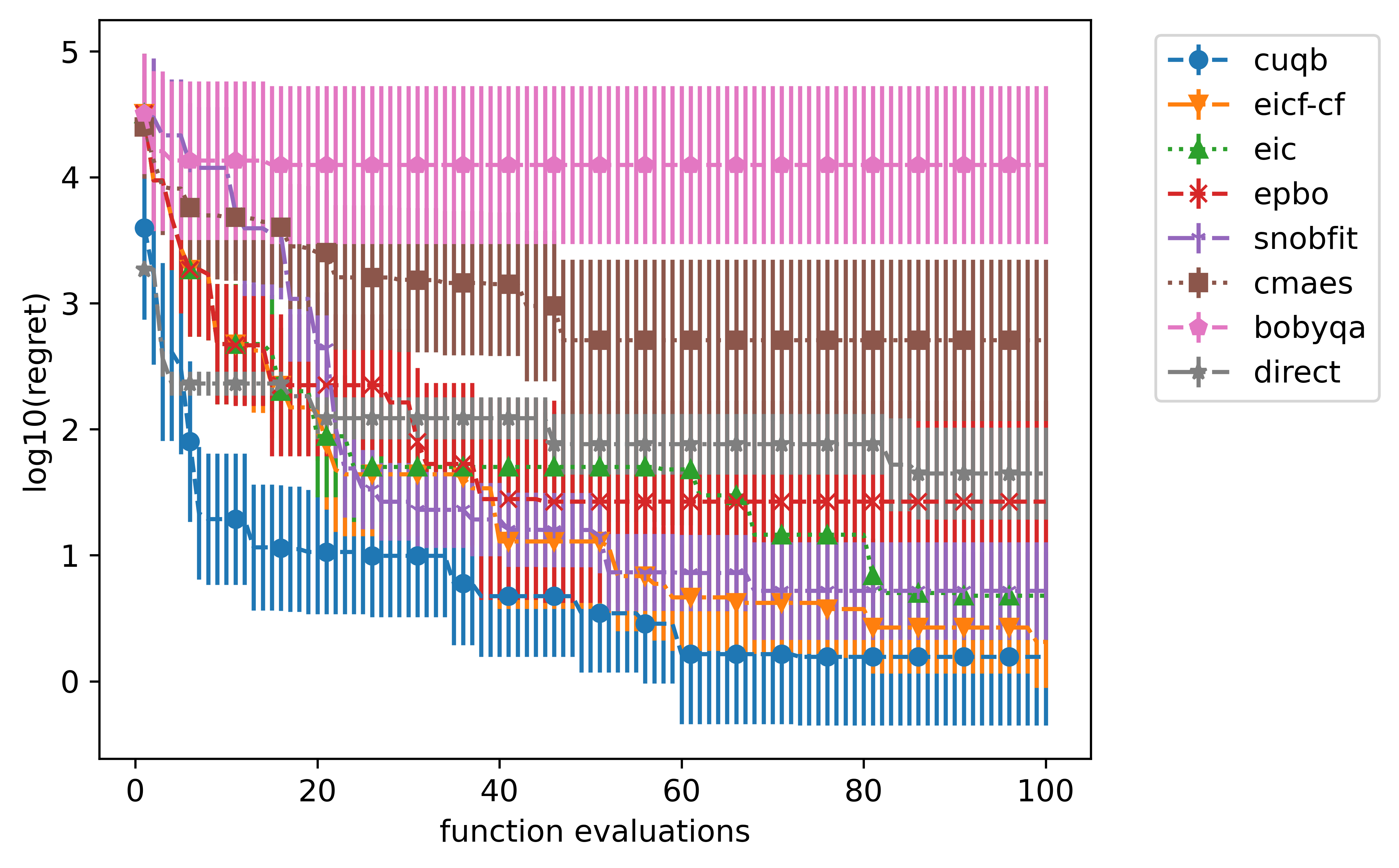}
        \caption{$\sigma = 0.05$, proposed recommender}
    \end{subfigure}
    \vskip\baselineskip
    \vspace{-4mm}
    \caption{Expected simple penalty-based regret (on a logarithmic scale) for the Williams-Otto RTO problem with measurement noise with approximate confidence bounds estimated from 10 independent realizations. Results are shown for two measurement noise levels $\sigma = 0.01$ (top) and $\sigma = 0.05$ (bottom), the naive (left) and proposed quantile bound recommender (right), and for all comparison solvers discussed in Section \ref{subsec:baseline}.}
\label{fig:wo-noise-regret}
\end{figure}

\section{Conclusions}
\label{sec:conclusions}

This paper proposes a new algorithm, CUQB, as a simple and effective approach for efficient constrained global optimization of composite functions whose input is an expensive vector-valued black-box function. These types of problems commonly arise when modeling real-world science, engineering, and control applications and allows for easy incorporation of so-called hybrid (grey-box) model structures. CUQB takes advantage of the principle that one should be optimistic in the face of uncertainty, which has been successfully exploited in its unconstrained and black-box conuterparts. 
In particular, our algorithm solves a constrained auxiliary problem defined in terms of quantile bound functions that can be used to identify a high probability range of prediction outcomes for the objective and constraints given all available data and prior knowledge. In addition to developing an efficient optimization procedure for CUQB, we also perform extensive theoretical analysis on its performance. Specifically, we develop bounds on the cumulative regret and constraint violation that can be translated into convergence rate bounds to the global solution under certain regularity assumptions. We also derive an infeasibility detection scheme as an extension of recent work for the black-box case. Several detailed numerical experiments are performed that demonstrate the effectiveness of CUQB, which we find to be highly competitive with state-of-the-art alternative methods.

Three future directions for extending upon the version of CUQB presented in this paper are briefly discussed next. First, the acquisition optimization subproblem is generally nonconvex and thus is difficult to solve to global optimality. In fact, there has been limited work on global optimization involving Gaussian process surrogate models, with \cite{schweidtmann2021deterministic} being one of the only works to tackle this challenging problem using spatial branch-and-bound combined with generalized McCormick relaxations. Although finite-time convergence to a global solution can be guaranteed under relatively mild assumption, this implementation is currently too slow to apply in practice and so additional work is needed to make global optimization viable in realistic settings. Second, the theoretical results technically hold for design sets with finite cardinality. Although we show that the impact on the number of design points on the rate of convergence is quite weak (logarithmic), it remains an open question exactly how CUQB will theoretically perform in the worst-case in continuous design spaces. We conjecture that similar results can be derived for CUQB in continuous spaces as long as the quantile bound is appropriately related to GP bounds in continuous spaces that have been recently established in \cite{chowdhury2017kernelized}. Third, while technically possible to apply CUQB to equality-constrained problems by transforming $c(x) = 0$ into two inequality constraints $c(x) \geq 0$ and $-c(x) \geq 0$, it is unclear how effective such a reformulation will work in practice. Noisy observations of the black-box functions will further complicate this task, so that additional modifications may be needed to overcome the challenges of a limited feasible region.

\appendix
\numberwithin{equation}{section}
\section{Appendix: Proofs for Theoretical Results in Section \ref{sec:cuqb}}
\label{appendix:new}

\subsection{Proof of Proposition \ref{prop:1}}

We start by proving the following lemma.

\begin{lemma} \label{lemma:quantile}
    Let $y_{(\lceil \alpha L \rceil)}$ denote the $\lceil p L \rceil$-th order statistic of the dataset $y_1,\ldots, y_L$ for $p \in (0,1)$. Let the samples $\{ y_i \}_{i = 1}^L$ be i.i.d. with cumulative distribution function $F(y) = \text{Pr}\{ Y \leq y \}$ and $F_{-}(y) = \text{Pr}\{ Y < y \}$. Let $Q(p) = \inf\{ y : F(y) \geq p \}$ be the corresponding $p$-quantile. Then, $y_{(\lceil p L \rceil)} \to Q(p)$ in probability as $L \to \infty$.
\end{lemma}

\begin{proof}
    By definition, the $k$-th order statistic must satisfy
    \begin{align} \label{eq:order-stat-eqiv-1}
        y_{(k)} \leq y ~ &\Leftrightarrow ~ \textstyle\sum_{l=1}^L \mathbb{I}(y_{l} \leq y) \geq k, \\ \label{eq:order-stat-eqiv-2}
        y_{(k)} \geq y ~ &\Leftrightarrow ~ \textstyle\sum_{l=1}^L \mathbb{I}(y_{l} \geq y) \geq L - k + 1.
    \end{align}
    For $L \geq 1$ and $p \in (0,1)$, $\lceil pL \rceil > 0$ so that $y_{(\lceil p L \rceil)}$ is well-defined. 
    Recall the formal definition of convergence in probability, i.e., we must show that
    \begin{align}
        \lim_{L \to \infty} \text{Pr}\left\lbrace | y_{(\lceil p L \rceil)} - Q(p) | \geq \epsilon \right\rbrace = 0, ~~~ \forall \epsilon > 0.
    \end{align}
    We can rewrite the probability expression as follows
    \begin{align}
        &\text{Pr}\left\lbrace | y_{(\lceil p L \rceil)} - Q(p) | \geq \epsilon \right\rbrace \\\notag
        & = \text{Pr}\left\lbrace y_{(\lceil p L \rceil)} \leq Q(p)-\epsilon \right\rbrace + \text{Pr}\left\lbrace y_{(\lceil p L \rceil)} \geq Q(p) + \epsilon \right\rbrace.
    \end{align}
    We will analyze these two terms separately. Using the equivalence in \eqref{eq:order-stat-eqiv-1}, we can establish the following
    \begin{align}
        \text{Pr}\left\lbrace y_{(\lceil p L \rceil)} \leq Q(p)-\epsilon \right\rbrace = \text{Pr}\left\lbrace \frac{1}{L}\sum_{l=1}^L \mathbb{I}(y_{l} \leq Q(p) - \epsilon) \geq \frac{\lceil p L \rceil}{L} \right\rbrace.
    \end{align}
    By the strong law of large numbers, $\frac{1}{L}\sum_{l=1}^L \mathbb{I}(y_{l} \leq Q(p) - \epsilon) \to F(Q(p)-\epsilon)$ in probability and $\frac{\lceil p L \rceil}{L} \to p$ as $L \to \infty$. Therefore, taking the limit of the above expression leads to
    \begin{align}
        \lim_{L \to \infty} \text{Pr}\left\lbrace y_{(\lceil p L \rceil)} \leq Q(p)-\epsilon \right\rbrace = \text{Pr}\left\lbrace F(Q(p)-\epsilon) \geq p \right\rbrace = 0,
    \end{align}
    which follows from the definition of $Q(p)$ (smallest $y$ value such that $F(y) \geq p$) and the fact that $F(y)$ is non-decreasing such that $F(Q(p) - \epsilon) < p$ for any $\epsilon >0$. We can repeat a similar analysis for the second term after algebraic manipulations
    \begin{align}
        &\text{Pr}\left\lbrace y_{(\lceil p L \rceil)} \geq Q(p) + \epsilon \right\rbrace \\\notag
        &= \text{Pr}\left\lbrace \frac{1}{L}\sum_{l=1}^L \mathbb{I}(y_{l} \geq Q(p) + \epsilon) \geq \frac{L - \lceil p L \rceil + 1}{L} \right\rbrace, \\\notag
        &= \text{Pr}\left\lbrace \frac{1}{L}\sum_{l=1}^L \mathbb{I}(y_{l} < Q(p) + \epsilon) \leq \frac{\lceil p L \rceil}{L} - \frac{1}{L} \right\rbrace,
    \end{align}
    where we have used $\mathbb{I}(y_{l} \geq y) = 1 - \mathbb{I}(y_{l} < y)$. Again, applying the strong law of large numbers, we have $\frac{1}{L}\sum_{l=1}^L \mathbb{I}(y_{l} < Q(p) + \epsilon) \to F_{-}(Q(p) + \epsilon)$ and $\frac{\lceil p L \rceil}{L} - \frac{1}{L} \to p$ as $L \to \infty$. Taking the limit of then yields
    \begin{align}
        \lim_{L \to \infty} \text{Pr}\left\lbrace y_{(\lceil p L \rceil)} \geq Q(p) + \epsilon \right\rbrace = \text{Pr}\left\lbrace F_{-}(Q(p) + \epsilon) < p \right\rbrace = 0,
    \end{align}
    which follows from the fact that the limit approaches from below $p$ and $F$ is increasing with $F(Q(p)) \geq p$. It follows that $y_{(\lceil p L \rceil)} \in (Q(p) - \epsilon, Q(p) + \epsilon)$ with probability approaching 1 for any $\epsilon > 0$.
\end{proof}

From \cite[Proposition 2]{blondel2020fast}, we have that $s_\varepsilon(\boldsymbol{\theta}) \to s(\boldsymbol{\theta})$ as $\varepsilon \to 0$ for all $\boldsymbol{\theta} \in \mathbb{R}^L$. This means that the soft version of the empirical quantile will converge to the true empirical quantile value in \eqref{eq:estimated-quantile} as $\varepsilon \to 0$ for any $p \in (0,1)$, any $x \in \mathcal{X}$, and any $L \in \mathbb{N}$. Then, according to Lemma \ref{lemma:quantile}, we must have that $\hat{u}^{(L)}_{i,t}(x) \to u_{i,t}(x)$ in probability for all $x \in \mathcal{X}$. $\hfill \square$

\subsection{Proof of Proposition \ref{prop:2}}

This result follows from the fact that the posterior distribution of $h(x)$ is an $m$-dimensional normal with mean vector $\mu_t(x)$ and covariance matrix $\Sigma_t(x)$. By the closure of Gaussian distributions under linear transformations, the posterior distribution of $a_i^\top(x) h(x) + b_i(x)$ must be a univariate normal with mean $a_i^\top(x) \mu_{t}(x) + b_i(x)$ and variance $a_i^\top(x) \Sigma_{t}(x) a_i(x)$. The stated result then follows from the analytic expression for the quantile function of normal distribution and the fact that $\Phi^{-1}(\alpha) = -\Phi^{-1}(1-\alpha)$ for any $\alpha \in (0,1)$. $\hfill \square$

\section{Appendix: Proofs for Theoretical Results in Section \ref{sec:theory}}
\label{appendix:A}

\subsection{Proof of Lemma \ref{lemma:1}}
\label{subsec:lem1}

    The quantile bounds $l_{i,t}(x)$ and $u_{i,t}(x)$ are random variables since they depend on the noisy observations $y_t$. Therefore, we can define the following event that the true function is contained in these bounds for all $x \in \mathcal{X}$ and $t \geq 0$ as follows
    \begin{align}
        \mathcal{E}_i = \cap_{x \in \mathcal{X}}\cap_{t \geq 0}\{ l_{i,t}(x) \leq g_i(x, h(x)) \leq u_{i,t}(x) \}, ~~~ \forall i \in \mathbb{N}_0^n.
    \end{align}
    We can then establish the following sequence of inequalities on the joint probability of these events holding for all functions
    \begin{subequations}        
    \begin{align} \label{eq:complement}
        \text{Pr}\left\lbrace \cap_{i=0}^n \mathcal{E}_i \right\rbrace &= 1 - \text{Pr}\left\lbrace \cup_{i=0}^n \overline{\mathcal{E}_i} \right\rbrace, \\ \label{eq:booles_n}
        &\geq 1 - \sum_{i=0}^n \text{Pr}\left\lbrace \overline{\mathcal{E}_i} \right\rbrace, \\ \label{eq:booles_xt}
        &\geq 1 - \sum_{i=0}^n\sum_{x \in \mathcal{X}} \sum_{t \geq 0} \left( F_{Y_{i,t}(x)}(l_{i,t}(x)) + 1 - F_{Y_{i,t}(x)}(u_{i,t}(x)) \right), \\ \label{eq:cdf_quant_def}
        &\geq 1 - \sum_{i=0}^n\sum_{x \in \mathcal{X}} \sum_{t \geq 0} \left( \frac{\alpha_{i,t}}{2} + 1 - \left(1 - \frac{\alpha_{i,t}}{2} \right) \right), \\ \label{eq:inf_series}
        & = 1 - \delta,
    \end{align}
    \end{subequations}
    where \eqref{eq:complement} follows from De Morgan's complement law; \eqref{eq:booles_n} follows from Boole's inequality (union bound) being applied across the different functions; \eqref{eq:booles_xt} follows from Boole's inequality being applied across $x \in \mathcal{X}$, $t \geq 0$ and the definition of the CDF; \eqref{eq:cdf_quant_def} follows from the identity $F_Y(Q_Y(\alpha)) \geq \alpha$ for any random variable $Y$ and probability value $\alpha \in (0,1)$; and \eqref{eq:inf_series} follows from the fact that the infinite series $\sum_{t \geq 0} \frac{6}{\pi^2(t+1)^2} = 1$ is convergent such that $\sum_{t \geq 0} \alpha_{i,t} = \frac{\delta}{(n+1)|\mathcal{X}|}$. $\hfill \square$

\subsection{Proof of Lemma \ref{lemma:2}}
\label{subsec:lem2}

    According to Lemma \ref{lemma:1}, the set of functions $\{ g_i(x,h(x)) \}_{i \in \mathbb{N}_0^n}$ are bounded by their lower $l_{i,t}(x)$ and upper quantile bounds $u_{i,t}(x)$ for all $x \in \mathcal{X}$ and iterations $t \geq 0$ with probability at least $1 - \delta$. The statements below are conditioned on this joint event being true, so must also hold with probability at least $1 - \delta$.

    Under Assumption \ref{assump:3}, $x^\star$ is feasible and therefore must satisfy $g_i(x^\star,h(x^\star)) \geq 0$ for all $i \in \mathbb{N}_1^n$. Since $u_{i,t}(x^\star) \geq g_i(x^\star,h(x^\star)) \geq 0$, there must exist at least one feasible solution to \eqref{eq:cuqb_subproblem}, meaning infeasibility is not declared. We then establish the following sequence of inequalities on the instantaneous regret
    \begin{subequations}        
    \begin{align}
        r_t &\leq r_t^+, \\ \label{eq:add_u0}
        &= \left[ f_0^\star - u_{0,t}(x_{t+1}) + u_{0,t}(x_{t+1}) - f_0(x_{t+1}) \right]^+, \\ \label{eq:poslin_prop}
        &\leq \left[ f_0^\star - u_{0,t}(x_{t+1}) \right]^+ + \left[ u_{0,t}(x_{t+1}) - f_0(x_{t+1}) \right]^+, \\ \label{eq:upp_low_bound}
        &\leq \left[ u_{0,t}(x^\star) - u_{0,t}(x_{t+1}) \right]^+ + \left[ u_{0,t}(x_{t+1}) - l_{0,t}(x_{t+1}) \right]^+, \\ \label{eq:defn_width}
        &= u_{0,t}(x_{t+1}) - l_{0,t}(x_{t+1}) = w_{0,t}(x_{t+1}),
    \end{align}
    \end{subequations}
    where \eqref{eq:add_u0} follows from the definition of positive instantaneous regret; \eqref{eq:poslin_prop} follows from $[a + b]^+ \leq [a]^+ + [b]^+$, $\forall a,b \in \mathbb{R}$; \eqref{eq:upp_low_bound} follows from the assumed upper and lower bounds on the objective function; and \eqref{eq:defn_width} follows from \eqref{eq:cuqb_subproblem}, which implies $u_{0,t}(x^\star) \leq u_{0,t}(x_{t+1})$ (since we are maximizing over the feasible region). We can perform a similar analysis on the constraint violation
    \begin{subequations}        
    \begin{align}
        v_{i,t} &= [-f_i(x_{t+1})]^+, \\
        &= [-f_i(x_{t+1}) + u_{i,t}(x_{t+1}) - u_{i,t}(x_{t+1})]^+, \\
        &\leq [-f_i(x_{t+1}) + u_{i,t}(x_{t+1})]^+ + [- u_{i,t}(x_{t+1})]^+, \\ \label{eq:feas-prob}
        &= [-f_i(x_{t+1}) + u_{i,t}(x_{t+1})]^+, \\
        &\leq [u_{i,t}(x_{t+1}) - l_{i,t}(x_{t+1})]^+, \\
        &= w_{i,t}(x_{t+1}),
    \end{align}
    \end{subequations}
    where \eqref{eq:feas-prob} follows from feasibility of $x_{t+1}$ for the CUQB subproblem \eqref{eq:cdf_quant_def}. $\hfill \square$

\subsection{Proof of Lemma \ref{lemma:3}}
\label{subsec:lem3}

    From \cite[Lemma 5.1]{srinivas09}, the following concentration inequality must hold for each element of the random vector $X$
    \begin{align}
        \text{Pr}\{ | X_j - \mu_{X,j} | \leq \beta^{1/2} \sigma_{X,j} \} \geq 1 - e^{-\beta/2} = 1 - \textstyle\frac{\alpha}{2m}.
    \end{align}
    where the latter equality follows by substituting $\beta = 2\log(2m/\alpha)$.
    Applying Boole's inequality across $j \in \mathbb{N}_1^m$, we have
    \begin{align} \label{eq:xbound}
        \text{Pr}\{ | X_j - \mu_{X,j} | \leq \beta^{1/2} \sigma_{X,j}, ~ \forall j \in \mathbb{N}_1^d \} \geq 1 - \textstyle\frac{\alpha}{2}.
    \end{align}
    From the Lipschitz continuity of $g$, we have that $g(x) - g(\mu) \leq L_g\sum_{j=1}^m | x_j - \mu_{X,j} |$ holds for any $x$, which can be combined with \eqref{eq:xbound} to yield
    \begin{align}
        \text{Pr}\{ g(X) - g(\mu) \leq L_g \beta^{1/2} \textstyle\sum_{j=1}^m \sigma_{X,j}  \} \geq 1 - \textstyle\frac{\alpha}{2}.
    \end{align}
    We can rearrange this expression to be in terms of $Y$
    and compare it to the definition of the quantile function
    \begin{align} 
        \text{Pr}\{ Y \leq Q_Y(1 - \textstyle\frac{\alpha}{2}) \} = 1 - \textstyle\frac{\alpha}{2} \leq \text{Pr}\{ Y \leq g(\mu) + L_g \beta^{1/2} \textstyle\sum_{j=1}^m \sigma_{X,j} \}.
    \end{align}
    Since these events are one-sided, this can be true if and only if
    \begin{align} \label{eq:upperquant}
    Q_Y(1 - \textstyle\frac{\alpha}{2}) \leq g(\mu) + L_g \beta^{1/2} \textstyle\sum_{j=1}^m \sigma_{X,j}.    
    \end{align}
    We can repeat a similar analysis for the lower bound to find 
    \begin{align} \label{eq:lowerquant}
      Q_Y(\textstyle\frac{\alpha}{2}) \geq g(\mu) - L_g \beta^{1/2} \textstyle\sum_{j=1}^m \sigma_{X,j}.
    \end{align}
    Subtracting \eqref{eq:lowerquant} from \eqref{eq:upperquant} then yields the stated result. $\hfill \square$

\subsection{Proof of Theorem \ref{thm:2}}
\label{subsec:thm2}

    Using previous results, we can establish the following sequence of inequalities for the cumulative regret
    \begin{subequations}
    \begin{align}
        R_T &= \textstyle \sum_{t=0}^{T-1} r_t, \\
        &\leq \textstyle\sum_{t=0}^{T-1} r^+_t = R_T^+, \\
        &\leq \textstyle\sum_{t=0}^{T-1} w_{0,t}(x_{t+1}), \\
        &\leq \textstyle\sum_{t=0}^{T-1} 2 L_0 \beta_{t+1}^{1/2} \sum_{j=1}^m \sigma_{j,t}(x_{t+1}), \\
        &\leq 2 L_0 \beta_{T}^{1/2} \textstyle\sum_{j=1}^m \textstyle\sum_{t=0}^{T-1} \sigma_{j,t}(x_{t+1}), \\
        &\leq 4 L_0 \beta_{T}^{1/2} \textstyle\sum_{j=1}^m \sqrt{(T+2)\gamma_{j,T}} = 4 L_0 \beta_{T}^{1/2} \Psi_T.
    \end{align}
    \end{subequations}
    The third and fourth line follow from Lemmas \ref{lemma:2} and \ref{lemma:3}, respectively. The fifth line follows from monotonicity of the sequence $\{ \beta^{1/2}_{t} \}_{t \geq 1}$ and the final line follows from Lemma \ref{lemma:4} and the definition of $\Psi_T$. We can repeat a similar procedure for the cumulative violation of constraint $i$ to show
    \begin{subequations}
        \begin{align}
            V_{i,T} &= \textstyle\sum_{t=0}^{T-1} v_{i,t+1}, \\
            &\leq \textstyle\sum_{t=0}^{T-1} w_{i,t}(x_{t+1}), \\
            &\leq \textstyle\sum_{t=0}^{T-1} 2L_i \beta_{t+1}^{1/2}\sum_{j=1}^m \sigma_{j,t}(x_{t+1}), \\
            &\leq 2L_i\beta_{T}^{1/2} \textstyle\sum_{j=1}^{m}\textstyle\sum_{t=0}^{T-1} \sigma_{j,t}(x_{t+1}), \\
            &\leq 4 L_i \beta_T^{1/2} \textstyle\sum_{j=1}^{m} \sqrt{(T+2)\gamma_{j,T}} = 4L_i\beta_{T}^{1/2} \Psi_T.
        \end{align}
    \end{subequations}
    Lastly, note $\beta^{1/2}_{T}$ grows logarithmically with $T$ and $\Psi_T$ is a finite summation over terms $\mathcal{O}(\sqrt{\gamma_{j,T} T})$ such that the claimed scaling laws must hold true, which completes the proof. $\hfill \square$

\subsection{Proof of Theorem \ref{thm:3}}
\label{subsec:thm3}

    Consider the following sum
    \begin{align}
         \sum_{t=0}^{T-1} \left( r^+(x_{t+1}) + \sum_{i=1}^n v_{i}(x_{t+1}) \right) &= R_T^+ + \sum_{i=1}^n V_{i,T} \leq 4 \mathcal{L} \beta_T^{1/2} \Psi_T,
    \end{align} 
    The minimum of a sequence of points must be less than or equal to their average
    \begin{align}
        \min_{t \in \{ 1, \ldots, T \}} \left( r^+(x_{t+1}) + \sum_{i=1}^n v_{i}(x_{t+1}) \right) &\leq \frac{1}{T} \sum_{t=0}^{T-1} \left( r^+(x_{t+1}) + \sum_{i=1}^n v_{i}(x_{t+1}) \right), \\\notag
        &\leq \frac{4 \mathcal{L} \beta_T^{1/2} \Psi_T}{T}.
    \end{align}
    This implies that there must exist a point $\tilde{x}_T \in \{ x_1,\ldots,x_T \}$ that satisfies
    \begin{align}
        r^+(\tilde{x}_{T}) + \sum_{i=1}^n v_{i}(\tilde{x}_{T}) \leq \frac{4 \mathcal{L} \beta_T^{1/2} \Psi_T}{T}.
    \end{align}
    Since each of these terms is non-negative, the bound must hold for all terms, i.e., 
    \begin{subequations}        
    \begin{align}
        r(\tilde{x}_T) \leq r^+(\tilde{x}_{T}) &\leq \frac{4 \mathcal{L} \beta_T^{1/2} \Psi_T}{T}, \\
        v_{i}(\tilde{x}_{T}) &\leq \frac{4 \mathcal{L} \beta_T^{1/2} \Psi_T}{T}, ~~ \forall i \in \mathbb{N}_1^n.
    \end{align}
    \end{subequations}
    The stated result then immediately follows. $\hfill \square$

\subsection{Proof of Theorem \ref{thm:4}}
\label{subsec:thm4}

    The lower bound holds since $f_0(x) - \rho\sum_{i=1}^n [-f_i(x)]^+ \leq f_0^\star$ for all $x \in \mathcal{X}$ by the exact penalty property for any $\rho \geq \bar{\rho}$. To derive the upper bound, we first define a quantity that we refer to as instantaneous penalty-based regret
    \begin{align}
        r_{EP}(x_t;\rho) &= r_t + \rho \textstyle\sum_{i=1}^n v_{i,t} = f_0^\star - f_0(x_t) + \rho \textstyle\sum_{i=1}^n [-f_i(x_t)]^+.
    \end{align}
    Next, we define a conservative (pessimistic) estimate of $r_{EP}(x_t;\rho)$ by replacing the true function with their lower quantile bounds
    \begin{align}
        \bar{r}_{EP}(x_t;\rho) &= f_0^\star - l_{0,t-1}(x_t) + \rho \textstyle\sum_{i=1}^n [-l_{i,t-1}(x_t)]^+,
    \end{align}
    such that $r_{EP}(x_t;\rho) \leq \bar{r}_{EP}(x_t;\rho)$ for any choice of $x_t$ with probability $\geq 1 - \delta$ by Lemma \ref{lemma:1}. Following the same logic used in Theorem \ref{thm:2}, we can derive
    \begin{align} \label{eq:boundrEPT}
        \textstyle\sum_{t=1}^T \bar{r}_{EP}(x_t;\rho) \leq 4 \mathcal{L}_r(\rho) \beta_T^{1/2} \Psi_T,
    \end{align}
    since the bounds established on $r_{t+1}$ and $v_{i,t+1}$ in Lemma \ref{lemma:2} also hold for their pessimistic counterparts. The final step is to recognize that
    \begin{align} \label{eq:boundrEPmin}
        r_{EP}(\tilde{x}_{T}^r;\rho) \leq \bar{r}_{EP}(\tilde{x}_{T}^r;\rho) = \min_{t \in \{1,\ldots,T\}} \bar{r}_{EP}(x_t;\rho) \leq \frac{1}{T} \sum_{t=1}^T \bar{r}_{EP}(x_t;\rho),
    \end{align}
    where $\tilde{x}_T^r = x_{t^\star}$ with index $t^\star$ given by \eqref{eq:recommended-point}. The upper bound then follows by combining \eqref{eq:boundrEPT} and \eqref{eq:boundrEPmin} and substituting the definition of $r_{EP}(\tilde{x}_{T}^r;\rho)$. $\hfill \square$

\subsection{Proof of Corollary \ref{cor:1}}
\label{subsec:cor1}

The following bounds must hold with probability $\geq 1-\delta$
\begin{align} \label{eq:bounds-cor1}
   l_{i,t}(x) \leq f_i(x) \leq u_{i,t}(x), ~~~ \forall i \in \mathbb{N}_0^n, \forall t \geq 0, \forall x \in \mathcal{X},
\end{align}
based on Lemma \ref{lemma:1}. By definition, an exact penalty function for \eqref{eq:grey-box-opt} must satisfy
\begin{align} \label{eq:exact-pen-cor1}
    f_0(x) - \rho\textstyle\sum_{i=1}^n [-f_i(x)]^+ - f_0^\star \leq 0, ~~~ \forall x \in \mathcal{X},
\end{align}
for some $\rho \geq 0$. We can construct an upper bound for the left-hand side of this inequality by individually bounding each term in \eqref{eq:exact-pen-cor1} using \eqref{eq:bounds-cor1} as follows
\begin{align*}
    f_0(x) &\leq u_{0,0}(x), &\forall x \in \mathcal{X}, \\
    [-f_i(x)]^+ &\geq [-u_{i,0}(x)]^+, &\forall x \in \mathcal{X},\forall i \in \mathbb{N}_1^n, \\
    f_0^\star = \max_{x \in \mathcal{X}, f_{i}(x) \geq 0, \forall i \in \mathbb{N}_1^n} &\geq l_{0,0}^\star = \max_{x \in \mathcal{X}, l_{i,0}(x) \geq 0, \forall i \in \mathbb{N}_1^n} l_{0,0}(x),
\end{align*}
which also hold with probability $\geq 1-\delta$. Using these inequalities, we can establish
\begin{align}
    \left( f_0(x) - \rho\textstyle\sum_{i=1}^n [-f_i(x)]^+ - f_0^\star \right) \leq \left( u_{0,0}(x) - \rho\textstyle\sum_{i=1}^n[-u_{i,0}(x)]^+ - l_{0,0}^\star \right),
\end{align}
for all $x \in \mathcal{X}$, meaning \eqref{eq:rho-tilde} is a sufficient condition for the exact penalty property \eqref{eq:exact-pen-cor1} to hold, which completes the proof. $\hfill \square$

\subsection{Proof of Theorem \ref{thm:5}}
\label{subsec:thm5}

    The proof follows the same line of reasoning as \cite[Theorem 5.1]{xu2022constrained}. Suppose that infeasibility has not been declared up until step $T$ and let $i$ be in the index of the infeasible constraint. By Lemma \ref{lemma:1}, we have $f_i(x_{t+1}) \in [l_{i,t}(x_{t+1}), u_{i,t}(x_{t+1})]$ for all $t \in \mathbb{N}_0^{T-1}$ with probability at least $1 - \delta$.
    Then, since infeasibility has yet to be declared, we must have
    \begin{align}
        u_{i,t}(x_{t+1}) \geq 0, ~~ \forall t \in \mathbb{N}_0^{T-1},
    \end{align}
    and
    \begin{align}
        l_{i,t}(x_{t+1}) \leq f_{i}(x_{t+1}) \leq \max_{x \in \mathcal{X}} f_i(x) = -\epsilon,
    \end{align}
    with probability $\geq 1 - \delta$. We can use these to lower bound the sum
    \begin{align}
        \textstyle\sum_{t=0}^{T-1} \left( u_{i,t}(x_{t+1}) - l_{i,t}(x_{t+1}) \right) \geq \textstyle\sum_{t=0}^{T-1} \epsilon = T \epsilon.
    \end{align}
    Combining Lemmas \ref{lemma:3} and \ref{lemma:4}, we can upper bound this sum as follows
    \begin{subequations}        
    \begin{align}
        \textstyle\sum_{t=0}^{T-1} \left( u_{i,t}(x_{t+1}) - l_{i,t}(x_{t+1}) \right) &\leq \textstyle\sum_{t=0}^{T-1} 2L_i \beta_{t+1}^{1/2} \sum_{j=1}^m \sigma_{j,t}(x_{t+1}), \\
        &\leq 2L_i \beta_{T}^{1/2} \textstyle\sum_{j=1}^m \textstyle\sum_{t=0}^{T-1}\sigma_{j,t}(x_{t+1}), \\
        &\leq 4L_i \beta_{T}^{1/2}\textstyle\sum_{j=1}^m \sqrt{(T+2)\gamma_{j,T}},
    \end{align}
    \end{subequations}
    where the expression for $\beta_{T}^{1/2}$ is given in Theorem \ref{thm:2}. Therefore, we have
    \begin{align}
        T \epsilon \leq 4L_i \beta_{T}^{1/2}\textstyle\sum_{j=1}^m \sqrt{(T+2)\gamma_{j,T}},
    \end{align}
    which implies
    \begin{align} \label{eq:eps-bound}
        \epsilon \leq \frac{4L_i \beta_{T}^{1/2}\textstyle\sum_{j=1}^m \sqrt{(T+2)\gamma_{j,T}}}{T} = \mathcal{O}\left( \sqrt{\frac{\log(T) \gamma_{T}^\text{max} }{T}} \right),
    \end{align}
    since $\beta_T = \mathcal{O}(\log(T))$ by definition. The inequality \eqref{eq:eps-bound} implies there exists a constant $\tilde{C} > 0$ such that
    \begin{align}
        \epsilon \leq \tilde{C}\sqrt{\frac{\log(T) \gamma_{T}^\text{max} }{T}},
    \end{align}
    which can be rearranged, with $C = 1/\tilde{C}$, as follows 
    \begin{align} \label{eq:Cepsbound}
        C \epsilon \leq \sqrt{\frac{\log(T) \gamma_{T}^\text{max} }{T}}.
    \end{align}
    However, the worst-case MIG satisfies $\lim_{T \to \infty} \frac{\log(T) \gamma_T^\text{max}}{T} = 0$ by assumption such that \eqref{eq:Cepsbound} will be violated when $T$ is large enough. Thus, infeasibility will be declared on or before $\overline{T}$, which is the first iteration that \eqref{eq:Cepsbound} is violated. $\hfill \square$

\section{Appendix: Unconstrained Synthetic Test Problems}
\label{appendix:B}

\subsection{Booth}
\label{subsec:booth}
\noindent The Booth problem is defined by the following functions
\begin{align*}
    g_0(x,y) &= -\left( y_1 + (2x_1 + x_2 - 5)^2 \right).
\end{align*}
\begin{align*}
    h(x) = (x_1 + 2x_2 - 7)^2.
\end{align*}
\begin{align*}
    -10 \leq x_i \leq 10, ~~~ \forall i = 1, 2.
\end{align*}
The global maximum is $0$ with $x^\star = [1, 3]^\top$.

\subsection{Wolfe}
\label{subsec:wolfe}
\noindent The Wolfe problem is defined by the following functions
\begin{align*}
    g_0(x,y) &= -\left( \frac{4}{3}y_1 + x_3 \right).
\end{align*}
\begin{align*}
    h(x) = (x_1^2 + x_2^2 - x_1 x_2)^{0.75}.
\end{align*}
\begin{align*}
    0 \leq x_i \leq 2, ~~~ \forall i = 1, \ldots, 3.
\end{align*}
The global maximum is $0$ with $x^\star = [1, 1, 1]^\top$.

\subsection{Rastrigin}
\label{subsec:rastrigin}
\noindent The Rastrigin problem is defined by the following functions
\begin{align*}
    g_0(x,y) &= -\left( y_1 + y_2 + 30 + x_3^2 - 10\cos(2\pi x_3) \right).
\end{align*}
\begin{align*}
    h(x) = \begin{bmatrix}
        x_1^2 - 10\cos(2\pi x_1) \\
        x_2^2 - 10\cos(2\pi x_2)
    \end{bmatrix}.
\end{align*}
\begin{align*}
    -5 \leq x_i \leq 5, ~~~ \forall i = 1, \ldots, 3.
\end{align*}
The global maximum is $0$ with $x^\star = [0, 0, 0]^\top$.

\subsection{Colville}
\label{subsec:colville}
\noindent The Colville problem is defined by the following functions
\begin{align*}
    g_0(x,y) &= -\bigg( y_1 + 90(x_3^2 - x_4)^2 + 10.1( (x_2-1)^2 + (x_4-1)^2 ) \\\notag 
    & ~~~~~~~~ + 19.8(x_2-1)(x_4-1) \bigg).
\end{align*}
\begin{align*}
    h(x) = 100(x_1^2 - x_2)^2 + (x_3 - 1)^2 + (x_1 - 1)^2.
\end{align*}
\begin{align*}
    -10 \leq x_i \leq 10, ~~~ \forall i = 1, \ldots, 4.
\end{align*}
The global maximum is $0$ with $x^\star = [1, 1, 1, 1]^\top$.

\subsection{Friedman}
\label{subsec:friedman}
\noindent The Friedman problem is defined by the following functions
\begin{align*}
    g_0(x,y) &= -\left( 10y_1 + 20(x_3 - 0.5)^2 + 10 x_4 + 5 x_5 \right).
\end{align*}
\begin{align*}
    h(x) = \sin(\pi x_1 x_2).
\end{align*}
\begin{align*}
    0 \leq x_i \leq 1, ~~~ \forall i = 1, \ldots, 5.
\end{align*}
The global maximum is $27.5$ with $x^\star = [x_1^\star, x_2^\star, 0.5, -1.5, -1.5]^\top$ for any $x_1^\star$ and $x_2^\star$ satisfying $\sin(\pi x_1^\star x_2^\star) = -1$.

\subsection{Dolan}
\label{subsec:dolan}
\noindent The Dolan problem is defined by the following functions
\begin{align*}
    g_0(x,y) &= -\left( y_1 - y_2 + 0.2x_5^2 - x_2 - 1 \right).
\end{align*}
\begin{align*}
    h(x) = \begin{bmatrix}
        (x_1 + 1.7x_2)\sin(x_1) \\
        1.5 x_3 - 0.1 x_4 \cos(x_5 + x_4 - x_1)
    \end{bmatrix}.
\end{align*}
\begin{align*}
    -100 \leq x_i \leq 100, ~~~ \forall i = 1, \ldots, 5.
\end{align*}
The global maximum is $529.87$ with $x^\star = [98.964,100,100,99.224,-0.25]^\top$.

\subsection{Rosenbrock}
\label{subsec:rb}
\noindent The Rosenbrock problem is defined by the following functions
\begin{align*}
    g_0(x, y) &= -\bigg( \textstyle\sum_{i=1}^3 (100 y_i^2  + (1-x_i)^2) + 100(x_5 - x_4^2) + y_4 \\\notag
    & ~~~~~~~~ + 100(x_6 - x_5^2) + (1 - x_5)^2 \bigg).
\end{align*}
\begin{align*}
    h(x) = \begin{bmatrix}
    x_2^2 - x_1^2 \\
    x_3^2 - x_2^2 \\
    x_4^2 - x_3^2 \\
    (1-x_4)^2
    \end{bmatrix}.
\end{align*}
\begin{align*}
    -2 \leq x_i \leq 2, ~~~ \forall i = 1,\ldots, 6.
\end{align*}
The global maximum is 0 with $x^\star = [0, 0, 0, 0, 0, 0]^\top$.

\subsection{Zakharov}
\label{subsec:zakharov}
\noindent The Zakharov problem is defined by the following functions
\begin{align*}
    g_0(x, y) &= -\left( \textstyle\sum_{i=}^7 x_i^2 + \textstyle\sum_{i=}^7 (0.5 i x_i)^2 + y_1 \textstyle\sum_{i=}^7 (0.5 i x_i)^2 \right).
\end{align*}
\begin{align*}
    h(x) = \textstyle\sum_{i=}^7 (0.5 i x_i)^2.
\end{align*}
\begin{align*}
    -5 \leq x_i \leq 10, ~~~ \forall i = 1, \ldots, 7.
\end{align*}
The global maximum is $0$ with $x^\star = [0, 0, 0, 0, 0, 0, 0]^\top$.

\subsection{Powell}
\label{subsec:powell}
\noindent The Powell problem is defined by the following functions
\begin{align*}
    g_0(x, y) &= -\bigg( y_1 + (x_5 + 10x_6)^2 + y_2 + 5(x_7 - x_8)^2 \\\notag 
    & ~~~~~~~~ + (x_2 - 2x_3)^4 + y_3 + 10(x_1 - x_4)^4 + y_4 \bigg).
\end{align*}
\begin{align*}
    h(x) = \begin{bmatrix}
    (x_1 + 10x_2)^2 \\
    5(x_3 - x_4)^2 \\
    (x_6 - 2x_7)^4 \\
    10(x_5 - x_8)^4
    \end{bmatrix}.
\end{align*}
\begin{align*}
    -4 \leq x_i \leq 5, ~~~ \forall i = 1, \ldots, 8.
\end{align*}
The global maximum is $0$ with $x^\star = [0, 0, 0, 0, 0, 0, 0, 0]^\top$.

\subsection{Styblinski-Tang}
\label{subsec:st}
\noindent The Styblinski-Tang problem is defined by the following functions
\begin{align*}
    g_0(x, y) &= -\left( \textstyle\sum_{i=1}^4 y_i + \textstyle\sum_{i=5}^9 (0.5 x_i^4 - 16x_i^2 + 5x_i) \right).
\end{align*}
\begin{align*}
    h(x) = \begin{bmatrix}
    0.5(x_1^4 - 16x_1^2 + 5x_1) \\
    0.5(x_2^4 - 16x_2^2 + 5x_2) \\
    0.5(x_3^4 - 16x_3^2 + 5x_3) \\
    0.5(x_4^4 - 16x_4^2 + 5x_4)
    \end{bmatrix}.
\end{align*}
\begin{align*}
    -5 \leq x_i \leq 5, ~~~ \forall i = 1, \ldots, 9.
\end{align*}
The global maximum is $352.49$ with $x^\star = -2.904[1, 1, 1, 1, 1, 1, 1, 1, 1]^\top$.

\section{Appendix: Constrained Synthetic Test Problems}
\label{appendix:C}

\subsection{Bazaraa}
\label{subsec:bazaraa}
\noindent The Bazaraa problem is defined by the following functions
\begin{align*}
    g_0(x, y) &= -\left( 2x_1^2 + 2x_2^2 - y_2 \right), \\\notag
    g_1(x, y) &= -\left( 5x_1 + x_2 - 5 \right), \\\notag
    g_2(x, y) &= -\left( y_1 - x_1 \right).
\end{align*}
\begin{align*}
    h(x) = \begin{bmatrix}
    2x_2^2 \\
    2x_1 x_2 + 6 x_1 + 4 x_2
    \end{bmatrix}.
\end{align*}
\begin{align*}
    0.01 \leq x_i \leq 1, ~~~ \forall i = 1, 2.
\end{align*}
The global maximum is $6.613$ with $x^\star = [0.868, 0.659]^\top$.

\subsection{Spring}
\label{subsec:spring}
\noindent The Spring problem is defined by the following functions
\begin{align*}
    g_0(x, y) &= -\left( 2y_1 + 2x_3 \right), \\\notag
    g_1(x, y) &= -\left( 2x_2^2 - x_1 \right), \\\notag
    g_2(x, y) &= -\left( \frac{4 x_2^2 - x_1 x_2}{12566(x_1^3 x_2 - x_1^4)} + \frac{1}{5108 x_1 x_2 - 1} \right), \\\notag
    g_3(x, y) &= -\left( 1 - 140.45\frac{x_1}{y_2} \right), \\\notag
    g_4(x, y) &= -\left( \frac{2}{3}(x_1 + x_2) - 1 \right). 
\end{align*}
\begin{align*}
    h(x) = \begin{bmatrix}
    x_1^2 x_2 \\
    x_2^3 x_3
    \end{bmatrix}.
\end{align*}
\begin{align*}
    0.05 &\leq x_1 \leq 2, \\\notag
    0.25 &\leq x_2 \leq 1.3 \\\notag
    2 &\leq x_3 \leq 15.
\end{align*}
The global maximum is $-0.0127$ with $x^\star = [0.052, 0.357, 11.289]^\top$.

\subsection{Ex314}
\label{subsec:ex314}
\noindent The Ex314 problem is defined by the following functions
\begin{align*}
    g_0(x, y) &= y_2, \\\notag
    g_1(x, y) &= -x_1 y_1 - 2x_2^2 + 2x_1x_2 + 2x_2x_3 - 2x_1x_3 - 2x_3^2 + 20x_1 - 9x_2 + 13x_3-24, \\\notag
    g_2(x, y) &= x_1 + x_3 + x_3 -4, \\\notag
    g_3(x, y) &= 3x_2 + x_6 - 6.
\end{align*}
\begin{align*}
    h(x) = \begin{bmatrix}
    4 x_1 - 2x_2 + 2 x_3 \\
    x_2 - x_3 - 2x_1
    \end{bmatrix}.
\end{align*}
\begin{align*}
    -2 &\leq x_1 \leq 2, \\\notag
    0 &\leq x_2 \leq 6 \\\notag
    -3 &\leq x_3 \leq 3.
\end{align*}
The global maximum is $4$ with $x^\star = [0.5, 0.0, 3.0]^\top$.

\subsection{Rosen-Suzuki}
\label{subsec:rosen-suzuki}
\noindent The Rosen-Suzuki problem is defined by the following functions
\begin{align*}
    g_0(x,y) &= -( x_1^2 + x_2^2 + x_4^2 - 5x_1 - 5x_2 + y_1 ), \\\notag
    g_1(x,y) &= 8 - x_1^2 - x_2^2 - x_3^2 - x_4^2 - x_1 + x_2 - x_3 + x_4, \\\notag
    g_2(x,y) &= 10 - x_1^2 - 2x_2^2 - y_2 + x_1 + x_4, \\\notag
    g_3(x,y) &= 5 - 2x_1^2 - x_2^2 - x_3^2 - 2x_1 + x_2 + x_4.
\end{align*}
\begin{align*}
    h(x) = \begin{bmatrix}
    2x_3^2 - 21x_3 + 7x_4 \\
    x_3^2 + 2x_4^2
    \end{bmatrix}.
\end{align*}
\begin{align*}
    -2 \leq x_i \leq 2, ~~~ \forall i = 1,\ldots, 4.
\end{align*}
The global maximum is $44$ with $x^\star = [0, 1, 2, -1]^\top$.

\subsection{st\_bpv1}
\label{subsec:st-bpv1}
\noindent The st\_bpv1 problem is defined by the following functions
\begin{align*}
    g_0(x,y) &= -( y_1 + x_2 x_4 ), \\\notag
    g_1(x,y) &= -( 30 - y_2 ), \\\notag
    g_2(x,y) &= -( 20 - y_3 ), \\\notag
    g_3(x,y) &= -( x_3 + x_4 - 15 ).
\end{align*}
\begin{align*}
    h(x) = \begin{bmatrix}
    x_1 x_3 \\
    x_1 + 3x_2 \\
    2x_1 + x_2
    \end{bmatrix}.
\end{align*}
\begin{align*}
    0 \leq x_1 \leq 27, \\\notag
    0 \leq x_2 \leq 16, \\\notag
    0 \leq x_3 \leq 10, \\\notag
    0 \leq x_4 \leq 10.
\end{align*}
The global maximum is $-10$ with $x^\star = [27, 1, 0, 10]^\top$.

\subsection{Ex211}
\label{subsec:ex211}
\noindent The Ex211 problem is defined by the following functions
\begin{align*}
    g_0(x,y) &= -( 42x_1 - 50y_1 + 44x_2 + 45x_3 + 47x_4 + 47.5 x_5 ), \\\notag
    g_1(x,y) &= -( 20x_1 + y_2 + 4x_5 - 39 ).
\end{align*}
\begin{align*}
    h(x) = \begin{bmatrix}
    x_1^2 + x_2^2 + x_3^2 + x_4^2 + x_5^2 \\
    12x_2 + 11x_3 + 7x_4
    \end{bmatrix}.
\end{align*}
\begin{align*}
    0 \leq x_i \leq 1, ~~~ \forall i = 1, \ldots, 5.
\end{align*}
The global maximum is $17$ with $x^\star = [1, 1, 0, 1, 0]^\top$.

\subsection{Ex212}
\label{subsec:ex212}
\noindent The Ex212 problem is defined by the following functions
\begin{align*}
    g_0(x,y) &= 10x_6 + y_1 + 0.5(x_1^2 + x_2^2 + x_3^2 + x_4^2 + x_5^2), \\\notag
    g_1(x,y) &= -( 6x_1 + 3x_2 + 3x_3 + 2x_4 + x_5 - 6.5 ), \\\notag
    g_2(x,y) &= -( y_2 - 20 ).
\end{align*}
\begin{align*}
    h(x) = \begin{bmatrix}
    10.5x_1 + 7.5x_2 + 3.5x_3 + 2.5x_4 + 1.5x_5 \\
    10x_1 + 10x_3 + x_6
    \end{bmatrix}.
\end{align*}
\begin{align*}
    0 \leq x_i \leq 30, ~~~ \forall i = 1, \ldots, 6.
\end{align*}
The global maximum is $213$ with $x^\star = [0, 1, 0, 1, 1, 20]^\top$.

\subsection{g09}
\label{subsec:g09}
\noindent The g09 problem is defined by the following functions
\begin{align*}
    g_0(x,y) &= -y_1 - x_3^4 - 3(x_4 - 11)^2 - 10x_5^6 - 7x_6^2 - x_7^4 + 4x_6x_7 + 10x_6 + 8x_7, \\\notag
    g_1(x,y) &= 127 -2x_1x_2 - y_2 - 5x_5 , \\\notag
    g_2(x,y) &= 282 - 7x_1 - 3x_2 - 10x_3^2 - x_4 + x_5, \\\notag
    g_3(x,y) &= 196 - 23x_1 + x_2^ 2 - 6x_6^2 + 8x_7, \\\notag
    g_4(x,y) &= -4x_1^2 - x_2^2 + 3x_1x_2 - 2x_3^2 - 5x_6 + 11x_7.
\end{align*}
\begin{align*}
    h(x) = \begin{bmatrix}
    (x_1 - 10)^2 + 5(x_2 - 12)^2 \\
    3x_2^4 + x_3 + 4x_4^2
    \end{bmatrix}.
\end{align*}
\begin{align*}
    -10 \leq x_i \leq 10, ~~~ \forall i = 1, \ldots, 7.
\end{align*}
The global maximum is $-680.63$ with $x^\star = [2.33, 1.95, -0.48, 4.37, -0.62, 1.04, 1.59]^\top$.

\subsection{Ex724}
\label{subsec:ex724}

\noindent The Ex724 problem is defined by the following functions
\begin{align*}
    g_0(x,y) &= -( y_3 + 0.4 (x_2/x_8)^{0.67} - x_1 + 10 ), \\\notag
    g_1(x,y) &= -( 0.0588 x_5 x_7 + 0.1x_1 -1 ), \\\notag
    g_2(x,y) &= -( 0.0588 x_6 x_8 + 0.1x_1 + 0.1x_2 - 1 ), \\\notag
    g_3(x,y) &= -( 4(x_3/x_5) + 2/y_1 + 0.0588(x_7/x_3)^{1.3} - 1 ), \\\notag
    g_4(x,y) &= -( y_2 + 0.0588x_4^{1.3} x_8 - 1 ).
\end{align*}
\begin{align*}
    h(x) = \begin{bmatrix}
    x_3^{0.71} x_5 \\
    4(x_4/x_6) + 2/(x_4^{0.71} x_6) \\
    0.4(x_1 / x_7)^{0.67} - x_2
    \end{bmatrix}.
\end{align*}
\begin{align*}
    0.1 \leq x_i \leq 10, ~~~ \forall i = 1, \ldots, 8.
\end{align*}
The global maximum is $-3.92$ with $x^\star = [6.35, 2.34, 0.67, 0.53, 5.95, 5.32, 1.04, 0.42]^\top$.

\subsection{Ex216}
\label{subsec:ex216}

\noindent The Ex216 problem is defined by the following functions
\begin{align*}
    g_0(x,y) &= 48x_1 - 0.5y_1 - 50x_5^2 - 50x_6^2 - 50x_7^2 - 50x_8^2 - 50x_9^2 - 50x_{10}^2 \\\notag 
    & ~~~~~~~~ + 42x_2 + y_3 + 47x_7 + 42x_8 + 45x_9 + 46x_{10}, \\\notag
    g_1(x,y) &= y_2 - 2x_7 - 6x_8 - 2x_9 - 2x_10 + 4, \\\notag
    g_2(x,y) &= 6x_1 - 5x_2 + 8x_3 - 3x_4 + x_6 + 3x_7 + 8x_8 + 9x_9 - 3x_{10} - 22, \\\notag
    g_3(x,y) &= -5x_1 + 6x_2 + 5x_3 + 3x_4 + 8x_5 - 8x_6 + 9x_7 + 2x_8 - 9x_{10} + 6, \\\notag
    g_4(x,y) &= y_4 + 3x_7 - 9x_8 - 9x_9 - 3x_{10} + 23, \\\notag
    g_5(x,y) &= -8x_1 + 7x_2 - 4x_3 - 5x_4 - 9x_5 + x_6 - 7x_7 - x_8 + 3x_9 - 2x_{10} + 12.
\end{align*}
\begin{align*}
    h(x) = \begin{bmatrix}
    100 x_1^2 + 100 x_2^2 + 100 x_3^2 + 100 x_4^2 \\
    -2x_1 6 x_2 - x_3 - 3x_5 - 3x_6 \\
    48 x_3 + 45 x_4 + 44 x_5 + 41 x_6 \\
    9 x_1 + 5 x_2 - 9 x_4 + x_5 - 8 x_6
    \end{bmatrix}.
\end{align*}
\begin{align*}
    0 \leq x_i \leq 1, ~~~ \forall i = 1, \ldots, 10.
\end{align*}
The global maximum is $39$ with $x^\star = [1, 0, 0, 1, 1, 1, 0, 1, 1, 1]^\top$.

\bibliographystyle{elsarticle-num}           
\bibliography{references}

\end{document}